\documentclass{article}

\PassOptionsToPackage{numbers, compress}{natbib}
\usepackage[preprint]{packages/nips_2018}

\usepackage[utf8]{inputenc} 
\usepackage[T1]{fontenc}    
\usepackage{hyperref}       
\usepackage{url}            
\usepackage{booktabs}       
\usepackage{amsfonts}       
\usepackage{nicefrac}       
\usepackage{microtype}      

\usepackage{graphicx} 
\usepackage{subfigure}
\usepackage{capt-of}

\usepackage{algorithm}
\usepackage{algorithmic}

\usepackage{amsmath,amssymb,amsthm}
\usepackage{cancel}
\usepackage{enumitem}
\usepackage{xfrac}   

\newtheorem{theorem}{Theorem}
\newtheorem{lemma}{Lemma}
\newtheorem{corollary}{Corollary}
\newtheorem{claim}{Claim}
\newtheorem{proposition}{Proposition}

\usepackage[toc,page]{appendix}
\addtocontents{toc}{\protect\setcounter{tocdepth}{-1}}

\title{Global Convergence to the Equilibrium of GANs using Variational Inequalities}

\author{
  Ian Gemp \\
  College of Information and Computer Sciences \\
  University of Massachusetts Amherst \\
  Amherst, MA 01003 \\
  \texttt{imgemp@cics.umass.edu} \\
   \And
   Sridhar Mahadevan \\
   College of Information and Computer Sciences \\
   University of Massachusetts Amherst \\
   Amherst, MA 01003 \\
   \texttt{mahadeva@cics.umass.edu} \\
}

\begin{document}

\maketitle

\begin{abstract}
In optimization, the negative gradient of a function denotes the direction of steepest descent. Furthermore, traveling in any direction orthogonal to the gradient maintains the value of the function. In this work, we show that these orthogonal directions that are ignored by gradient descent can be critical in equilibrium problems. Equilibrium problems have drawn heightened attention in machine learning due to the emergence of the Generative Adversarial Network (GAN). We use the framework of Variational Inequalities to analyze popular training algorithms for a fundamental GAN variant: the Wasserstein Linear-Quadratic GAN. We show that the steepest descent direction causes divergence from the equilibrium, and convergence to the equilibrium is achieved through following a particular orthogonal direction. We call this successful technique \emph{Crossing-the-Curl}, named for its mathematical derivation as well as its intuition: identify the game's axis of rotation and move ``across'' space in the direction towards smaller ``curling''.
\end{abstract}

\section{Introduction}
When minimizing $f(x)$ over $x \in \mathcal{X}$, it is known that $f$ decreases fastest if $x$ moves in the direction $-\nabla f(x)$. In addition, any direction orthogonal to $-\nabla f(x)$ will leave $f(x)$ unchanged. In this work, we show that these orthogonal directions that are ignored by gradient descent can be critical in equilibrium problems, which are central to game theory. If each player $i$ in a game updates with $x^{(i)} \leftarrow x^{(i)} - \rho \nabla_{x^{(i)}} f^{(i)}(x)$, $x=[x^{(1)};x^{(2)};\ldots]^\top$ can follow a cyclical trajectory, similar to a person riding a merry-go-round (see Figure~\ref{fig:merrygoround}). This toy scenario actually perfectly reflects an aspect of training for a particular machine learning model mentioned below, and is depicted more technically later on in Figure~\ref{fig:w1bplot}. To arrive at the equilibrium point, a person riding the merry-go-round should walk perpendicularly to their direction of travel, taking them directly to the center.

Equilibrium problems have drawn heightened attention in machine learning due to the emergence of the Generative Adversarial Network (GAN)~\cite{goodfellow2014generative}. GANs have served a variety of applications including generating novel images~\cite{karras2017progressive}, simulating particle physics~\cite{de2017learning}, and imitating expert policies in reinforcement learning~\cite{ho2016generative}. Despite this plethora of successes, GAN training remains heuristic.

Deep learning has benefited from an understanding of simpler, more fundamental techniques. For example, multinomial logistic regression formulates learning a multiclass classifier as minimizing the cross-entropy of a log-linear model where class probabilities are recovered via a \texttt{softmax}. The minimization problem is convex and is solved efficiently with guarantees using stochastic gradient descent (SGD). Unsurprisingly, the majority of deep classifiers incorporate a \texttt{softmax} at the final layer, minimize a cross-entropy loss, and train with a variant of SGD. This progression from logistic regression to classification with deep neural nets is not mirrored in GANs. In contrast, from their inception, GANs were architected with deep nets. Only recently has the Wasserstein Linear-Quadratic GAN (LQ-GAN)~\cite{feizi2017understanding,nagarajan2017gradient} been proposed as a minimal model for understanding GANs.
\begin{figure}[htbp]
    \centering
    \includegraphics[scale=0.35]{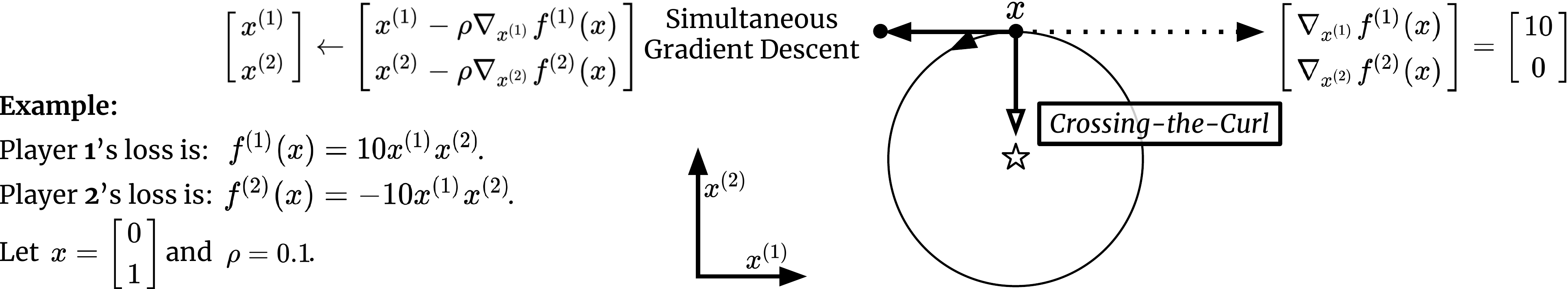}
    \caption{The goal is to find the equilibrium point (denoted by the star) of the merry-go-round. If someone follows simultaneous gradient descent, she will ride along in circles forever. However, if she travels perpendicularly to this direction, a.k.a. \emph{Crosses-the-Curl}, she will arrive at the equilibrium.}
    \label{fig:merrygoround}
\end{figure}
%
%

In this work, we analyze the convergence of several GAN training algorithms in the LQ-GAN setting. We survey several candidate theories for understanding convergence in GANs, naturally leading us to select Variational Inequalities, an intuitive generalization of the widely relied-upon theories from Convex Optimization. According to our analyses, none of the current GAN training algorithms is globally convergent in this setting. We propose a new technique, \emph{Crossing-the-Curl}, for training GANs that converges with high probability in the N-dimensional (N-d) LQ-GAN setting. 

This work makes the following contributions (proofs can be found in the supplementary material):
\begin{itemize}[leftmargin=*]
    \item The first global convergence analysis of several GAN training methods for the N-d LQ-GAN,
    \item \emph{Crossing-the-Curl}, the first technique with $\mathcal{O}(N/k)$ stochastic convergence for the N-d LQ-GAN,
    \item An empirical demonstration of \emph{Crossing-the-Curl} in the multivariate LQ-GAN setting as well as some common neural network driven settings in Appendix~\ref{deepnets}.
\end{itemize}

\section{Generative Adversarial Networks}
\label{sec:gan}
The Generative Adversarial Network (GAN)~\cite{goodfellow2014generative} formulates learning a generative model of data as finding a Nash equilibrium of a minimax game. The generator ($\min$ player) aims to synthesize realistic data samples by transforming vectors drawn from a fixed source distribution, e.g., $\mathcal{N}(\mathbf{0},I_d)$. The discriminator ($\max$ player) attempts to learn a scoring function that assigns low scores to synthetic data and high scores to samples drawn from the true dataset. The generator's transformation function, $G$, and discriminator's scoring function, $D$, are typically chosen to be neural networks parameterized by weights $\theta$ and $\phi$ respectively. The minimax objective of the original GAN~\cite{goodfellow2014generative} is
\begin{align}
    \min_{\theta} \max_{\phi} \Big\{ V(\theta,\phi) = \mathbb{E}_{y \sim p(y)} [ g (D_{\phi}(y)) ] + \mathbb{E}_{z \sim p(z)} [ g (-D_{\phi}(G_{\theta}(z)) ] \Big\}, \label{eqn:Vorig}
\end{align}
where $p(z)$ is the source distribution, $p(y)$ is the true data distribution, and $g(x)=-\log(1+e^{-x})$.
%

In practice, finding the solution to~(\ref{eqn:Vorig}) consists of local updates, e.g., SGD, to $\theta$ and $\phi$. This continues until 1) $V$ has stabilized, 2) the generated data is judged qualitatively accurate, or 3) training has de-stabilized and appears irrecoverable, at which point, training is restarted. The difficulty of training GANs has spurred research that includes reformulating the minimax objective~\cite{arjovsky2017wasserstein,mao2017least,mroueh2017fisher,mroueh2017mcgan,nowozin2016f,uehara2016generative,zhao2016energy}, devising training heuristics~\cite{gulrajani2017improved,karras2017progressive,salimans2016improved,roth2017stabilizing}, proving the existence of equilibria~\cite{arora2017generalization}, and conducting local stability analyses~\cite{gidel2018variational,mescheder2017numerics,mescheder2018training,nagarajan2017gradient}.

We acknowledge here that our algorithm, \emph{Crossing-the-Curl}, was independently proposed in~\cite{balduzzi2018mechanics} as \emph{Symplectic Gradient Adjustment} (SGA). In contrast to that work, this paper specifies a non-trivial application of this algorithm to LQ-GAN which obtains global convergence with high probability.

Recent work has studied a simplified setting, the Wasserstein LQ-GAN, where $G$ is a linear function, $D$ is a quadratic function, $g(x)=x$, and $p(z)$ is Gaussian~\cite{feizi2017understanding,nagarajan2017gradient}. Follow-up research has shown that, in this setting, the optimal generator distribution is a rank-$k$ Gaussian containing the top-$k$ principal components of the data~\cite{feizi2017understanding}. Furthermore, it is shown that if the dimensionality of $p(z)$ matches that of $p(y)$, LQ-GAN is equivalent to maximum likelihood estimation of the generator's resulting Gaussian distribution. To our knowledge, no GAN training algorithm with guaranteed convergence is currently known for this setting. We revisit the LQ-GAN in more detail in Section~\ref{sec:lqgan}.

\section{Convergence of Equilibrium Dynamics}
\label{sec:theories}
In this section, we review Variational Inequalities (VIs) and compare it to the ODE Method leveraged in recent work~\cite{nagarajan2017gradient}. See~\ref{DiffGames} and~\ref{AGT} for a discussion of two additional theories. Throughout the paper, $\mathcal{X}\subseteq \mathbb{R}^n$ refers to a convex set and $F$ refers to a vector field operator (or map) from $\mathcal{X}$ to $\mathbb{R}^n$, although many of the results for VIs apply to set-valued maps, e.g., subdifferentials, as well. Here, we will cover the basics of the theories and introduce select theorems when necessary later on.

\subsection{Variational Inequalities}
Variational Inequalities (VIs) are used to study equilibrium problems in a number of domains including mechanics, traffic networks, economics, and game theory~\cite{dafermos,facchinei-pang:vi,hartman-stampacchia:acta,nagurney:pdsbook}. The Variational Inequality problem, VI$(F,\mathcal{X})$, is to find an $x^*$ such that for all $x$ in the feasible set $\mathcal{X}$, $\langle F(x^*), x - x^* \rangle \ge 0$. Under mild conditions (see Appendix~\ref{vine}), $x^*$ constitutes a Nash equilibrium point. For readers familiar with convex optimization, note the consistent similarity throughout this subsection for when $F=\nabla f$. In game theory, $F$ often maps to the set of player gradients. For example, the map corresponding to the minimax game in Equation~(\ref{eqn:Vorig}) is $F: \mathbb{R}^{|\theta|+|\phi|} \rightarrow [\nabla V_{\theta}; -\nabla V_{\phi}]  \in \mathbb{R}^{|\theta|+|\phi|}$.

A map, $F$, is monotone~\cite{aslam1998generalized} if
$\langle F(x) - F(x'), x - x' \rangle \ge 0$ for all $x \in \mathcal{X}$ and $x' \in \mathcal{X}$.
Alternatively, if the Jacobian matrix of $F$ is positive semidefinite (PSD), then $F$ is monotone~\cite{nagurney:pdsbook,schaible1996generalized}. A matrix, $J$, is PSD if for all $x \in \mathbb{R}^n$, $x^\top  J x \ge 0$, or equivalently, $J$ is PSD if $(J$$+$$J^\top ) \succeq 0$.

As in convex optimization, a hierarchy of monotonicity exists. For all $x \in \mathcal{X}$ and $x' \in \mathcal{X}$, $F$ is
\begin{align}
    \text{monotone iff} & \quad \langle F(x) - F(x'), x - x' \rangle \ge 0, \label{monotone} \\
    \text{pseudomonotone iff} & \quad \langle F(x'), x - x' \rangle \ge 0 \implies \langle F(x), x - x' \rangle \ge 0, \nonumber \\
    \text{and quasimonotone iff} & \quad \langle F(x'), x - x' \rangle > 0 \implies \langle F(x), x - x' \rangle \ge 0. \label{def:quasimon}
\end{align}
If, in Equation~(\ref{monotone}), ``$\ge$'' is replaced by ``$>$'', then $F$ is strictly-monotone; if ``$\ge$'' is replaced by ``$s ||x-x'||^2$'', then $F$ is $s$-strongly-monotone. If $F$ is a gradient, then replace monotone with convex.

Table~\ref{tab:VIconv} cites algorithms with convergence rates for several settings. Whereas gradient descent achieves optimal convergence rates for various convex optimization settings, extragradient~\cite{korpelevich} achieves optimal rates for VIs. Results have been extended to the online learning setting as well~\cite{gemp2016online,gemp2017online}.
\begin{table}[]
    \centering
    \caption{Existing convergence rates for VI algorithms in different settings.}
    \begin{tabular}{c|c|c|c}
                  & Strongly-Monotone & (Smooth/Sharp+)Monotone & Pseudomonotone \\
    Deterministic & $\mathcal{O}(e^{-k})$~\cite{cavazzuti2002nash} & ($\mathcal{O}(1/k)$~\cite{nemirovski2004prox,cai20141}) $\mathcal{O}(1/\sqrt{k})$~\cite{juditsky2011solving} & $\mathcal{O}(1/\sqrt{k})$~\cite{dang2015convergence} \\ 
    Stochastic & $\mathcal{O}(1/k)$~\cite{kannan2017pseudomonotone} & ($\mathcal{O}(1/k)$~\cite{yousefian2014optimal,kannan2017pseudomonotone}) $\mathcal{O}(1/\sqrt{k})$~\cite{juditsky2011solving} 
    &  $\mathcal{O}(1/\sqrt{k})$~\cite{iusem2017extragradient} 
    \end{tabular}
    \label{tab:VIconv}
\end{table}

\subsection{The ODE Method \& Hurwitz Jacobians}
Recently,~\citet{nagarajan2017gradient} performed a \emph{local} stability analysis of the gradient dynamics of Equation~(\ref{eqn:Vorig}), proving that the Jacobian of $F$ evaluated at $x^*$ is Hurwitz\footnote{Our definition of Hurwitz is equivalent to the more standard: $-J$ is Hurwitz if $\max_i[\text{Re}(\lambda_i(-J))] < 0$.}~\cite{borkar:book,borkar2000ode,khalil1996noninear}, i.e., the real parts of its eigenvalues are strictly positive. This means that if simultaneous gradient descent using a ``square-summable, not summable'' step sequence enters an $\epsilon$-ball with a low enough step size, it will converge to the equilibrium. This applies only in the deterministic setting because stochastic gradients can cause the iterates to exit this ball and diverge. Note that while the real parts of eigenvalues reveal exponential growth or decay of trajectories, the imaginary parts reflect any rotation in the system\footnote{Linearized Dynamical System: $x(t) = \sum_{i} c_i v_i e^{\lambda_i t}$; Euler's formula: $e^{(a+ib)t} = e^{at}(\cos(bt)+i\sin(bt))$.}.

The Hurwitz and monotonicity properties are complementary (see~\ref{monvshurwitz}).
To summarize, Hurwitz encompasses dynamics with exponentially stable trajectories and with arbitrary rotation, while monotonicity includes cycles (Jacobians with zero eigenvalues) and is similar to convex optimization.

Given the preceding discussion, we believe VIs and monotone operator theory will serve as a strong foundation for deriving fundamental convergence results for GANs; this theory is
\begin{enumerate}[leftmargin=*]
    \item Similar to convexity suggesting its adoption by the GAN community should be smooth,
    \item Mature with natural mechanisms for handling constraints, subdifferentials, and online scenarios,
    \item Rich with algorithms with finite sample convergence for a hierarchy of monotone operators.
\end{enumerate}

Finally, we suggest~\cite{scutari2010convex} for a lucid comparison of convex optimization, game theory, and VIs.

\section{The Wasserstein Linear Quadratic GAN}
\label{sec:lqgan}
In the Wasserstein Linear-Quadratic GAN, the generator and discriminator are restricted to be linear and quadratic respectively: $G(z) = Az + b$ and $D(y) = y^\top  W_2 y + w_1^\top  y$. Equation~(\ref{eqn:Vorig}) becomes
\begin{align}
    \min_{A,b} \max_{W_2,w_1} \Big\{ V(W_2,w_1,A,b) = \mathbb{E}_{y \sim p(y)} [ D(y) ] - \mathbb{E}_{z \sim p(z)} [ D(G(z) ] \Big\}. \label{eqn:Vlqgan}
\end{align}
Let $\mathbb{E}[y]=\mu$, $\mathbb{E}[(y-\mu)^\top(y-\mu)]=\Sigma$, $\mathbb{E}[z]=0$, and $\mathbb{E}[z^2]=I$. If $A$ is constrained to be lower triangular with positive diagonal, i.e., of Cholesky form, then $(W_2^*,w_1^*,A^*,b^*) = (\mathbf{0},\mathbf{0},\Sigma^{1/2},\mu)$ is the unique minimax solution (see Proposition~\ref{multivarsoln}). The majority of this work focuses on the case where $p(y)$ and $p(z)$ are 1-d distributions. Equation~(\ref{eqn:Vlqgan}) simplifies to
\begin{align}
    \min_{a>0,b} \max_{w_2,w_1} \Big\{ V(w_2,w_1,a,b) = w_2 (\sigma^2 + \mu^2 - a^2 - b^2) + w_1 (\mu - b) \Big\}. \label{eqn:Vlqgan1d}
\end{align}
The map $F$ associated with this zero-sum game is constructed by concatenating the gradients of the two players' losses ($f_G = V, f_D = -V$):
\begin{align}
    F &= \begin{bmatrix}
    \frac{\partial f_D}{\partial w_2}, \frac{\partial f_D}{\partial w_1}, \frac{\partial f_G}{\partial a}, \frac{\partial f_G}{\partial b}
    \end{bmatrix}^\top = \begin{bmatrix}
    a^2 + b^2 - \sigma^2 - \mu^2, & b - \mu, & -2 w_2 a , & -2 w_2 b - w_1
    \end{bmatrix}^\top. \nonumber
\end{align}

\section{Crossing-the-Curl}
\label{sec:crosscurl}
In this section, we will derive our proposed technique, \emph{Crossing-the-Curl}, motivated by an examination of the ($w_1,b$)-subsystem of LQ-GAN, i.e., $(w_2,a)$ fixed at $(0,a_0)$ for any $a_0$. The results discussed here hold for the N-dimensional case as well. The map associated with this subsystem is plotted in Figure~\ref{fig:w1bplot} and formally stated in Equation~(\ref{eqn:Fw1b}).
\begin{figure}[htbp]
    \centering
    \begin{minipage}{0.45\textwidth}
        \centering
        \includegraphics[scale=0.35]{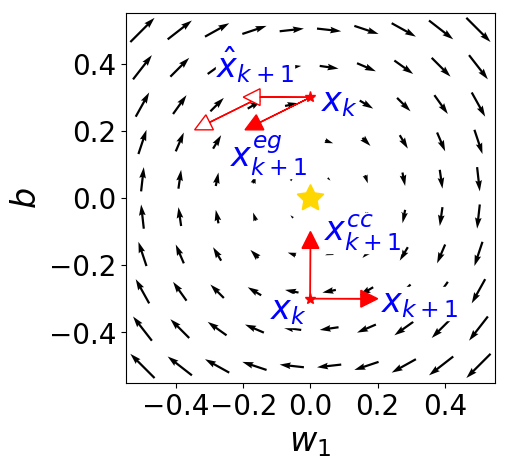} 
    \end{minipage}\hfill
    \begin{minipage}{0.5\textwidth}
        \begin{align}
            F^{w_1,b} &= [b - \mu, -w_1]^\top \label{eqn:Fw1b} \\
            J^{w_1,b} &= \begin{bmatrix}
                0 & 1 \\
                -1 & 0
            \end{bmatrix} \nonumber \\
            x_k &= [w_{1,k},b_k]^\top \nonumber \\
            x_{k+1} &= x_k - \rho_k F^{w_1,b}(x_k) \label{eqn:simgd}
        \end{align}
    \end{minipage}
    \caption{Vector field plot of $F^{w_1,b}$ for $\mu=0$ with extragradient, $x^{eg}_{k+1}$ (see updates~(\ref{EGstep1}) \&~(\ref{EGstep2})), simultaneous gradient descent, $x_{k+1}$, and \emph{Crossing-the-Curl}, $x^{cc}_{k+1}$, updates overlayed on top.}
    \label{fig:w1bplot}
\end{figure}

The Jacobian of $F^{w_1,b}$ is not Hurwitz, and simultaneous gradient descent, defined in Equation~(\ref{eqn:simgd}), will diverge for this problem (see~\ref{simgddiverge}). However, $F^{w_1,b}$ is monotone $(J+J^\top=\mathbf{0})$ and $1-$Lipschitz in the sense that $|| F^{w_1,b}(x)-F^{w_1,b}(x') ||^2 \le 1 ||x-x'||^2$. Table~\ref{tab:VIconv} offers an extragradient method (see Figure~\ref{fig:w1bplot}) with $\mathcal{O}(1/k)$ convergence rate, which is optimal for worst case monotone maps.

Nevertheless, an algorithm that travels perpendicularly to the vector field will proceed directly to the equilibrium. The intuition is to travel in the direction that is perpendicular to both $F$ and the axis of rotation. For a 2-d system, the axis of rotation can be obtained by taking the curl of the vector field. To derive a direction perpendicular to both $F$ and the axis of rotation, we can take their cross product:
\begin{flalign}
    F_{cc} = -\frac{1}{2}(\overbrace{\nabla \times F}^{\text{curl}}) \times F &= -\frac{1}{2}\{ \nabla_F (v \cdot F) - (v \cdot \nabla) F \} \Big\vert_{v=F} = -\Big(\frac{J-J^\top}{2} \Big)F = \begin{bmatrix} w_1 \\ b-\mu \end{bmatrix} &  \nonumber
\end{flalign}
where $\nabla_F$ is Feynman notation for the gradient with respect to $F$ only and $\vert_{v=F}$ means evaluate the expression at $v=F$. The~$\sfrac{-1}{2}$ factor ensures the algorithm moves toward regions of ``tighter cycles'' and simplifies notation. It may be sensible to perform some linear combination of simultaneous gradient descent and \emph{Crossing-the-Curl}, so we will refer to $(I-\eta(J-J^\top ))F$ as $F_{\eta cc}$.

Note that the fixed point of $F_{cc}$ remains the same as the original field $F$. Furthermore, the reader may recognize $F_{cc}$ as the gradient of the function $\frac{1}{2}(w_1^2 + (b-\mu)^2)$, which is strongly convex, allowing an $\mathcal{O}(e^{-k})$ convergence rate in the deterministic setting. $F_{cc}$ is derived from intuition in 2-d, however, we discuss reasons in the next subsection for why this approach generalizes to higher dimensions.

\subsection{Discussion \& Relation to Other Methods}
For the ($w_1,b$)-subsystem, \emph{Crossing-the-Curl} is equivalent to two other methods: the consensus algorithm~\cite{mescheder2017numerics} and a Taylor series approximation to extragradient~\cite{korpelevich}.
\begin{figure}[htbp]
    \centering
    \begin{minipage}{0.55\textwidth}
        \begin{align}
            \hat{x}_{k+1} &= x_k - \eta F(x_k) \label{EGstep1} \\
            x_{k+1} &= x_k - \rho F(\hat{x}_{k+1}) \label{EGstep2} \\
            &= x_k - \rho \underbrace{(I - \eta J(x_k)) F(x_k)}_{F_{eg}} + \mathcal{O}(\rho \eta^2) \label{eqn:tayeg}
        \end{align}
    \end{minipage}\hfill
    \begin{minipage}{0.45\textwidth}
        \begin{align}
            x_{k+1} &= x_k - \rho (F(x_k) + \eta \nabla ||F||^2) \nonumber \\
            &= x_k - \rho \underbrace{(I + \eta J^\top (x_k)) F(x_k)}_{F_{con}} \label{eqn:con}
        \end{align}
    \end{minipage}
    \caption{A Taylor series expansion of extragradient~(\ref{eqn:tayeg}) and the consensus algorithm~(\ref{eqn:con}).}
\end{figure}

These equivalences occur because the Jacobian is skew-symmetric ($J^\top$$=$$-J$) for the ($w_1,b$)-subsystem. In the more general case, where $J$ is not necessarily skew-symmetric, \emph{Crossing-the-Curl} represents a combination of the two techniques. Extragradient (EG) is key to solving VIs and the consensus algorithm has delivered impressive results for GANs, so this is promising for $F_{cc}$. To our knowledge, $F_{eg}$ is novel and has not appeared in the Variational Inequality literature.

\emph{Crossing-the-Curl} stands out in many ways though. Observe that in higher dimensions, the subspace orthogonal to $F$ is $(n-1)$ dimensional, which means $(J^\top$$-$$J)F$ is no longer the unique direction orthogonal to $F$. However, every matrix can be decomposed into a symmetric part with real eigenvalues, $\sfrac{1}{2}(J+J^\top)$, and a skew-symmetric part with purely imaginary eigenvalues, $\sfrac{1}{2}(J-J^\top)$. Notice that for an optimization problem, $J$$-$$J^\top$$=$$H$$-$$H^\top$$=$$0$ where $H$ is the Hessian.\footnote{Assuming the objective function has continuous second partial derivatives\textemdash see Schwarz's theorem.} It is the imaginary eigenvalues, i.e., rotation, that set equilibrium problems apart from optimization and necessitate the development of new algorithms like extragradient. It is reassuring that this matrix appears explicitly in $F_{cc}$. In addition, $F_{cc}$ reduces to gradient descent when applied to an optimization problem making the map agnostic to the type of problem at hand: optimization or equilibration.

The curl also shares close relation to the gradient. The gradient is $\nabla$ applied to a scalar function and the curl is $\nabla$ crossed with a vector function. Furthermore, under mild conditions, every vector field, $F: \mathbb{R}^3 \rightarrow \mathbb{R}^3$, admits a Helmholdtz decomposition: $F=-\nabla f + \nabla \times G$ where $f$ is a scalar function and $G$ is a vector function suggesting the gradient and curl are both fundamental components.

Consider the perspective of $F_{cc}$ as preconditioning $F$ by a skew-symmetric matrix. Preconditioning with a positive definite matrix dates back to Newton's method and has reappeared in machine learning with natural gradient~\cite{amari1998natural}. \citet{dafermos1983iterative} considered asymmetric positive definite preconditioning matrices for VIs. \citet{thomas2014genga} extended the analysis of natural gradient to PSD matrices. We are not aware of any work using skew-symmetric matrices for preconditioning. The scalar $x^\top A x \equiv 0$ for any skew-symmetric matrix $A$, so calling $(J^\top-J)$ a PSD matrix is not adequately descriptive.

Note that \emph{Crossing-the-Curl} does not always improve convergence; this technique can transform a strongly-monotone field into a saddle and an unstable fixed point (non-monotone) into a strongly-monotone field (see~\ref{ccmon2nonmon} for examples), so this technique should generally be used with caution.

Lastly, \emph{Crossing-the-Curl} is inexpensive to compute. The Jacobian-vector product, $JF$, can be approximated accurately and efficiently with finite differences. Likewise, $J^\top F$ can be computed efficiently with double backprop~\cite{drucker1992improving} by taking the gradient of $\sfrac{1}{2}||F||^2$. In total, three backprops are required, one for $F(x_k)$, one for $F(\hat{x}_{k+1})$, and one for $\sfrac{1}{2}||F(x_k)||^2$.

In our analysis, we also consider the gradient regularization proposed in~\cite{nagarajan2017gradient}, $F_{reg}$, the Unrolled GAN proposed in~\cite{metz2016unrolled}, $F_{unr}$, alternating gradient descent, $F_{alt}$, as well as any linear combination of $F$, $JF$, and $J^\top F$, deemed $F_{lin}$, which forms a family of maps that includes $F_{eg}$, $F_{con}$, and $F_{cc}$:
\begin{gather}
    F_{reg} = \begin{bmatrix}
    F_{D}; & F_{G} + \eta \nabla_G ||F_{D}||^2
    \end{bmatrix}^\top,
    \hspace{2.0cm}
    F_{lin} = (\rho I + \beta J^\top - \gamma J) F.  \nonumber
\end{gather}

Keep in mind that we are proposing $F_{lin}$ as a generalization of \emph{Crossing-the-Curl}. We state our main results here for the $(w_1,b)$-subsystem.
\begin{proposition}
\label{Flinw1bstrong_body}
For any $\alpha$, $F^{w_1,b}_{lin}$ with at least one of $\beta$ and $\gamma$ positive and both non-negative is strongly monotone. Also, its Jacobian is Hurwitz. See Proposition~\ref{w1bmonotone}.
\end{proposition}
\begin{corollary}
$F^{w_1,b}_{cc}$, $F^{w_1,b}_{\eta cc}$, $F^{w_1,b}_{eg}$, and $F^{w_1,b}_{con}$ with $\eta>0$ are strongly-monotone with Hurwitz Jacobians. See Proposition~\ref{Flinw1bstrong_body}.
\end{corollary}
\begin{proposition}
$F^{w_1,b}_{alt}$, $F^{w_1,b}_{unr}$, $F^{w_1,b}$, and $F^{w_1,b}_{reg}$ with any $\eta$ are monotone, but not strictly monotone. Of these maps, only $F^{w_1,b}_{reg}$'s Jacobian is Hurwitz. See Propositions~\ref{unrolledalt} and~\ref{w1bmonotone}.
\end{proposition}





\section{Analysis of the Full System}
\label{sec:fullanalysis}
Here, we analyze the maps for each of the algorithms discussed above, testing for quasimonotonicity (the weakest monotone property) and whether the Jacobian is Hurwitz for the full LQ-GAN system.

Proving quasiconvexity of 4th degree polynomials has been proven strongly NP-Hard~\cite{ahmadi2013np}. This implies that proving monotonicity of 3rd degree maps is strongly NP-Hard. The original $F$ contains quadratic terms suggesting it may welcome a quasimonotone analysis, however, the remaining maps all contain 3rd degree terms. Unsurprisingly, analyzing quasimonotonicity for $F_{lin}$ represents the most involved of our proofs given in Appendix~\ref{subsec:lincombnotquasi}.

The definition stated in~(\ref{def:quasimon}) suggests checking the truth of an expression depending on four separate variables: $x$, $x'$, $y$, $y'$. While we used this definition for certain cases, the following alternate requirements proposed in~\cite{crouzeix1996criteria} made the complete analysis of the system tractable. We restate simplified versions of the requirements we leveraged for convenience.

Consider the following conditions:
\begin{enumerate}[label=(\Alph*)]
    \item For all $x \in \mathcal{X}$ and $v \in \mathbb{R}^n$ such that $v^\top F(x) = 0$ we have $v^\top J(x) v \ge 0$.
    \item For all $x \in \mathcal{X}$ and $x^* \in \mathcal{X}$ such that $F(x^*)=0$, we have that $F(x)^\top (x-x^*) \ge 0$.
\end{enumerate}
\begin{theorem}[\cite{crouzeix1996criteria}, Theorem 3]
Let $F: \mathcal{X} \rightarrow \mathbb{R}^n$ be differentiable on the open convex set $\mathcal{X} \subset \mathbb{R}^n$.
\begin{enumerate}[leftmargin=*,label=(\roman*)]
    \item $F$ is quasimonotone on $\mathcal{X}$ only if (A) holds, i.e. (A) is necessary but not sufficient.
    \item $F$ is pseudomonotone on $\mathcal{X}$ if (A) and (B) hold, i.e. (A) and (B) are sufficient but not necessary.
\end{enumerate}
\end{theorem}

Condition (A) says that for a map to be quasimonotone, the map must be monotone along directions orthogonal to the vector field. In addition to this, condition (B) says that for a map to be pseudomonotone, the dynamics, $-F$, must not be leading away from the equilibrium anywhere.

Equipped with these definitions, we can conclude the following:
\begin{proposition}
\label{nonequasi}
None of the maps, including $F_{lin}$ with any setting of coefficients, is quasimonotone for the full LQ-GAN. See Corollary~\ref{lincombnotquasi} and Propositions~\ref{Fregnotquasi} through~\ref{Faltnotquasi}.
\end{proposition}
\begin{proposition}
\label{nonehurwitz}
None of the maps, including $F_{lin}$ with any setting of coefficients, has a Hurwitz Jacobian for the full LQ-GAN. See Propositions~\ref{lincombnothurwitz} and~\ref{Fregnotquasi} through~\ref{Faltnotquasi}.
\end{proposition}

\subsection{Learning the Variance: The ($w_2,a$)-Subsystem}
\label{subsec:var}
Results from the previous section suggest that we cannot solve the full LQ-GAN, but given that we can solve the ($w_1,b$)-subsystem, we shift focus to the ($w_2,a$)-subsystem assuming the mean has already been learned exactly, i.e., $b=\mu$. We will revisit this assumption later.

We can conclude the following for the ($w_2,a$)-subsystem:
\begin{proposition}
$F^{w_2,a}$, $F^{w_2,a}_{reg}$, $F^{w_2,a}_{unr}$, $F^{w_2,a}_{alt}$, and $F^{w_2,a}_{con}$ are not quasimonotone. Also, their Jacobians are not Hurwitz. See Propositions~\ref{Fnotquasi} through~\ref{consensusnotquasi}.
\end{proposition}
\begin{proposition}
$F^{w_2,a}_{eg}$ and $F^{w_2,a}_{cc}$ are pseudomonotone which implies an $\mathcal{O}(1/\sqrt{k})$ stochastic convergence rate. See Propositions~\ref{Fegpseudo} and~\ref{Fccpseudow2a}. Their Jacobians are not Hurwitz. See Proposition~\ref{lincombnothurwitz}.
\end{proposition}
\begin{proposition}
\label{noflin}
No monotone $F^{w_2,a}_{lin}$ exists. See Proposition~\ref{w2anotmonotone}.
\end{proposition}
%
These results are not purely theoretical. Figure~\ref{fig:trajcomp} displays trajectories resulting from each of the maps.
\begin{figure}[htbp]
    \centering
    \begin{minipage}{0.5\textwidth}
        \centering
        \includegraphics[scale=0.45]{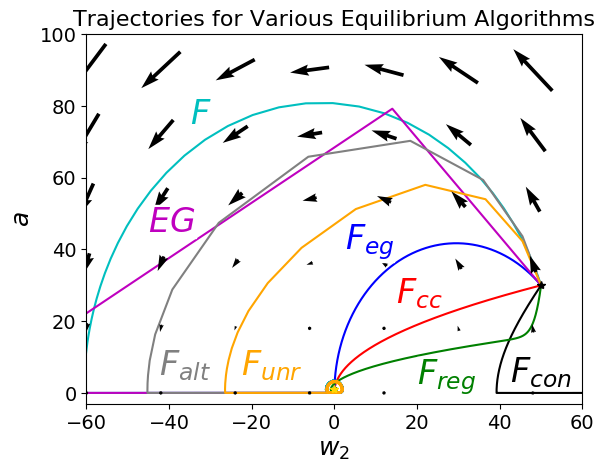}
        \end{minipage}\hfill
    \begin{minipage}{0.45\textwidth}
        \begin{align}
            F^{w_2,a}_{eg} &= \begin{bmatrix}
            4w_2a^2 \\ 2a(a^2-\sigma^2) - 4 w_2^2 a
            \end{bmatrix} \label{eqn:Fegw2a} \\
            &\downarrow \ast\sfrac{1}{4a^2} \nonumber \\
            F^{w_2,a}_{eg'} &= \begin{bmatrix}
            w_2 \\ \frac{a^2-\sigma^2 - 2w_2^2}{2a}
            \end{bmatrix} \label{eqn:Feg'w2a}; \\
            F^{w_2,a}_{cc} &= \begin{bmatrix}
            4w_2a^2 \\ 2a(a^2-\sigma^2)
            \end{bmatrix} \label{eqn:Fccw2a} \\
            &\downarrow \ast\sfrac{1}{4a^2} \nonumber \\
            F^{w_2,a}_{cc'} &= \begin{bmatrix}
            w_2 \\ \frac{a^2-\sigma^2}{2a}
            \end{bmatrix} \label{eqn:Fcc'w2a}
        \end{align}
    \end{minipage}
    \caption[Caption for Trajectories Figure]{(Left) Comparison of trajectories on the ($w_2,a$)-subsystem.\footnotemark[4] The vector field plotted is for the original system, $\dot{x}=-F^{w_2,a}(x)$. Observe how $F^{w_2,a}_{cc}$ takes a more direct route to the equilibrium. (Right) Maps derived after rescaling $F^{w_2,a}_{cc}$ and $F^{w_2,a}_{eg}$.}
    \label{fig:trajcomp}
\end{figure}
\footnotetext[4]{ODEs were simulated using Heun-Euler with Phase Space Error Control~\cite{higham2000phase}.}

We can further improve upon $F^{w_2,a}_{eg}$ and $F^{w_2,a}_{cc}$ by rescaling with $\sfrac{1}{4a^2}$: (\ref{eqn:Fegw2a})$\rightarrow$(\ref{eqn:Feg'w2a}) and (\ref{eqn:Fccw2a})$\rightarrow$(\ref{eqn:Fcc'w2a}) respectively. This results in strongly-monotone and strongly-convex systems respectively, improving the stochastic convergence rate to $\mathcal{O}(1/k)$. In deriving these results, we assumed the mean was given. We can relax this assumption and analyze the ($w_2,a$)-subsystem under the assumption that the mean is ``close enough''. Using a Hoeffding bound, we find that $k > \big( \frac{y_{hi}-y_{low}}{-|\mu|+\sqrt{\mu^2+d\sigma^2}} \big)^2 \log[\frac{\sqrt{2}}{\delta^{1/2}}]$ iterations of $F^{w_1,b}_{cc}$ are required to achieve a $1-\delta$ probability of the mean being accurate enough to ensure the ($w_2,a$)-subsystem is strongly-monotone. Note that this approach of first learning the mean, then the variance retains the overall $\mathcal{O}(1/k)$ stochastic rate. We summarize the main points here.
\begin{claim}
A nonlinear scaling of $F^{w_2,a}_{eg}$ and $F^{w_2,a}_{cc}$ results in strictly monotone and $\sfrac{1}{2}$-strongly monotone subsystems respectively. See Proposition~\ref{Fw2astrongccstricteg}.
\end{claim}
\begin{claim}
If the mean is first well approximated, i.e., $b^2 \le \mu^2 + \sigma^2$, then $F^{w_2,a}_{cc'}$ remains 1) $\sfrac{1}{2}$-strongly-monotone if the ($w_1,b$)-subsystem is ``shut off'' or 2) strictly-monotone if the $(w_1,b)$-subsystem is re-weighted with a high coefficient. See Propositions~\ref{progshutoff} and~\ref{progongoing}.
\end{claim}
\begin{proposition}
$F^{W_2,A}_{eg}$ and $F^{W_2,A}_{cc}$ are not quasimonotone for the 2-d LQ-GAN system (with and without $(AA^\top)^{-1}$ scaling). See Proposition~\ref{multivarnotmonotone}.
\end{proposition}
%
Several takeaways emerge. One is that the stability of the system is highly dependent on the mean first being learned. In other words, \emph{batch norm} is required for the monotonicity of LQ-GAN, so it is not surprising that GANs typically fail without these specialized layers.

Second is that stability is achieved by first learning a simple subsystem, ($w_1,b$), then learning the more complex, ($w_2,a$)-subsystem. This theoretically confirms the intuition behind progressive training of GANs~\cite{karras2017progressive}, which have generated the highest quality images to date. 

Thirdly, because $J^{cc'}_{w_2,a}$ is symmetric (and $\succ 0$), we can integrate $F^{cc'}_{w_2,a}$ to discover the convex function it is implicitly descending via gradient descent: $f^{cc'}_{w_2,a}=1/2 [(a^2-\sigma^2) - \sigma^2 \log(\sfrac{a^2}{\sigma^2})]$. Compare this to KL-divergence: $KL(\sigma||a)=1/2[(\sigma^2/a^2) + \log(\sfrac{a^2}{\sigma^2}) - 1]$. In contrast to $KL$, $f^{cc'}_{w_2,a}$ is convex in $a$ and may be a desirable alternative due to less extreme gradients near $a=0$.

\begin{table}[htbp]
    \centering
    \caption{For convenience, we summarize many of our theoretical results in this table. Legend: $M$=Monotone, $C$=Convex, $H$=Hurwitz, $S$=Strongly, $s$=Strictly, $P$=Pseudo, $Q$=Quasi, $/$=Not.}
    \begin{tabular}{c|c|c|c|c|c|c|c|c|c}
        Subsystem & $F$ & $F_{alt}$ & $F_{unr}$ & $F_{reg}$ & $F_{con}$ & $F_{eg}$ & $F_{cc}$ & $F_{eg'}$ & $F_{cc'}$ \\ \hline
        ($w_1,b$) & $\text{M}$,$\cancel{\text{H}}$ & $\text{M}$,$\cancel{\text{H}}$ & $\text{M}$,$\cancel{\text{H}}$ & $\text{M}$,$\text{H}$ & $\text{SC}$,$\text{H}$ & $\text{SC}$,$\text{H}$ & $\text{SC}$,$\text{H}$ & NA & NA \\ \hline
        ($w_2,a$) & $\cancel{\text{QM}}$,$\cancel{\text{H}}$ & $\cancel{\text{QM}}$,$\cancel{\text{H}}$ & $\cancel{\text{QM}}$,$\cancel{\text{H}}$ & $\cancel{\text{QM}}$,$\cancel{\text{H}}$ & $\cancel{\text{QM}}$,$\cancel{\text{H}}$ & $\text{PM}$,$\cancel{\text{H}}$ & $\text{PM}$,$\cancel{\text{H}}$ & $\text{sM}$,$\text{H}$ & $\text{SC}$,$\text{H}$
    \end{tabular}
    \label{tab:monsummary}
\end{table}
%
\subsection{Learning the Covariance: The ($W_2,A$)-Off-Diagonal Subsystem}
\label{subsec:covar}
After learning both the mean and variance of each dimension, the covariance of separate dimensions can be learned. Proposition~\ref{prop:covar} in the Appendix states that the subsystem relevant to learning each row of $A$ is strictly monotone when all other rows are held fixed. In fact, the maps for these subsystems are affine and skew-symmetric just like the ($w_1,b$)-subsystem. This implies that \emph{Crossing-the-Curl} applied successively to each row of $A$ can solve for $A^*$; pseudocode is presented in Algorithm~\ref{alg:cc_lqgan} in Appendix~\ref{sub:cc_lqgan}. Note that this procedure is reminiscent of the Cholesky–Banachiewicz algorithm which computes $A$ row by row, beginning with the first row. The resulting algorithm is $\mathcal{O}(N/k)$.
\section{Experiments}
\label{sec:exp}
Our theoretical analysis proves convergence of the stagewise procedure using \emph{Crossing-the-Curl} for the N-d LQGAN. Experiments solving the ($w_2,a$)-subsystem alone for randomly generated $\mathbb{E}[(y-\mu)^2]=\sigma^2$ support the analysis of Subsection~\ref{subsec:var}\textemdash see the first row of Table~\ref{tab:lqganexp}. Not listed in the first row of the table are $F_{cc'}$ and $F_{eg'}$ which converge in $32$ and $33$ steps on average respectively with a constant step size of $0.1$. Our novel maps, $F_{cc}$ and $F_{eg}$, converge in a quarter of the iterations of the next best method ($F_{reg}$), and $F_{cc'}$ and $F_{eg'}$ in nearly a quarter of their parent counterparts. These experiments used analytical results of the expectations, i.e., the systems are deterministic.
\begin{table}[htbp]
    \centering
    \caption{Each entry in the table reports two quanities. First is the average number of steps, $k$, required for each dynamical system, e.g., $\dot{x}=-F(x)$, to reduce $||x_k-x^*||/||x_0-x^*||$ to $0.001$ for the ($W_2,A$)-subsystem. The second, in parentheses, reports the fraction of trials that the algorithm met this threshold in under 100,000 iterations. Dim denotes the dimensionality of $y \sim p(y)$ for the LQ-GAN being trained (with $|\theta|+|\phi|$ in parentheses). For each problem, $x_0$ is randomly initialized 10 times for each of ten randomly initialized $\Sigma$'s, i.e., 100 trials per cell. Extragradient (EG) is run with a fixed step size. All other ODEs are solved via Heun-Euler with Phase Space Error Control~\cite{higham2000phase}.}
    \begin{tabular}{c|c|c|c|c|c|c}
        Dim & $F$ & $EG$ & $F_{con}$ & $F_{reg}$ & $F_{eg}$ & $F_{cc}$ \\ \hline
        1 (2) & $10^5$ (0) & $83315$ (0.4) & $6354$ (0.94) & $395$ (1) & $116$ (1) & $\mathbf{110}$ (1)  \\ \hline
        2 (6) & $10^5$ (0) & $98244$ (0.05) & $33583$ (0.68) & $2595$ (1) & $\mathbf{1321}$ (1) & $1441$ (1)  \\ \hline
        4 (10) & $10^5$ (0) & $99499$ (0.01) & $77589$ (0.23) & $\mathbf{33505}$ (0.7) & $34929$ (0.67) & $34888$ (0.68)
    \end{tabular}
    \label{tab:lqganexp}
\end{table}

The second and third rows of the table reveal that convergence slows considerably for higher dimensions. However, the stagewise procedure discussed in Subsection~\ref{subsec:covar} is guaranteed to converge given the mean has been learned to a given accuracy. This procedure solves the 4-d \emph{deterministic} LQ-GAN in $\mathbf{20549}$ iterations with a $\mathbf{0.88}$ success rate. For the 4-d \emph{stochastic} LQ-GAN using two-sample minibatch estimates, this procedure achieves $||x_k-x^*||/||x_0-x^*||<0.1$ in 100,000 iterations with a 0.75 success rate.


\section{Conclusion}
In this work, we performed the first global convergence analysis for a variety of GAN training algorithms. According to Variational Inequality theory, none of the current GAN training algorithms is globally convergent for the LQ-GAN. We proposed an intuitive technique, \emph{Crossing-the-Curl}, with the first global convergence guarantees for any generative adversarial network. As a by-product of our analysis, we extract high-level explanations for why the use of \emph{batch norm} and progressive training schedules for GANs are critical to training. In experiments with the multivariate LQ-GAN, \emph{Crossing-the-Curl} achieves performance superior to any existing GAN training algorithm.

For future work, we will investigate alternate parameterizations of the discriminator such as $D(y)=w_2(y-w_1)^2$. We will also work on devising heuristics for setting the coefficients of $F_{lin}$.

\section{Acknowledgments}
\emph{Crossing-the-Curl} was independently proposed in~\cite{balduzzi2018mechanics} called \emph{Symplectic Gradient Adjustment} (SGA). Like \emph{Crossing-the-Curl}, this algorithm is motivated by attacking the challenges of rotation in differentiable games, however, it is derived by performing gradient descent on the Hamiltonian as opposed to generalizing a particular perpendicular direction selected from intuition in 2-d. Given the equivalence between SGA and \emph{Crossing-the-Curl}, our work can also be viewed as proving that a non-trivial application of this algorithm can be used to solve the LQ-GAN. On the other hand, we have also proven in Proposition~\ref{noflin} that a naive application of this algorithm is insufficient for solving LQ-GAN suggesting more research is required to understand and more efficiently solve this complex problem.

\newpage
\bibliography{bib/bib}
\bibliographystyle{plainnat}

\newpage
\tableofcontents
\addtocontents{toc}{\protect\setcounter{tocdepth}{2}}
\appendix

\section{Appendix}

\subsection{A Survey of Candidate Theories Continued}

\subsubsection{Algorithmic Game Theory}
\label{AGT}
Algorithmic Game Theory (AGT) offers results on convergence to equilibria when a game, possibly online, is convex~\cite{gordon2008no}, socially-convex~\cite{even2009convergence}, or smooth~\cite{roughgarden2009intrinsic}. A convex game is one in which all player losses are convex in their respective variables, i.e. $f_i(x_i,x_{-i})$ is convex in $x_i$. A socially-convex game adds the additional requirements that 1) there exists a strict convex combination of the player losses that is convex and 2) each player's loss is concave in the variables of each of the other players. In other words, the players as a whole are cooperative, yet individually competitive. Lastly, smoothness ensures that ``the externality imposed on any one player by the actions of the others is bounded''~\cite{roughgarden2009intrinsic}. In a zero-sum game such as~(\ref{eqn:Vorig}), one player's gain is exactly the other player's loss making smoothness an unlikely fit for studying GANs. See~\cite{gemp2017online} for examples where the three properties above overlap with monotonicity in VIs.

\subsubsection{Differential Games}
\label{DiffGames}
Differential games~\cite{basar1999dynamic,friesz2010dynamic} consider more general dynamics such as $\ddot{x}=-F(x)$, not just first order ODEs, however, the focus is on systems that separate control, $u$, and state $x$, i.e. $\dot{x} = -F(x(t),u(t),t)$. More specific to our interests, Differential Nash Games can be expressed as Differential VIs, a specific class of infinite dimensional VIs with explicit state dynamics and explicit controls; these, in turn, can be framed as infinite dimensional VIs without an explicit state.

\subsection{Nash Equilibrium vs VI Solution}
\label{vine}
\begin{theorem}
Repeated from~\cite{cavazzuti2002nash}. Let $(\mathbf{C},K)$ be a cost minimization game with player cost functions $C_i$ and feasible set $K$. Let $x^*$ be a Nash equilibrium. Let $F = [\frac{\partial C_1}{\partial x_1},\ldots,\frac{\partial C_N}{\partial x_N}]$. Then
\begin{align}
&\langle F(x^*), x - x^* \rangle \ge 0 \\
&\forall x \in (\{ x^* + \mathbf{I}_K(x^*) \} \cap K) \subseteq K
\end{align}
\end{theorem}
where $\mathbf{I}_K(x^*)$ is the internal cone at $x^*$. When $C_i(\mathbf{x}_i,\mathbf{x}_{-i})$ is pseudoconvex in $\mathbf{x}_i$ for all $i$, this condition is also sufficient. Note that this is implied if $F$ is pseudomonotone, i.e. pseudomonotonicity of $F$ is a stronger condition.

\subsection{Table of Maps Considered in Analysis}
\begingroup
\renewcommand{\arraystretch}{1.5} 
\begin{table}[htbp]
    \centering
    \begin{tabular}{c|c}
        Name & Map \\
        $F$ & $[-\nabla_{\phi} V ; \nabla_{\theta} V]$ \\ \hline
        $F^{w_1,b}$ & $[b-\mu, -w_1]^\top$ \\
        $F^{w_1,b}_{alt}$ & $[b-\mu + \rho_k w_1, -w_1]^\top$ \\
        $F^{w_1,b}_{unr}$ & $[b-\mu, \rho_k \Delta k (b-\mu) - w_1]^\top$ \\
        $F^{w_1,b}_{reg}$ & $[b-\mu, -w_1 + 2\eta (b-\mu)]^\top$ \\
        $F^{w_1,b}_{con}$ & $[w_1, b-\mu]^\top$ \\
        $F^{w_1,b}_{eg}$ & $[w_1, b-\mu]^\top$ \\
        $F^{w_1,b}_{cc}$ & $[w_1, b-\mu]^\top$ \\
        $F^{w_1,b}_{\eta cc}$ & $[b-\mu + \eta w_1, -w_1 + \eta(b-\mu)]^\top$ \\ 
        $F^{w_1,b}_{lin}$ & $[\alpha(b-\mu)+(\beta+\gamma)w_1, -\alpha w_1 + (\beta+\gamma)(b-\mu)]^\top$ \\ \hline
        $F^{w_2,a}$ & $[a^2 - \sigma^2, -2w_2 a]^\top$ \\
        $F^{w_2,a}_{alt}$ & $[a^2-\sigma^2, 2\rho_k a^3 - 2a(\rho_k \sigma^2 + w_2)]^\top$ \\
        $F^{w_2,a}_{unr}$ & $[a^2-\sigma^2, 4\rho_k \Delta k a^3 - 2a(2\rho_k \Delta k \sigma^2 + w_2)]^\top$ \\
        $F^{w_2,a}_{reg}$ & $[a^2 - \sigma^2, -2w_2 a + 4\eta a(\sigma^2 + a^2)]^\top$ \\
        $F^{w_2,a}_{con}$ & $[a^2-\sigma^2 + 4 \beta w_2 a^2, 2a\beta (a^2-\sigma^2) + 4\beta w_2^2 a - 2w_2 a]^\top$ \\
        $F^{w_2,a}_{eg}$ & $[4 w_2 a^2, 2a(a^2-\sigma^2) - 4 w_2^2 a]^\top$ \\
        $F^{w_2,a}_{cc}$ & $[4 w_2 a^2, 2a (a^2-\sigma^2)]^\top$ \\
        $F^{w_2,a}_{eg'}$ & $[w_2, \frac{a^2-\sigma^2-2w_2^2}{2a}]^\top$ \\
        $F^{w_2,a}_{cc'}$ & $[w_2, \frac{a^2-\sigma^2}{2a}]^\top$ \\
        $F^{w_2,a}_{lin}$ & $[\alpha(a^2-\sigma^2) + 4(\beta+\gamma) w_2 a^2, 2a(\beta+\gamma)(a^2-\sigma^2) + 4(\beta-\gamma)w_2^2a - 2\alpha w_2 a]^\top$ \\ \hline
        $F^{W_2,A}_{cc}$ & $2[\forall i<N: \sum_{d \le i} A_{id} A_{Nd} - \Sigma_{iN} \,\,, \forall i<N:  -\sum_{d < N} A_{di} W_{dN}]^\top$
    \end{tabular}
    \vspace{0.2cm}
    \caption{Table of vector field maps where $V$ is the minimax objective, $\rho_k$ is a stepsize, $\Delta k$ is \# of \emph{unrolled} steps, $\Sigma$ is the sample covariance matrix, $N$ is the row of $A$ being learned, and $\alpha, \gamma, \beta, \eta$ are hyperparameters. Notice that all maps require an unbiased estimate of the mean ($\mu$) or the variance ($\sigma^2, \Sigma$) of $p(y)$ which can be obtained with one (mean) or two (variance) samples.}
    \label{tab:maps}
\end{table}
\endgroup

All maps corresponding to the ($w_1,b$)-subsystem in Table~\ref{tab:maps} maintain the desired unique fixed point, $F(x^*)=0$, where $x^*=(w_1^*,b^*)=(0,\mu)$.

For the ($w_2,a$)-subsystem, all maps except $F_{lin}$ with certain settings of ($\alpha,\beta,\gamma$) and $F_{con}$ maintain the desired unique fixed point, $x^*=(w_2^*,a^*)=(0,\sigma)$. $F_{con}$ introduces an additional spurious fixed point at
\begin{align}
    a &= \sqrt{\frac{-3+\sqrt{9+32\sigma^2\beta^2}}{16\beta^2}}, \\
    w_2 &= \frac{\sigma^2-a^2}{4\beta a^2}.
\end{align}
$F_{con}$ is a special case of $F_{lin}$ where $\alpha=1$, $\beta=1$, and $\gamma=0$.

\subsection{Minimax Solution to Constrained Multivariate LQ-GAN is Unique}
\begin{proposition}
\label{multivarsoln}
Assume $z \sim p(z)$ and $y \sim p(y)$ are both in $\mathbb{R}^n$. If $W_2$ is constrained to be symmetric and $A$ is constrained to be of Cholesky form, i.e., lower triangular with positive diagonal, then the unique minimax solution to Equation~(\ref{eqn:Vlqgan1d}) is $(W_2^*,w_1^*,A^*,b^*) = (\mathbf{0},\mathbf{0},\Sigma^{1/2},\mu)$ where $\Sigma^{1/2}$ is the unique, non-negative square root of $\Sigma$.
\end{proposition}
\begin{proof}
\begin{align}
    V(G,D) &= \mathbb{E}_{y \sim \mathcal{N}(\mu,\Sigma)} \Big[ y^\top W_2 y + w_1^\top y \Big] + \mathbb{E}_{z \sim \mathcal{N}(0,I_n)} \Big[ -(Az+b)^\top W_2 (Az+b) - w_1^\top (Az+b) \Big] \\
    &= \mathbb{E}_{y \sim \mathcal{N}(\mu,\Sigma)} \Big[ \sum_i \sum_j W_{2ij} y_i y_j + \sum_i w_{1i} y_i \Big] \\
    &- \mathbb{E}_{z \sim \mathcal{N}(0,I_n)} \Big[ \sum_i \sum_j W_{2ij} (b_i + \sum_k A_{ik} z_k)(b_j + \sum_k A_{jk} z_k) + \sum_i w_{1i} (b_i + \sum_k A_{ik} z_k) \Big]
\end{align}




Taking derivatives and setting equal to zero, we find that the fixed point at the interior is unique.

\begin{align}
\dot{W_2} &= \mathbb{E}_{y \sim \mathcal{N}(\mu,\Sigma)} \Big[ yy^\top \Big] - \mathbb{E}_{z \sim \mathcal{N}(0,I_n)} \Big[ (Az+b)(Az+b)^\top \Big] \\
\dot{w_1} &= \mathbb{E}_{y \sim \mathcal{N}(\mu,\Sigma)} \Big[ y \Big] - \mathbb{E}_{z \sim \mathcal{N}(0,I_n)} \Big[ (Az+b) \Big] \\
\dot{A} &= \mathbb{E}_{z \sim \mathcal{N}(0,I_n)} \Big[ (W_2+W_2^\top)Azz^\top + (W_2+W_2^\top)bz^\top + w_1z^\top \Big] \\
\dot{b} &= \mathbb{E}_{z \sim \mathcal{N}(0,I_n)} \Big[ (W_2+W_2^\top)Az + (W_2+W_2^\top)b + w_1 \Big]
\end{align}

\begin{align}
\dot{w_1} &= \mu - b = 0 \Rightarrow b = \mu \\
\dot{W_2} &= \mathbb{E}_{y \sim \mathcal{N}(\mu,\Sigma)} \Big[ (y-\mu)(y-\mu)^\top + \mu y^\top + y \mu^\top - \mu \mu^\top \Big] - \mathbb{E}_{z \sim \mathcal{N}(0,I_n)} \Big[ (Az+b)(Az+b)^\top \Big] \\
&= \Sigma + \mu \mu^\top - \mathbb{E}_{z \sim \mathcal{N}(0,I_n)} \Big[ A zz^\top A^\top + Azb^\top + b(Az)^\top + bb^\top \Big] \\
&= \Sigma + \mu \mu^\top - AA^\top - bb^\top = \Sigma - AA^\top = 0 \Rightarrow A = \Sigma^{1/2} \label{Aunique} \\
\dot{A} &= \mathbb{E}_{z \sim \mathcal{N}(0,I_n)} \Big[ (W_2+W_2^\top)Azz^\top + (W_2+W_2^\top)bz^\top + w_1z^\top \Big] \\
&= (W_2+W_2^\top)A = 0 \Rightarrow W_2 + W_2^\top = 0 \Rightarrow W_2=-W_2^\top = 0 \label{W2unique} \\
\dot{b} &= \mathbb{E}_{z \sim \mathcal{N}(0,I_n)} \Big[ (W_2+W_2^\top)Az + (W_2+W_2^\top)b + w_1 \Big] \\
&= (W_2+W_2^\top)b + w_1 = w_1 = 0
\end{align}

The last implication in Equation~(\ref{Aunique}) follows because $A$ is constrained to be of Cholesky form, i.e., lower triangular with positive diagonal, and every symmetric positive definite matrix has a unique Cholesky decomposition.

The second to last implication of Equation~(\ref{W2unique}) follows because $A=\Sigma^{1/2}$ is necessarily full rank. Note this implies $A^\top$ is also full rank. The null space of a full rank matrix is the zeros vector, which implies $W_2 + W_2^\top = 0$. $W_2$ is symmetric, so this implies $W_2=0$.
\end{proof}

\subsection{Divergence of Simultaneous Gradient Descent for the ($w_1,b$)-Subsystem}
\label{simgddiverge}
Consider the case where the mean of $p(z)$ is zero:
\begin{align}
    F^{w_1,b} &= [b, -w_1] = J^{w_1,b} x, \\
    J^{w_1,b} &= \begin{bmatrix}
        0 & 1 \\
        -1 & 0
    \end{bmatrix}, \\
    x_k &= [w_{1,k},b_k]^\top, \\
    x_{k+1} &= x_k - \rho_k F^{w_1,b}(x_k), \\
    x^* &= [0,0].
\end{align}
We will show that simultaneous gradient descent always produces an iterate that is farther away from the equilibrium than the previous iterate, i.e. $||x_{k+1}-x^*||^2/||x_{k}-x^*||^2 > 1$.
\begin{align}
    ||x_{k+1}-x^*||^2/||x_{k}-x^*||^2 &= ||x_k - \rho_k J^{w_1,b} x_k||^2 /||x_{k}||^2 \\
    &= ||(I - \rho_k J^{w_1,b}) x_k||^2 /||x_{k}||^2 \\
    &= \frac{x_k^\top (I - \rho_k J^{w_1,b})^\top (I - \rho_k J^{w_1,b}) x_k}{x_k^\top x_k} \\
    &= \frac{x_k^\top M x_k}{x_k^\top x_k} \quad \text{Rayleigh quotient of $M$} \\
    &\ge \lambda_{\min}(M),
\end{align}
where
\begin{align}
    M &= (I - \rho_k J^{w_1,b})^\top (I - \rho_k J^{w_1,b}) \\
    &= \begin{bmatrix}
    1 + \rho_k^2 & 0 \\ 0 & 1 + \rho_k^2
    \end{bmatrix}, \\
    \lambda_{\min}(M) &= 1 + \rho_k^2 > 1.
\end{align}
Therefore, simultaneous gradient descent diverges from the equilibrium of the ($w_1,b$)-subsystem for any step size scheme, $\rho_k$.

\subsection{Derivation of \emph{Crossing-the-Curl}}
\label{Fccderiv}
%
%
%
Here, we derive our proposed technique in 3-d, however, the result of the derivation can be computed in arbitrary dimensions:
\begin{align}
  (\nabla \times F) \times F &= -F \times (\nabla \times F) \\
  &= -v \times (\nabla \times F) \text{ where } v=F \\
  &= -\nabla_F (v \cdot F) + (v \cdot \nabla) F \text{ where } \nabla_F \text{ is Feynman notation} \\
  &= -\Big( v_1 \Big[ \frac{\partial F_1}{\partial x_1}, \ldots, \frac{\partial F_1}{\partial x_n} \Big] + \ldots + v_n \Big[ \frac{\partial F_n}{\partial x_1}, \ldots, \frac{\partial F_n}{\partial x_n} \Big] \Big) \\
  &+ \Big(v_1 \frac{\partial}{\partial x_1} + \ldots + v_n \frac{\partial}{\partial x_n} \Big) F \\
  &= (J-J^\top)F.
\end{align}

\subsection{Monotonicity: Definitions and Requirements}
For all $x \in \mathcal{X}$ and $x' \in \mathcal{X}$,
\begin{align}
    \langle F(x) - F(x'), x - x' \rangle (> 0 , \ge s ||x-x'||^2) \ge 0 &\quad \text{(strictly, s-strongly)-monotone,} \\
    \langle F(x'), x - x' \rangle \ge 0 \implies F(x'), x - x' \rangle (> 0) \ge 0 &\quad \text{(strictly-)pseudomonotone,} \\
    \langle F(x'), x - x' \rangle > 0 \implies F(x'), x - x' \rangle \ge 0 &\quad \text{quasimonotone.}
\end{align}

While we used these definitions in our analysis for certain cases, the following alternate requirements proposed in~\cite{crouzeix1996criteria} made the complete analysis of the system tractable. We restate them here for convenience. Note that we what we refer to as condition (B) in the main body of the paper is actually a stronger version of condition (C) below with $v=(x^*-x)/t$.

Consider the following conditions:
\begin{enumerate}[label=(\Alph*)]
    \item For all $x \in \mathcal{X}$ and $v \in \mathbb{R}^n$ such that $v^\top F(x) = 0$ we have $v^\top J(x) v \ge 0$.\label{cond:A}
    \item For all $x \in \mathcal{X}$ and $v \in \mathbb{R}^n$ such that $F(x)=0$, $v^\top J(x) v = 0$, and $v^\top F(x+\tilde{t}v)>0$ for some $\tilde{t}<0$, we have that for all $\bar{t}>0$, there exists $t \in (0,\bar{t}]$ such that $t \in I_{x,v}$ and $v^\top F(x+tv) \ge 0$.\label{cond:B'}
    \item For all $x \in \mathcal{X}$ and $v \in \mathbb{R}^n$ such that $F(x)=0$ and $v^\top J(x) v = 0$, we have that for all $\bar{t}>0$, there exists $t \in (0,\bar{t}]$ such that $t \in I_{x,v}$ and $v^\top F(x+tv) \ge 0$.\label{cond:C'}
\end{enumerate}
\begin{theorem}[\cite{crouzeix1996criteria}, Theorem 3]
Let $F: \mathcal{X} \rightarrow \mathbb{R}^n$ be differentiable on the open convex set $\mathcal{X} \subset \mathbb{R}^n$.
\begin{enumerate}[label=(\roman*)]
    \item $F$ is quasimonotone on $\mathcal{X}$ if and only if (A) and (B') hold.
    \item $F$ is pseudomonotone on $\mathcal{X}$ if and only if (A) and (C') hold.
\end{enumerate}
\end{theorem}

\subsection{A Comparison of Monotonicity and Hurwitz}
\label{monvshurwitz}
The monotonicity and Hurwitz properties are complementary.

\subsubsection{Hurwitz Does Not Imply Quasimonotonicity}
\label{hurwitznotquasi}
Let $F(x)=Jx$, $J = \begin{bmatrix}
1 & 4 \\
-1 & 1
\end{bmatrix}$, $S=\begin{bmatrix}
0 & 1 \\ -1 & 0
\end{bmatrix}$, and $v=SJx = [-x_1 + x_2 , -x_1 - 4x_2]^\top$. Then $\lambda_{1,2}(J) = 1 \pm 2i$ so $J$ is Hurwitz, and
\begin{align}
[v^\top J v]\Big\vert_{(-1,1)}  = [x_1^2 + 3 x_1 x_2 + x_2^2]\Big\vert_{(-1,1)} = -1,
\end{align}
which, by condition~\ref{cond:A}, implies $F$ is not quasimonotone.

\subsubsection{Monotonicity Does Not Imply Hurwitz}
Let $F(x)=Jx$ and $J =
\begin{bmatrix}
0 & 1 \\
-1 & 0
\end{bmatrix}$. Then $\lambda_{1,2}(J) = \pm i$
so $J$ is not Hurwitz, but
\begin{align}
J + J^\top  =
\begin{bmatrix}
0 & 0 \\
0 & 0
\end{bmatrix} \succeq 0, \quad \lambda_{1,2} = 0,
\end{align}
so $F$ is monotone.

\subsubsection{Monotonicity and Hurwitz Can Overlap}
Let $F(x)=Jx$ and $J =
\begin{bmatrix}
1 & 0 \\
0 & 1
\end{bmatrix}$. Then $\lambda_{1,2}(J) = 1$
so $J$ is Hurwitz and
\begin{align}
J + J^\top  =
\begin{bmatrix}
1 & 0 \\
0 & 1
\end{bmatrix} \succeq 0, \quad \lambda_{1,2} = 1,
\end{align}
so $F$ is monotone.

\begin{proposition}[(Strict,Strong)-Monotonicity Implies Hurwitz]
\label{smon2hurwitz}
If $F$ is differentiable and strictly-monontone, then the Jacobian of $F$, $J$, is Hurwitz. If $F$ is differentiable and $s$-strongly-monotone, then $J$ is Hurwitz with $\min(\mathbb{R}(\mathbf{\lambda})) \ge s$.
\end{proposition}
\begin{proof}
Assume $A$ is a real, square matrix and $A$ is either positive definite or strongly-positive definite, i.e. $v^\top A v \succeq 0$ or $v^\top A v \succeq s ||v||^2$ with $v \in \mathbb{C}^n$. Let $^*$ denote the conjugate transpose and note that $\langle u, w \rangle = u^*w$. Let $\lambda=a+bi$ be a potentially complex eigenvalue of $A$ and $v$ be its corresponding eigenvector, i.e. $Av=\lambda v$. We aim to prove that if $A$ satisfies the above assumptions, then $a>0$, i.e., $A$ is Hurwitz.
\begin{align}
    \langle (A+A^\top) v, v \rangle &= \langle Av, v \rangle + \langle A^\top v,v \rangle \\
    \langle A^\top v,v \rangle &= (A^\top v)^* v \\
    &= v^* (A^\top)^* v \\
    &= v^* (Av) \text{ because $A$ is real} \\
    &= \langle v,Av \rangle \\
    0 &< (\text{ or } s||v||^2 \le)  \langle \frac{1}{2} (A+A^\top)v, v \rangle \\
    &= \frac{1}{2} (\langle Av,v \rangle + \langle v,Av \rangle) \\
    &= \frac{1}{2} ((a+bi) \langle v,v \rangle + \overline{(a+bi)} \langle v,v \rangle) \\
    &= \frac{1}{2} [(a+bi)||v||^2 + (a-bi)||v||^2] \\
    &= a||v||^2 \\
    &\Rightarrow a > 0 \text{ or } a>=s
\end{align}
If $F$ is (strictly,strongly)-monotone, then the Jacobian of $F$ is a real, square, (positive definite,strongly-positive definite) matrix, therefore, it matches the above assumptions. Hence, the conclusion follows.
\end{proof}

\subsection{\emph{Crossing-the-Curl} Can Make Monotone Fields, Non-Monotone}
\label{ccmon2nonmon}
Here, we provide examples of negative results for \emph{Crossing-the-Curl}. This is to emphasize that our proposed technique can cause problems if not used with caution. The headings below describe the before and afters when applying our proposed technique to the map $F(x)=Jx$.

Monotone to Non-Monotone.
\begin{align}
    J &= \begin{bmatrix}
    4 & 1 \\ -1 & 1
    \end{bmatrix} \\
    J^{sym} &= \begin{bmatrix}
    4 & 0 \\ 0 & 1
    \end{bmatrix} , \lambda_{1,2} = 4,1 \\
    J^{sym}_{cc} &= \begin{bmatrix}
    2 & 3 \\ 3 & 2
    \end{bmatrix} , \lambda_{1,2} = 5,-1 \label{mon2saddle}
\end{align}

Increase in condition number: $\kappa = \sfrac{11}{5} \rightarrow 4$.
\begin{align}
    J &= \begin{bmatrix}
    1 & \sfrac{1}{4} \\ -1 & 1
    \end{bmatrix} \\
    J^{sym} &= \begin{bmatrix}
    1 & -\sfrac{3}{8} \\ -\sfrac{3}{8} & 1
    \end{bmatrix} , \lambda_{1,2} = \sfrac{11}{8}, \sfrac{5}{8} \\
    J^{sym}_{cc} &= \begin{bmatrix}
    \sfrac{5}{4} & 0 \\ 0 & \sfrac{5}{16}
    \end{bmatrix} , \lambda_{1,2} = \sfrac{5}{4}, \sfrac{5}{16} \label{increaseconditionnumber}
\end{align}

Saddle becomes Monotone.
\begin{align}
    J &= \begin{bmatrix}
    -1 & 1 \\ -1 & 1
    \end{bmatrix} \\
    J^{sym} &= \begin{bmatrix}
    -1 & 0 \\ 0 & 1
    \end{bmatrix} , \lambda_{1,2} = -1,1 \\
    J^{sym}_{cc} &= \begin{bmatrix}
    2 & -2 \\ -2 & 2
    \end{bmatrix} , \lambda_{1,2} = 4, 0 \label{saddle2mon}
\end{align}

Unstable point becomes stable.
\begin{align}
    J &= \begin{bmatrix}
    -2 & 1 \\ -1 & -1
    \end{bmatrix} \\
    J^{sym} &= \begin{bmatrix}
    -2 & 0 \\ 0 & -1
    \end{bmatrix} , \lambda_{1,2} = -2,-1\\
    J^{sym}_{cc} &= \begin{bmatrix}
    2 & -1 \\ -1 & 2
    \end{bmatrix} , \lambda_{1,2} = 3,1 \label{unstable2mon}
\end{align}

$F^{w_2,a}_{eg'}$ becomes non-monotone.
\begin{align}
    F &= \begin{bmatrix}
    w_2 \\ \frac{a^2 - \sigma^2 - 2w_2^2}{2a}
    \end{bmatrix} \\
    F_{cc} &= \begin{bmatrix}
    -\frac{w_2(a^2-\sigma^2-2 w_2^2)}{2a^2} \\
    \frac{w_2^2}{a}
    \end{bmatrix} \\
    Tr[J_{cc}]\Big\vert_{w_2=0,a=2\sigma} &= -\frac{3}{8} \Rightarrow J_{cc} \cancel{\succeq} 0 \label{feg'nonmon}
\end{align}

\begin{proposition}
\emph{Crossing-the-Curl} forces monotonicity for normal, affine fields.
\label{affinemon}
\end{proposition}
\begin{proof}
Let $F=Jx+b$ and assume $J$ is normal, i.e., $JJ^\top = J^\top J$. Then
\begin{align}
F_{cc} &= (J^\top - J) F \\
&= (J^\top - J) (Jx + b) \\
J_{cc} &= (J^\top -J) J \\
&= (J^\top J - J J) \\
J^{sym}_{cc} &= \frac{2J^\top J - JJ - J^\top J^\top}{2} \\
&= \frac{J^\top J + J J^\top - JJ - J^\top J^\top}{2} + \frac{J^\top J - J J^\top}{2} \\
&= \frac{J^\top J + J J^\top - JJ - J^\top J^\top}{2} \text{ because $J$ is normal} \\
&= \frac{-(J-J^\top)(J-J^\top)}{2} \\
&= \frac{(J-J^\top)^\top (J-J^\top)}{2} \\
z^\top J^{sym}_{cc} z &= \frac{1}{2} [(J-J^\top)z]^\top [(J-J^\top)z] \\
&= \frac{1}{2} ||(J-J^\top)z||^2 \ge 0 \Rightarrow J_{cc} \succeq 0.
\end{align}    
\end{proof}

\subsection{Analysis of the ($w_1,b$)-Subsystem}
\begin{proposition}
\label{unrolledalt}
Unrolled GANs and Alternating Updates are Monotone for the ($w_1,b$)-subsystem.
\end{proposition}
\begin{proof}
In Unrolled GANs, the generator computes the gradient of $V$ assuming the discriminator has already made several updates. Define the discriminator's update as
\begin{align}
    w_{1,k+1} &= w_{1,k} - \rho F_{w_1}(w_{1,k},b_k) = U_{k}(w_{1,k}),
\end{align}
and denote the composition of $U$, $\Delta k$-times as
\begin{align}
    U^{\Delta k}_{k}(w_{1,k}) &= U_{k}(\cdots(U_{k}(U_{k}(w_{1,k}))\cdots)
\end{align}
where $\Delta k$ is some positive integer. Then the update for Unrolled GANs is
\begin{align}
    w_{1,k+1} &= w_{1,k} - \rho \frac{\partial V(w_{1,k},b_k)}{\partial w_1} \\
    b_{k+1} &= b_k - \rho \frac{\partial V(U^{\Delta k}_{k}(w_{1,k}),b_k)}{\partial b}.
\end{align}
In the case of the ($w_1,b$)-subsystem, we can write these unrolled updates out explicitly. Remember $F=[b-\mu,-w_1]^\top$, so
\begin{align}
    U_{k}(w_{1,k}) &= w_{1,k} - \rho (b_k-\mu), \\
    U^{\Delta k}_{k}(w_{1,k}),b_k) &= w_{1,k} - \rho \Delta k (b_k-\mu).
\end{align}
Plugging this back in, we find
\begin{align}
    w_{1,k+1} &= w_{1,k} - \rho (b_k-\mu) \\
    b_{k+1} &= b_k - \rho (\rho \Delta k (b_k-\mu) - w_{1,k}) \label{eqn:bunr},
\end{align}
where the corresponding map is $F^{unr}= [b_k-\mu,\rho \Delta k (b_k-\mu) - w_{1,k}]$. Taking a look at the Jacobian, we find
\begin{align}
    J^{unr} &= \begin{bmatrix}
    0 & 1 \\ -1 & \rho \Delta k
    \end{bmatrix} \\
    J^{unr}_{sym} &= \begin{bmatrix}
    0 & 0 \\ 0 & \rho \Delta k
    \end{bmatrix}\succeq 0.
\end{align}
Now, consider alternating updates:
\begin{align}
    w_{1,k+1} &= w_{1,k} - \rho (b_{k+1}-\mu) \\
    &= w_{1,k} - \rho (b_k - \rho (-w_{1,k}) -\mu) \\
    b_{k+1} &= b_k - \rho (-w_{1,k}).
\end{align}
Here, we considered updating $b$ first, but the ($w_1,b$)-subsystem is perfectly symmetric, so the analysis holds either way. If $w_1$ is updated first, this is equivalent to Unrolled GAN with $\Delta k = 1$ (see Equation~\ref{eqn:bunr}). The Jacobian is
\begin{align}
    J^{alt} &= \begin{bmatrix}
    \rho & 1 \\ -1 & 0
    \end{bmatrix} \\
    J^{alt}_{sym} &= \begin{bmatrix}
    \rho & 0 \\ 0 & 0
    \end{bmatrix}\succeq 0.
\end{align}
The Jacobian's for Unrolled GAN and alternating descent are both positive semidefinite, therefore, their maps are monotone (but not strictly-monotone). Note that these results imply neither is Hurwitz either because both Jacobians exhibit a zero eigenvalue.
\end{proof}

\begin{proposition}
\label{w1bmonotone}
$F_{lin}$, $F_{cc}$, $F_{eg}$, and $F_{con}$ are strongly-monotone for the ($w_1,b$)-subsystem (includes multivariate case). $F$ and $F_{reg}$ are monotone, but not strictly monotone. Moreover, $F_{lin}$, $F_{cc}$, $F_{eg}$, $F_{con}$, and $F_{reg}$ are Hurwitz for the ($w_1,b$)-subsystem (includes multivariate case). $F$ is not Hurwitz.
\end{proposition}
\begin{proof}
We start with the original map, $F^{w_1,b}$, and its Jacobian.
\begin{align}
    F &= \begin{bmatrix}
    b-\mu \\ -w_1
    \end{bmatrix} \\
    J &= \begin{bmatrix}
    0 & I \\ -I & 0
    \end{bmatrix} \\
    J = J_{sym} &= \begin{bmatrix}
    0 & 0 \\ 0 & 0
    \end{bmatrix} \succeq 0
\end{align}
The symmetrized Jacobian is positive semidefinite, therefore this system is monotone. Also, the real parts of the eigenvalues of its Jacobian are zero, therefore, $J$ is not Hurwitz.

Now we analyze $F^{w_1,b}_{cc}$, $F^{w_1,b}_{eg}$, and $F^{w_1,b}_{con}$, which as discussed in the main body, are equivalent.
\begin{align}
    F_{cc} = F_{eg} = F_{con} &= \begin{bmatrix}
    w_1 \\ b-\mu
    \end{bmatrix} \\
    J = J_{sym} &= \begin{bmatrix}
    I & 0 \\ 0 & I
    \end{bmatrix} \succeq 1
\end{align}
The symmetrized Jacobian is positive definite with a minimum eigenvalue of $1$, therefore this system is $1$-strongly-monotone. By Proposition~\ref{smon2hurwitz}, the Jacobians of these maps are Hurwitz for the $(w_1,b)$-subsystem.

Now we analyze the generalization $F^{w_1,b}_{lin} = (\alpha I - \beta J^\top - \gamma J)F^{w_1,b}$.
\begin{align}
    F_{lin} &= \begin{bmatrix}
    \alpha (b-\mu) + (\beta+\gamma) w_1 \\ -\alpha w_1 + (\beta+\gamma) (b-\mu)
    \end{bmatrix} \\
    J &= \begin{bmatrix}
    (\beta+\gamma) I & \alpha I \\ -\alpha I & (\beta+\gamma) I
    \end{bmatrix} \\
    J_{sym} &= \begin{bmatrix}
    (\beta+\gamma) I & 0 \\ 0 & (\beta+\gamma) I
    \end{bmatrix} \succeq \beta+\gamma
\end{align}
The symmetrized Jacobian is positive definite with a minimum eigenvalue of $(\beta+\gamma)$, therefore this system is $(\beta+\gamma)$-strongly-monotone. By Proposition~\ref{smon2hurwitz}, $J^{w_1,b}_{lin}$ is Hurwitz for the $(w_1,b)$-subsystem.

Now we analyze the regularized-gradient algorithm, $F^{w_1,b}_{reg}$.
\begin{align}
    F_{reg} &= \begin{bmatrix}
    b-\mu \\ -w_1 + 2 \eta (b-\mu)
    \end{bmatrix}, \eta>0 \\
    J_{reg} &= \begin{bmatrix}
    0 & I \\
    -I & 2\eta I
    \end{bmatrix}, \lambda_{1,2} = \eta \pm \sqrt{\eta^2-1} \Rightarrow \mathbb{R}(\lambda_{1,2}) > 0 \\
    J_{regsym} &= \begin{bmatrix}
    0 & 0 \\
    0 & 2\eta I
    \end{bmatrix} \succeq 0
\end{align}
Therefore, this map is monotone (but not strictly or strongly-monotone). Also, the real parts of the eigenvalues of its Jacobian are strictly positive, therefore, $J^{w_1,b}_{reg}$ is Hurwitz.

Note that for $F_{cc}$, $F_{eg}$, $F_{con}$, and $F_{lin}$, $J$ is symmetric, therefore, $F$ is the gradient of some function, $f(w_1,b)=\frac{1}{2}(w_1^2 + (b-\mu)^2)$. Also, note that the standard algorithm with step size $\rho_k=\frac{1}{k+1}$ is equivalent to the standard running estimate of the mean: $\mu_{k+1} = \frac{k}{k+1}\mu_k + \frac{1}{k+1}x_k$ where $x_k$ is the $k$-th sample.


\end{proof}

\subsection{A Linear Combination of $F$, $JF$, and $J^\top F$ is Not Quasimontone for the 1-d LQ-GAN}
\label{subsec:lincombnotquasi}
The Jacobian of $F_{lin}$, written below, will be useful for the proof. The proof proceeds by process of elimination, ruling out different regions of the space $[\alpha,\beta,\gamma] \in \mathbb{R}^3$ by showing that any $F_{lin}$ with those constants is not quasimonotone.

\begin{align}
(\alpha I + \beta J^\top  - \gamma J) F &=
\begin{bmatrix}
\alpha & 0 & -2(\beta + \gamma) a & -2(\beta+\gamma) b \\
0 & \alpha & 0 & -(\beta + \gamma) \\
2 (\beta + \gamma) a & 0 &  \alpha - 2 (\beta - \gamma) w_2 & 0 \\
2 (\beta + \gamma) b & (\beta + \gamma) & 0 & \alpha - 2 (\beta - \gamma) w_2
\end{bmatrix}
\begin{bmatrix}
-\sigma^2 + a^2 + b^2 \\
b \\
-2 w_2 a \\
-2 w_2 b - w_1
\end{bmatrix} \\
&=
\begin{bmatrix}
\alpha (-\sigma^2 + a^2 + b^2) + 4(\beta + \gamma) w_2 (a^2 + b^2) + 2(\beta + \gamma) w_1 b \\
\alpha b + (\beta + \gamma) (2 w_2 b + w_1) \\
2a (\beta + \gamma) (- \sigma^2 + a^2 + b^2) + 4 (\beta - \gamma) w_2^2 a - 2 \alpha w_2 a \\
2 (\beta + \gamma) b (-\sigma^2 + a^2 + b^2) + (\beta + \gamma) b + (2 (\beta - \gamma) w_2 - \alpha) (2 w_2 b + w_1)
\end{bmatrix}
\end{align}

Specifically, we first consider the sign of $\beta+\gamma$. Lemma~\ref{bpgpos} rules out negative values. Lemma~\ref{bpgneglehalf} rules out positive values when $\sigma^2 \le \sfrac{1}{2}$, and Lemma~\ref{bpgneggehalf} rules out positive values when $\sigma^2 > \sfrac{1}{2}$. Corollary~\ref{bpgeq0} concludes that $\beta+\gamma=0$.

Next, given $\beta+\gamma=0$, we consider the sign of $\alpha$. Lemmas~\ref{bl0} and~\ref{al0} rule out positive values of $\alpha$ when $\beta$ is greater than or less than or equal to zero respectively, i.e., $\alpha$ cannot be positive. Similarly, Lemmas~\ref{bg0} and~\ref{ag0} rule out negative values of $\alpha$ when $\beta$ is less than or greater than or equal to zero respectively, i.e., $\alpha$ cannot be negative. Corollary~\ref{aeq0} concludes that $\alpha=0$.

Lastly, given that $\beta+\gamma=\alpha=0$, Lemmas~\ref{bl02} and~\ref{bg02} prove that $\beta$ cannot be greater than or less than or equal to zero respectively. Corollary~\ref{beq0} concludes that $\beta=\gamma=0$. Therefore, the only quasimonotone linear combination is the trivial one resulting in $F=0$, which completes the proof.

\begin{lemma}
For $F_{lin}$ to be quasimonotone, $\beta+\gamma$ must not be strictly less than zero, i.e. $\beta+\gamma \cancel{<} 0$.
\label{bpgpos}
\end{lemma}
\begin{proof}
Consider
\begin{align}
y &= [0,0,\sigma,-\sigma] \\
x &= [0,0,\sigma,\sigma] \\
\langle F(y), x-y \rangle &= 2 \sigma F_b(y) = - 2 \sigma^2 (\beta + \gamma) (1 - 2\sigma^2 + 2\sigma^2 + 2\sigma^2) \\
&= -2 \sigma^2 (\beta + \gamma) (1 + 2\sigma^2) \\
\langle F(x), x-y \rangle &= 2 \sigma F_b(x) = 2 \sigma^2 (\beta + \gamma) (1 + 2\sigma^2)
\end{align}
If $(\beta + \gamma) < 0$, then this system is not quasimonotone. Therefore, assume $(\beta+\gamma)\ge0$ from now on.
\end{proof}

\begin{lemma}
\label{bpgneglehalf}
If $\sigma^2 \le \frac{1}{2}$, for $F_{lin}$ to be quasimonotone, $\beta+\gamma$ must not be strictly greater than zero, i.e. $\beta+\gamma \cancel{>} 0$.
\end{lemma}

\begin{proof}
We will use a different parameterization of $F_{lin}$ for this part of the proof.

\begin{align}
    J_{skew} &= (J^\top -J)/2 \\
    J_{sym} &= (J^\top +J)/2 \\
    \beta &= (\hat{\beta}+\hat{\gamma})/2 \\
    \gamma &= (\hat{\beta}-\hat{\gamma})/2 \\
    \hat{\beta} &= \beta+\gamma \\
    \hat{\gamma} &= \beta-\gamma
\end{align}

The linear combination is now defined as

\begin{align}
(\alpha I + \hat{\beta} J_{skew} + \hat{\gamma} J_{sym}) F &=
\begin{bmatrix}
\alpha & 0 & -2\hat{\beta} a & -2\hat{\beta} b \\
0 & \alpha & 0 & -\hat{\beta} \\
2 \hat{\beta} a & 0 &  \alpha - 2 \hat{\gamma} w_2 & 0 \\
2 \hat{\beta} b & \hat{\beta} & 0 & \alpha - 2 \hat{\gamma} w_2
\end{bmatrix}
\begin{bmatrix}
-\sigma^2 + a^2 + b^2 \\
b \\
-2 w_2 a \\
-2 w_2 b - w_1
\end{bmatrix} \\
&=
\begin{bmatrix}
\alpha (-\sigma^2 + a^2 + b^2) + 4\hat{\beta} w_2 (a^2 + b^2) + 2\hat{\beta} w_1 b \\
\alpha b + \hat{\beta} (2 w_2 b + w_1) \\
2a \hat{\beta} (- \sigma^2 + a^2 + b^2) + 4 \hat{\gamma} w_2^2 a - 2 \alpha w_2 a \\
2 \hat{\beta} b (-\sigma^2 + a^2 + b^2) + \hat{\beta} b + (2 \hat{\gamma} w_2 - \alpha) (2 w_2 b + w_1)
\end{bmatrix}
\end{align}

In order for a system to be quasimonotone, we require condition~\ref{cond:A} (among other properties). We will now show that this property is not satisfied for $F_{lin}$ with $\hat{\beta} > 0$ by considering two different cases.

\textbf{Case 1}: Consider the $(w_2,a)$-subsystem. Let

\begin{align}
    v &= \begin{bmatrix}
    0 & 0 & -1 & 0 \\
    0 & 0 & 0 & 0 \\
    1 & 0 & 0 & 0 \\
    0 & 0 & 0 & 0
    \end{bmatrix}
    \overbrace{
    \begin{bmatrix}
\alpha (-\sigma^2 + a^2 + b^2) + 4\hat{\beta} w_2 (a^2 + b^2) + 2\hat{\beta} w_1 b \\
\alpha b + \hat{\beta} (2 w_2 b + w_1) \\
2a \hat{\beta} (- \sigma^2 + a^2 + b^2) + 4 \hat{\gamma} w_2^2 a - 2 \alpha w_2 a \\
2 \hat{\beta} b (-\sigma^2 + a^2 + b^2) + \hat{\beta} b + (2 \hat{\gamma} w_2 - \alpha) (2 w_2 b + w_1)
\end{bmatrix}}^{F_{lin}} \\
&= \begin{bmatrix}
-2a \hat{\beta} (- \sigma^2 + a^2 + b^2) - 4 \hat{\gamma} w_2^2 a + 2 \alpha w_2 a \\
0 \\
\alpha (-\sigma^2 + a^2 + b^2) + 4\hat{\beta} w_2 (a^2 + b^2) + 2\hat{\beta} w_1 b \\
0
\end{bmatrix}
\end{align}

Above, we premultiply $F_{lin}$ by a skew symmetric matrix, which ensures $v^\top  F_{lin} = F_{lin}^\top  A_{skew} F_{lin} = 0$.

The relevant portion of the Jacobian of $F_{lin}$ is

\begin{align}
    J^{w_2,a}_{lin} &= \begin{bmatrix}
    4\hat{\beta} (a^2+b^2) & 2a\alpha + 8\hat{\beta} w_2 a \\
    8\hat{\gamma}w_2 a - 2\alpha a & 2\hat{\beta}(-\sigma^2 + 3a^2 + b^2) + 4 \hat{\gamma} w_2^2 - 2\alpha w_2
    \end{bmatrix}
\end{align}

Consider $x=[0,0,c\sigma,0]$ and both $\hat{\beta}$ and $\alpha$ fixed.

\begin{align}
    v^\top  J^{w_2,a}_{lin} v &= \lim_{c\rightarrow 0^+} 2 \hat{\beta} (-1 + c^2)^2 \sigma^6 [\alpha^2 (-1 + 3 c^2) + 
   8 \hat{\beta}^2 c^4 \sigma^2] \label{bpgneghalflim} \\
   &= -2 \hat{\beta} \sigma^6 \alpha^2 \ge 0
\end{align}

This implies either $\alpha=0$ or $\hat{\beta} \le 0$ for the system to be quasimonotone.

\textbf{Case 2}: Consider the $(a,b)$-subsystem. Let

\begin{align}
    v &= \begin{bmatrix}
    0 & 0 & 0 & 0 \\
    0 & 0 & 0 & 0 \\
    0 & 0 & 0 & -1 \\
    0 & 0 & 1 & 0
    \end{bmatrix}
    \overbrace{
    \begin{bmatrix}
\alpha (-\sigma^2 + a^2 + b^2) + 4\hat{\beta} w_2 (a^2 + b^2) + 2\hat{\beta} w_1 b \\
\alpha b + \hat{\beta} (2 w_2 b + w_1) \\
2a \hat{\beta} (- \sigma^2 + a^2 + b^2) + 4 \hat{\gamma} w_2^2 a - 2 \alpha w_2 a \\
2 \hat{\beta} b (-\sigma^2 + a^2 + b^2) + \hat{\beta} b + (2 \hat{\gamma} w_2 - \alpha) (2 w_2 b + w_1)
\end{bmatrix}}^{F_{lin}} \\
&= \begin{bmatrix}
0 \\
0 \\
-2 \hat{\beta} b (-\sigma^2 + a^2 + b^2) - \hat{\beta} b - (2 \hat{\gamma} w_2 - \alpha) (2 w_2 b + w_1) \\
2a \hat{\beta} (- \sigma^2 + a^2 + b^2) + 4 \hat{\gamma} w_2^2 a - 2 \alpha w_2 a
\end{bmatrix}
\end{align}

The relevant portion of the Jacobian of $F_{lin}$ is

\begin{align}
    J^{a,b}_{lin} &= \begin{bmatrix}
    2\hat{\beta} (-\sigma^2 + 3a^2 + b^2) + 4\hat{\gamma} w_2^2 - 2\alpha w_2 & 4ab \hat{\beta} \\
    4ab \hat{\beta} & 2\hat{\beta} (-\sigma^2+a^2 + 3b^2) + \hat{\beta} + 2w_2(2\hat{\gamma}w_2-\alpha)
    \end{bmatrix}
\end{align}

Consider $\alpha=0$ and $x=[0,0,\frac{\sigma}{10},\frac{\sigma}{2}]$. Then

\begin{align}
    v^\top  J^{a,b}_{lin} v &= \hat{\beta}^3 \sigma^4 \underbrace{(-1.44 + 4.46842 \sigma^2 - 3.37146 \sigma^4)}_{< 0 \,\, \forall \,\, \sigma^2 \in (0,1/2]}
\end{align}

Then $\alpha=0 \Rightarrow \hat{\beta} \le 0$. In either case, $\hat{\beta}$ must be nonpositive. Therefore, $\hat{\beta} = \beta + \gamma \cancel{>} 0$.
\end{proof}

\begin{proof}
[Alternate Proof for Lemma~\ref{bpgneglehalf}]
Part of the proof in Lemma~\ref{bpgneglehalf} looks at the limit in which $a$ approaches 0. One might presume a simple fix is to constrain $a$ to be larger than some small value, e.g., 1e-10, and use a large $\hat{\beta}$ value. Here, we show that even using $a=\frac{\sigma}{100}$ breaks quasimonotonicity. The variance of the data distribution is assumed to be unknown, which would make it very difficult to select a proper lower bound for $a$ that maintains quasimonotonicity within the feasible region.

Consider $x=[0,-1,\frac{\sigma}{100},\frac{\sigma}{2}]$ and the $(a,b)$-subsystem as in~\ref{bpgneglehalf}. Then

\begin{align}
    v^\top  J v &= \Big( -5.9976 \hat{\beta} \sigma^2 \Big) \alpha^2 + \Big( \hat{\beta}^2 \sigma^3 (-5.9976 + 8.9976 \sigma^2) \Big) \alpha + \hat{\beta}^3 \sigma^4 (-1.4994 + 4.4997 \sigma^2 - 3.375 \sigma^4).
\end{align}

If $\hat{\beta} > 0$, then this is a concave quadratic form in $\alpha$. To find where this function is positive, we need to find its roots.

\begin{align}
    \alpha_{\pm} &= \frac{-\Big( \hat{\beta}^2 \sigma^3 (-5.9976 + 8.9976 \sigma^2) \Big) }{2\Big( -5.9976 \hat{\beta} \sigma^2 \Big)} \\
    &\pm \frac{\sqrt{\Big( \hat{\beta}^2 \sigma^3 (-5.9976 + 8.9976 \sigma^2) \Big)^2-4\Big( -5.9976 \hat{\beta} \sigma^2 \Big)\Big( \hat{\beta}^3 \sigma^4 (-1.4994 + 4.4997 \sigma^2 - 3.375 \sigma^4) \Big)}}{2\Big( -5.9976 \hat{\beta} \sigma^2 \Big)} \\
    \sqrt{\cdot}^2 &= \hat{\beta}^4 \sigma^6 (5.9976^2 - (2)(5.9976)(8.9976) \sigma^2 + 8.7616^2 \sigma^4) \\
    &+ 4(5.9976) \hat{\beta}^4 \sigma^6 (-1.4994 + 4.4997 \sigma^2 - 3.375 \sigma^4) \\
    &= \hat{\beta}^4 \sigma^6 \Big( 5.9976^2 - (4)(1.4994)(5.9976) + (5.9976)[(4.4997)(4)-(2)(8.9976)]\sigma^2 \\
    &+ [8.9976^2-(4)(5.9976)(3.375)]\sigma^4 \Big) \\
    &= \hat{\beta}^4 \sigma^6 \Big( 0.02159136\sigma^2 - 0.01079424\sigma^4 \Big) \\
    &= \hat{\beta}^4 \sigma^8 \Big( 0.02159136 - 0.01079424\sigma^2 \Big) \\
    \frac{\sqrt{\cdot}}{2 ( -5.9976 \hat{\beta} \sigma^2)} &= -\hat{\beta} \sigma^2 \sqrt{( 0.02159136 - 0.01079424\sigma^2 )/(2^2*5.9976^2)} \\
    &= -\hat{\beta} \sigma^2 \sqrt{ 0.00015006002 - 0.00007502\sigma^2} \\
    \frac{-b}{2a} &= \frac{-\Big( \hat{\beta}^2 \sigma^3 (-5.9976 + 8.9976 \sigma^2) \Big)}{2\Big( -5.9976 \hat{\beta} \sigma^2 \Big)} \\
    &= \hat{\beta} \sigma (-\frac{1}{2} + 0.75010004001 \sigma^2) \\
    \alpha_{\pm} &= \hat{\beta} \sigma \Big( -\frac{1}{2} + 0.75010004001 \sigma^2 \pm \sigma \sqrt{ 0.00015006002 - 0.00007502\sigma^2} \Big) \\
    \alpha^2 &> \hat{\beta}^2 \sigma^2 (-0.48 + .751 \sigma^2)^2 \hspace{1.0cm} \text{ assuming $\sigma^2<1/2$} \\
    \hat{\beta}^2 &< \frac{1}{\sigma^2 (-0.48 + .751 \sigma^2)^2} \alpha^2
\end{align}

The $\alpha$ root with smaller magnitude provides an upper bound for $\hat{\beta}^2$.

Now consider again $x=[0,0,\frac{\sigma}{100},0]$ and equation~\ref{bpgneghalflim} with $c=\frac{1}{100}$.

\begin{align}
    v^\top  J_{lin} v &= 2 \hat{\beta} (-1 + c^2)^2 \sigma^6 [\alpha^2 (-1 + 3 c^2) + 
   8 \hat{\beta}^2 c^4 \sigma^2] \\
   \hat{\beta}^2 &\ge \alpha^2 \frac{1 - 3c^2}{8c^4\sigma^2} \\
   \hat{\beta}^2 &> \frac{12496250}{\sigma^2} \alpha^2
\end{align}

This provides a lower bound for $\hat{\beta}^2$.

\begin{align}
    \hat{\beta}^{hi}-\hat{\beta}_{lo} &=
    \alpha^2 \Big( \frac{1}{\sigma^2 (-0.48 + .751 \sigma^2)^2} - \frac{12496250}{\sigma^2} \Big) \\
    &= \frac{\alpha^2}{\sigma^2} \Big( \frac{1}{(-0.48 + .751 \sigma^2)^2} - 12496250 \Big) \\
    &< \frac{\alpha^2}{\sigma^2} \Big( 95-12496250 \Big) \hspace{1.0cm} \text{ assuming $\sigma^2<1/2$} \\
    &< 0
\end{align}

The upper bound we require for $\hat{\beta}$ is greater than the lower bound, therefore, no $\hat{\beta}$ will satisfy quasimonotonicity.
\end{proof}

\begin{lemma}
If $\sigma^2 > \frac{1}{2}$, for $F_{lin}$ to be quasimonotone, $\beta+\gamma$ must not be strictly greater than zero, i.e. $\beta+\gamma \cancel{>} 0$.
\label{bpgneggehalf}
\end{lemma}
\begin{proof}
For this proof, we make use of the traditional definition of quasimonotonicity. Consider
\begin{align}
    c &= \frac{1}{2} \sqrt{\sigma^2-\frac{1}{2}} \\
    y &= [0,0,c,-c] \\
    x &= [0,0,c,c] \\
    \langle F(y), x-y \rangle &= 2c F_b(y) = -2 (\beta + \gamma) c^2 (1 + 2c^2 + 2c^2 - 2\sigma^2) \\
    &= -2 (\beta + \gamma) c^2 (1 + \sigma^2 - \frac{1}{2} - 2\sigma^2) = -2 (\beta + \gamma) c^2 (\frac{1}{2} - \sigma^2) \\
    &= 2 (\beta + \gamma) c^2 \underbrace{(\sigma^2 - \frac{1}{2})}_{>0} \\
\langle F(x), x-y \rangle &= 2c F_b(x) = -2 (\beta + \gamma) c^2 \overbrace{(\sigma^2 - \frac{1}{2})}_{>0}
\end{align}
If $(\beta + \gamma) > 0$, then this system is not quasimonotone. In either case, $(\beta+\gamma) \cancel{>} 0$.
\end{proof}

\begin{corollary}[$F_{lin}$ requires $\beta+\gamma=0$ for quasimonotonicity.]
\label{bpgeq0}
Together, Lemmas~\ref{bpgpos},~\ref{bpgneglehalf} and~\ref{bpgneggehalf} imply that $(\beta+\gamma)$ must be $0$ to satisfy quasimonotonicity.
\end{corollary}

\begin{lemma}
If $(\beta+\gamma)=0$ and $\alpha > 0$, for $F_{lin}$ to be quasimonotone, $\beta$ must not be strictly greater than zero, i.e. $\beta \cancel{>} 0$.
\label{bl0}
\end{lemma}
\begin{proof}
For this proof, we make use of the traditional definition of quasimonotonicity. Consider
\begin{align}
    y &= [0,0,c\sigma,0], c>1 \\
    x &= [1,0,\underbrace{(c-\sqrt{c^2-1})}_{>0}\sigma,0] \\
    \langle F(y), x-y \rangle &= F_{w_2}(y) - \sqrt{c^2-1}\sigma F_a(y) = \alpha \sigma^2 \overbrace{(-1 + c^2)}^{>0} \\
\langle F(x), x-y \rangle &= F_{w_2}(x) - \sqrt{c^2-1}\sigma F_a(x) \\
&= \alpha \sigma^2 (-1 + (c-\sqrt{c^2-1})^2) - \sqrt{c^2-1} \sigma (8\beta (c-\sqrt{c^2-1})\sigma - 2\alpha (c-\sqrt{c^2-1})\sigma) \\
&= \alpha \sigma^2 (-1 + (c-\sqrt{c^2-1})^2 + 2(c-\sqrt{c^2-1})\sqrt{c^2-1}) - 8(c-\sqrt{c^2-1}) \sqrt{c^2-1} \sigma^2 \beta \\
&= \alpha \sigma^2 (-1 + c^2 - 2c\sqrt{c^2-1} + c^2 - 1 + 2c\sqrt{c^2-1} - 2(c^2-1)) - 8(c\sqrt{c^2-1}-c^2+1) \sigma^2 \beta \\
&= -8\underbrace{(c\sqrt{c^2-1}-c^2+1)}_{>0} \sigma^2 \beta
\end{align}
If $(\beta+\gamma)=0$ and $\alpha>0$, then $\beta \le 0$ for the system to be quasimonotone.
\end{proof}

\begin{lemma}
If $(\beta+\gamma)=0$, for $F_{lin}$ to be quasimonotone, $\alpha$ must not be strictly greater than zero, i.e. $\alpha \cancel{>} 0$.
\label{al0}
\end{lemma}
\begin{proof}
We will assume $\alpha>0$, which by Lemma~\ref{bl0} implies $\beta \le 0$. This will lead to a contradiction. Consider
\begin{align}
    y &= [1,0,4\sigma,0] \\
    x &= [0,0,2\sigma,0] \\
    \langle F(y), x-y \rangle &= -F_{w_2}(y) - 2\sigma F_a(y) = -15\alpha \sigma^2 - 2\sigma (32 \sigma \beta - 8 \sigma \alpha) \\
    &= \alpha \sigma^2 - 64 \beta \sigma^2 \\
\langle F(x), x-y \rangle &= -F_{w_2}(x) - 2\sigma F_a(x) = -3 \alpha \sigma^2
\end{align}
If $(\beta+\gamma)=0$ and $\alpha>0$ (implies $\beta \le 0$), then $\langle F(y), x-y \rangle > 0$ and $\langle F(x), x-y \rangle < 0$, which breaks quasimonotonicity. Therefore, $\alpha \cancel{>} 0$.
\end{proof}

\begin{lemma}
\label{bg0}
If $(\beta+\gamma)=0$ and $\alpha < 0$, for $F_{lin}$ to be quasimonotone, $\beta$ must not be strictly less than zero, i.e. $\beta \cancel{<} 0$.
\end{lemma}
\begin{proof}
Consider
\begin{align}
    y &= [0,0,c\sigma,0], c>1 \\
    x &= [-1,0,(c+\sqrt{c^2-1})\sigma,0] \\
    \langle F(y), x-y \rangle &= -F_{w_2}(y) + \sqrt{c^2-1}\sigma F_a(y) = -\alpha \sigma^2 \overbrace{(-1 + c^2)}^{>0} \\
\langle F(x), x-y \rangle &= -F_{w_2}(x) + \sqrt{c^2-1}\sigma F_a(x) \\
&= -\alpha \sigma^2 (-1 + (c+\sqrt{c^2-1})^2) + \sqrt{c^2-1} \sigma (8\beta (c+\sqrt{c^2-1}) \sigma + 2\alpha (c+\sqrt{c^2-1})\sigma) \\
&= \alpha \sigma^2 (1 - (c+\sqrt{c^2-1})^2 + 2\sqrt{c^2-1}(c+\sqrt{c^2-1})) + 8(c+\sqrt{c^2-1}) \sqrt{c^2-1} \sigma^2 \beta \\
&= \alpha \sigma^2 (1 - c^2 - c^2 + 1 -2c\sqrt{c^2-1} +2c\sqrt{c^2-1} + 2c^2 - 2) + 2\sqrt{c^2-1}(c+\sqrt{c^2-1})) \beta \\
&= 2\underbrace{\sqrt{c^2-1}(c+\sqrt{c^2-1})}_{> 0} \beta
\end{align}
If $\alpha<0$, then $\beta \ge 0$ to maintain quasimonotonicity.
\end{proof}

\begin{lemma}
\label{ag0}
If $(\beta+\gamma)=0$, for $F_{lin}$ to be quasimonotone, $\alpha$ must not be strictly less than zero, i.e. $\alpha \cancel{<} 0$.
\end{lemma}
\begin{proof}
We will assume $\alpha<0$, which by~\ref{bg0} implies $\beta \ge 0$. This will lead to a contradiction.
\begin{align}
    y &= [-1,0,c\sigma,0], c=\frac{1}{4} \\
    x &= [0,0,d\sigma,0], d=\frac{3}{2} \\
    \langle F(y), x-y \rangle &= F_{w_2}(y) + (d-c)\sigma F_a(y) = \alpha \sigma^2 \overbrace{(-1 + c^2)}^{<0} + (d-c)\sigma (8c \sigma \beta + 2c \sigma \alpha) \\
    &= \alpha \sigma^2 (-1 + c^2 + 2c(d-c)) + 8c(d-c) \sigma^2 \beta \\
    &= \alpha \sigma^2 (-1 - c^2 + 2cd) + 8c(d-c) \sigma^2 \beta \\
    &= -\frac{5}{16} \alpha \sigma^2 + 40 \sigma^2 \beta \\
\langle F(x), x-y \rangle &= F_{w_2}(x) + (d-c)\sigma F_a(x) = \alpha \sigma^2 \underbrace{(-1 + d^2)}_{>0} \\
&= \frac{5}{4} \alpha \sigma^2
\end{align}
If $(\beta+\gamma)=0$ and $\alpha<0$ (implies $\beta \ge 0$), then $\langle F(y), x-y \rangle > 0$ and $\langle F(x), x-y \rangle < 0$, which breaks quasimonotonicity. Therefore, $\alpha \ge 0$.
\end{proof}

\begin{corollary}
\label{aeq0}
Together, Corollary~\ref{bpgeq0} and Lemmas~\ref{bl0}-\ref{ag0} imply that $\alpha$ must equal zero for $F_{lin}$ to be quasimonotone.
\end{corollary}

\begin{lemma}
\label{bl02}
If $(\beta+\gamma)=0$ and $\alpha=0$, for $F_{lin}$ to be quasimonotone, $\beta$ must not be strictly greater than zero, i.e. $\beta \cancel{>} 0$.
\end{lemma}
\begin{proof}
Consider
\begin{align}
    y &= [1,0,1,0] \\
    x &= [1,-7,2,1] \\
    \langle F(y), x-y \rangle &= -7F_{w_1}(y) + F_a(y) + F_b(y) = 8\beta \\
\langle F(x), x-y \rangle &= -7F_{w_1}(x) + F_a(x) + F_b(x) = 16\beta+4\beta(2-7) \\
&= -4\beta
\end{align}
If $\beta > 0$, then this system is not quasimonotone. Therefore, $\beta \le 0$.
\end{proof}

\begin{lemma}
\label{bg02}
If $(\beta+\gamma)=0$ and $\alpha=0$, for $F_{lin}$ to be quasimonotone, $\beta$ must not be strictly less than zero, i.e. $\beta \cancel{<} 0$.
\end{lemma}
\begin{proof}
Consider
\begin{align}
    y &= [1,0,2,0] \\
    x &= [1,1,1,1] \\
    \langle F(y), x-y \rangle &= F_{w_1}(y) - F_a(y) + F_b(y) = -16\beta \\
\langle F(x), x-y \rangle &= F_{w_1}(x) - F_a(x) + F_b(x) = -8\beta + 12\beta = 4\beta
\end{align}
If $\beta < 0$, then this system is not quasimonotone. Therefore, $\beta \ge 0$.
\end{proof}

\begin{corollary}
[$(\beta+\gamma)=0,\alpha=0 \Rightarrow \beta = \gamma = 0$]
\label{beq0}
Together, Lemmas~\ref{bl02} and~\ref{bg02} imply that $\beta=0$, which, along with Corollary~\ref{bpgeq0}, imply that $\gamma=0$ as well.
\end{corollary}

\begin{corollary}
\label{lincombnotquasi}
[$\alpha=\beta=\gamma=0$]
Together, Corollaries~\ref{bpgeq0} and~\ref{aeq0}, and~\ref{beq0} imply that there is no non-trivial linear combination that induces a quasimonotone LQ-GAN system.
\end{corollary}

\begin{corollary}
\label{Fccnotquasi}
$F_{cc}$, $F_{eg}$, $F_{con}$, and $F$ are not quasimonotone for the LQ-GAN system.
\end{corollary}
\begin{proof}
These maps are all linear combinations of $F$, $JF$ and $J^\top F$, therefore, by Corollary~\ref{lincombnotquasi}, they are not quasimonotone for the LQ-GAN system.
\end{proof}

\subsection{Analysis of the $(w_2,a)$-Subsystem}

Note that if a map is not quasimonotone for the $(w_2,a)$-subsystem, then it is not quasimonotone for the full system. This is because an analysis of the $(w_2,a)$-subsystem is equivalent to an analysis of a subspace of the full system with $w_1=b=0$.

\begin{proposition}
\label{Fnotquasi}
$F$ is not quasimontone for the $(w_2,a)$-subsystem. Also, its Jacobian is not Hurwitz.
\end{proposition}
\begin{proof}
\begin{align}
    F &= \begin{bmatrix}
    -\sigma^2 + a^2 + b^2 \\
    b \\
    -2 w_2 a \\
    -2 w_2 b - w_1
    \end{bmatrix}
\end{align}

\begin{align}
    y &= [\sigma,0,3\sigma,0] \\
    x &= [3\sigma,0,5\sigma,0] \\
    \langle F(y), x-y \rangle &= 2 \sigma F_{w_2}(y) + 2 \sigma F_a(y) = 2\sigma (-\sigma^2 + 9 \sigma^2) + 2\sigma (-6 \sigma^2) \\
    &= 4 \sigma^3 \\
\langle F(x), x-y \rangle &= 2 \sigma F_{w_2}(x) + 2 \sigma F_a(x) = 2 \sigma^3 ( -1 + 25) + 2\sigma^3 (-30) \\
    &= -12 \sigma^3
\end{align}
Therefore, $F$ is not quasimonotone.

The Jacobian of $F$ for the ($w_2,a$)-subsystem is
\begin{align}
    J^{w_2,a} &= \begin{bmatrix}
    0 & 2a \\ -2a & -2w_2
    \end{bmatrix}.
\end{align}
The trace of $J^{w_2,a}$ is strictly negative for $w_2>0$, which implies $J^{w_2,a}$ has an eigenvalue with strictly negative real part. Therefore, $J^{w_2,a}$ is not Hurwitz.
\end{proof}

\begin{proposition}
\label{Fregnotquasi}
$F_{reg}$ is not quasimonotone for the $(w_2,a)$-subsystem. Also, its Jacobian is not Hurwitz.
\end{proposition}
\begin{proof}
\begin{align}
    F_{reg} &= \begin{bmatrix}
    -\sigma^2 + a^2 + b^2 \\
    b \\
    -2 w_2 a + 4 \eta a (-\sigma^2 + a^2 + b^2) \\
    -2 w_2 b - w_1 + 4 \eta b (-\sigma^2 + a^2 + b^2) + 2 \eta b
    \end{bmatrix}
\end{align}

In order for a system to be quasimonotone, we require condition~\ref{cond:A} (among other properties). We will now show that this property is not satisfied for the gradient-regularized system.

Consider the point $x=[w_2,0,a,0]$ and let $v$ be defined as follows:
\begin{align}
    v &= [2w_2 a^2 + 4 \eta a^2 (\sigma^2 - a^2) , 0, a(a^2 - \sigma^2), 0]
\end{align}
where $v$ is actually derived by considering the field formed by \emph{crossing the curl} for the 2-d subspace with $w_2$ and $a$ only.

$F_{reg}^\top  v$ is 0 as expected.
\begin{align}
    F_{reg}^\top  v &= - 2 w_2 a^2 (\sigma^2 - a^2) - 4 \eta a^2 (\sigma^2 - a^2)^2 + 2 w_2 a^2 (\sigma^2 - a^2) + 4 \eta a^2 (\sigma^2 - a^2)^2 \\
    &= 0
\end{align}

It suffices to consider the submatrix of the Jacobian corresponding to $w_2$ and $a$ only when computing $v^\top  J v$:
\begin{align}
    \frac{1}{2} v^\top  J_{reg} &=
\begin{bmatrix}
2w_2 a^2 + 4 \eta a^2 (\sigma^2 - a^2) & a(a^2 - \sigma^2)
\end{bmatrix}
\begin{bmatrix}
0 & a \\
-a & -w_2 - 2 \eta (\sigma^2 - 3a^2)
\end{bmatrix} \\
&=
\begin{bmatrix}
-a^2 (a^2 - \sigma^2) & 2 w_2 a^3 + 4 \eta a^3 (\sigma^2 - a^2) - w_2 a(a^2 - \sigma^2) + 2 \eta a (a^2 - \sigma^2) (3a^2 - \sigma^2)
\end{bmatrix} \\
&=
\begin{bmatrix}
-a^2 (a^2 - \sigma^2) & w_2 a (a^2 + \sigma^2) + 2 \eta a (a^2 - \sigma^2)^2
\end{bmatrix} \\
\frac{1}{2} v^\top  J_{reg} v &=
\begin{bmatrix}
-a^2 (a^2 - \sigma^2) & w_2 a (a^2 + \sigma^2) + 2 \eta a (a^2 - \sigma^2)^2
\end{bmatrix}
\begin{bmatrix}
2w_2 a^2 + 4 \eta a^2 (\sigma^2 - a^2) \\
a(a^2 - \sigma^2)
\end{bmatrix} \\
&= -2w_2 a^4 (a^2 - \sigma^2) + 4 \eta a^4 (a^2 - \sigma^2)^2 + w_2 a^2 (a^2 + \sigma^2) (a^2 - \sigma^2) + 2 \eta a^2 (a^2 - \sigma^2)^3 \\
&= w_2 a^2 (a^2 - \sigma^2) [ -2a^2 + (a^2 + \sigma^2) ] + 2 \eta a^2 (a^2 - \sigma^2)^2 [ 2a^2 + (a^2 - \sigma^2) ] \\
&= -w_2 a^2 (a^2 - \sigma^2)^2 + 2 \eta a^2 (a^2 - \sigma^2)^2 ( 3a^2 - \sigma^2 )
\end{align}

If $w_2>0$ and $a<\frac{\sigma}{\sqrt{3}}$, then there isn't an $\eta \ge 0$ that will make this system quasimonotone.

The Jacobian of $F^{w_2,a}_{reg}$ for the ($w_2,a$)-subsystem is
\begin{align}
    J^{w_2,a}_{reg} &= \begin{bmatrix}
    0 & 2a \\ -2a & -2w_2 - 4\eta(\sigma^2-3a^2)
    \end{bmatrix}.
\end{align}
The trace of $J^{w_2,a}$ is strictly negative for $w_2>0$ and $a<\sigma/\sqrt{3}$, which implies $J^{w_2,a}_{reg}$ has an eigenvalue with strictly negative real part. Therefore, $J^{w_2,a}_{reg}$ is not Hurwitz.
\end{proof}

\begin{proposition}
\label{Funrnotquasi}
$F_{unr}$ is not quasimonotone or Hurwitz for the $(w_2,a)$-subsystem. Also, its Jacobian is not Hurwitz.
\end{proposition}
\begin{proof}
We consider Unrolled GAN as described in~\cite{metz2016unrolled}. Some of the necessary arithmetic can be found in the supplementary Mathematica notebook. Define the discriminator's update as
\begin{align}
    w_{2,k+1} &= w_{2,k} - \alpha F_{w_2}(w_{2,k},a_k) = U_{k}(w_{2,k}),
\end{align}
where $\alpha > 0$ is a step size, and denote the composition of $U$, $\Delta k$-times as
\begin{align}
    U^{\Delta k}_{k}(w_{2,k}) &= U_{k}(\cdots(U_{k}(U_{k}(w_{2,k}))\cdots)
\end{align}
where $\Delta k$ is some positive integer. Then the update for Unrolled GANs is
\begin{align}
    w_{2,k+1} &= w_{2,k} - \alpha \frac{\partial V(w_{2,k},a_k)}{\partial w_2} \\
    a_{k+1} &= a_k - \alpha \frac{\partial V(U^{\Delta k}_{k}(w_{2,k}),a_k)}{\partial a}.
\end{align}
In the case of the ($w_2,a$)-subsystem, we can write these unrolled updates out explicitly. Remember $F=[a^2-\sigma^2,-2aw_2]$, so
\begin{align}
    U_{k}(w_{1,k}) &= w_{2,k} - \alpha (a_k^2-\sigma^2), \\
    U^{\Delta k}_{k}(w_{2,k}),a_k) &= w_{2,k} - \alpha \Delta k (a_k^2-\sigma^2).
\end{align}
Plugging this back in, we find
\begin{align}
    \begin{bmatrix}
    w_{2,k+1} \\ a_{k+1}
    \end{bmatrix} &= \begin{bmatrix}
    w_{2,k} \\ a_{k}
    \end{bmatrix} - \alpha F_{unr},
\end{align}
where the corresponding map is
\begin{align}
    F_{unr} &= \begin{bmatrix}
    a^2 -\sigma^2 \\
    4 \alpha \Delta k a^3 - 2 a (2 \alpha \Delta k \sigma^2 + w_2)
    \end{bmatrix} \label{eqn:Funr}.
\end{align}
We will use the following vector to test condition~\ref{cond:A} for quasimonotonicity of $F_{unr}$:
\begin{align}
    v &= \begin{bmatrix}
    0 & 1 \\ -1 & 0
    \end{bmatrix} F_{unr}. \label{eqn:vunr}
\end{align}
Computing $v^\top J_{unr} v$ and evaluating at $(w_2=1,a=\frac{\sigma^2}{\sqrt{3}})$ gives
\begin{align}
    v^\top J_{unr} v &= -\frac{8}{9} \sigma^4 < 0, \label{eqn:vJvunr}
\end{align}
therefore, $F_{unr}$ is not quasimonotone.

If we examine the determinant of $J_{unr}$ and evaluate it at $a=\frac{\sigma}{\sqrt{3}}$, we get
\begin{align}
    Det[J_{unr}]\Big\vert_{a=\frac{\sigma}{\sqrt{3}}} &= -2 w_2,
\end{align}
which is less than zero for positive $w_2$. Therefore, the Jacobian exhibits negative eigenvalues which means the system is not Hurwitz.
%
\end{proof}

\begin{proposition}
\label{Faltnotquasi}
$F_{alt}$ is not quasimonotone or Hurwitz for the $(w_2,a)$-subsystem. Also, its Jacobian is not Hurwitz.
\end{proposition}
\begin{proof}
We consider an alternating gradient descent scheme. Some of the necessary arithmetic can be found in the supplementary Mathematica notebook. First, we begin with the case where the discriminator updates first. The updates are
\begin{align}
    w_{2,k+1} &= w_{2,k} - \alpha (a_k^2 - \sigma^2) \\
    a_{k+1} &= a_k - \alpha (-2a_k w_{2,k+1}) \\
    &= a_k - \alpha (-2a_k w_{2,k} + 2 a_k \alpha (a_k^2 - \sigma^2)) \\
    &= a_k - \alpha (2 \alpha a_k^3  - 2 a_k (\alpha \sigma^2 + w_{2,k})),
\end{align}
where $\alpha > 0$ is a step size. The corresponding map is
\begin{align}
    F_{alt} &= \begin{bmatrix}
    a^2 -\sigma^2 \\
    2 \alpha a^3 - 2 a (\alpha \sigma^2 + w_2)
    \end{bmatrix}.
\end{align}
Note the similarity to the Unrolled GAN map Equation~(\ref{eqn:Funr}). The maps are equivalent if $\Delta k = \sfrac{1}{2}$. Unrolled GANs was shown to be not quasimonotone for any $\Delta k$, therefore, $F_{alt}$ is not quasimonotone as well.

If we examine the trace of $J_{alt}$ and evaluate it at ($w_2=5 \alpha \sigma^2, a=\sigma$), we get
\begin{align}
    Tr[J_{alt}]\Big\vert_{(w_2=5 \alpha \sigma^2, a=\sigma)} &= -6 \alpha \sigma^2,
\end{align}
which is strictly negative. Therefore, the Jacobian exhibits negative eigenvalues which means the system is not Hurwitz.

Now, consider the generator updating first. The updates are
\begin{align}
    w_{2,k+1} &= w_{2,k} - \alpha (a_{k+1}^2 - \sigma^2) \\
    &= w_{2,k} - \alpha ((a_k - \alpha (-2a_k w_{2,k}))^2 - \sigma^2) \\
    a_{k+1} &= a_k - \alpha (-2a_k w_{2,k}),
\end{align}
where the corresponding map is
\begin{align}
    F_{alt'} &= \begin{bmatrix}
    a^2 -\sigma^2 \\
    2 \alpha a^3 - 2 a (\alpha \sigma^2 + w_2)
    \end{bmatrix}.
\end{align}
Testing for condition~\ref{cond:A} as before (see Equations~(\ref{eqn:Funr})-~(\ref{eqn:vJvunr})), we find that
\begin{align}
    v^\top J_{alt'} v &= -\frac{1}{2} \sigma^4 w_2 + 4 \alpha \sigma^4 w_2^2 + 16 c^2 \sigma^4 w_2^3 + 
 16 c^3 \sigma^4 w_2^4 + 8 c^4 \sigma^4 w_2^5.
\end{align}
Using Descartes' Rule of Signs~\cite{descartes1886geometrie}, we can determine that this expression has exactly one positive root for $w_2$. This implies that $v^\top J_{alt'} v$ changes sign locally around this root when varying $w_2$, which means $v^\top J_{alt'} v < 0$ for some positive $w_2$. Therefore $F_{alt'}$ is not quasimonotone.

If we examine the determinant of $J_{alt'}$ and evaluate it at ($w_2=1,a=\sigma$), we get
\begin{align}
    Det[J_{alt'}]\Big\vert_{(w_2=1,a=\sigma)} &= -8 \alpha (1 + 2 \alpha (2 + \alpha (2 + \alpha))) \sigma^2,
\end{align}
which is less than zero for positive $w_2$. Therefore, the Jacobian exhibits negative eigenvalues which means the system is not Hurwitz.
%
\end{proof}

\subsubsection{Monotonicity of $F_{cc}$, $F_{eg}$, and $F_{con}$ for the ($w_2,a$)-Subsystem}
The following propositions concern the monotonicity of $F_{cc}$, $F_{eg}$, and $F_{con}$ for the ($w_2,a$)-subsystem. The field and Jacobian for $F_{lin}$ will be helpful for proofs of their properties.

\begin{align}
F^{w_2,a}_{lin} &=
\begin{bmatrix}
\alpha (-\sigma^2 + a^2) + 4(\beta + \gamma) w_2 a^2 \\
2a (\beta + \gamma) (- \sigma^2 + a^2) + 4 (\beta - \gamma) w_2^2 a - 2 \alpha w_2 a
\end{bmatrix}
\end{align}

\begin{align}
J^{w_2,a}_{lin} &=
\begin{bmatrix}
4 (\beta + \gamma) a^2 & 2\alpha a + 8(\beta+\gamma)w_2 a \\
8(\beta-\gamma)w_2 a - 2\alpha a & 2(\beta+\gamma)(-\sigma^2+3a^2) + 4(\beta-\gamma)w_2^2 - 2\alpha w_2
\end{bmatrix}
\end{align}

\begin{proposition}
$F_{con} = F + \beta J^\top F$ is not quasimontone for the ($w_2,a$)-subsystem. Also, its Jacobian is not Hurwitz.
\end{proposition}
\begin{proof}
This corresponds to $F_{lin}$ with $\alpha=1,\beta=\beta,\gamma=0$. We consider three cases. Let

\begin{align}
F^{w_2,a}_{con} &=
\begin{bmatrix}
(-\sigma^2 + a^2) + 4\beta w_2 a^2 \\
2a \beta (- \sigma^2 + a^2) + 4\beta w_2^2 a - 2 w_2 a
\end{bmatrix},
\end{align}

\begin{align}
J^{w_2,a}_{con} &=
\begin{bmatrix}
4 \beta a^2 & 2 a + 8 \beta w_2 a \\
8\beta w_2 a - 2 a & 2 \beta (-\sigma^2+3a^2) + 4 \beta w_2^2 - 2 w_2
\end{bmatrix},
\end{align}

\begin{align}
    v &= \begin{bmatrix}
    0 & -1 \\
    1 & 0
    \end{bmatrix}
    \begin{bmatrix}
(-\sigma^2 + a^2) + 4\beta w_2 a^2 \\
2a \beta (- \sigma^2 + a^2) + 4\beta w_2^2 a - 2 w_2 a
\end{bmatrix} \\
    &= \begin{bmatrix}
-2a \beta (- \sigma^2 + a^2) - 4\beta w_2^2 a + 2 w_2 a \\
(-\sigma^2 + a^2) + 4\beta w_2 a^2
\end{bmatrix}.
\end{align}

\textbf{Case 1}: Consider $x=[0,2\sigma]$. Then

\begin{align}
    v^\top  J^{w_2,a}_{con} v &= 18 \beta \sigma^6 (11 + 128 \beta^2 \sigma^2),
\end{align}
which implies $\beta\ge0$ for the system to be quasimonotone.

\textbf{Case 2}: Consider $x=[0,1/2\sigma]$. Then

\begin{align}
    v^\top  J^{w_2,a}_{con} v &= \frac{9}{32} \beta \sigma^6 (-1 + 2 \beta^2 \sigma^2),
\end{align}
which, combined with above, implies $\beta\ge \frac{1}{\sqrt{2}\sigma} \approx \frac{0.707}{\sigma}$ for the system to be quasimonotone.

\textbf{Case 3}: Consider $x=[2\sigma,\sigma]$. Then

\begin{align}
    v^\top  J^{w_2,a}_{con} v &= 64 \beta \sigma^6 (1 + 4 \beta \sigma (1 - 7 \beta \sigma)).
\end{align}
The quantity in parentheses must be positive for this system to be quasimonotone. This quantity is a concave quadratic form with an upper root of $\approx \frac{0.273}{\sigma}$. This implies $\beta\le \approx \frac{0.273}{\sigma}$ for the system to be quasimonotone.

The last two results cannot be satisfied by a single $\beta$, therefore, this system is not quasimonotone.

For completeness, we analyze the limit where the $F$ term is ignored. Consider $a=c\sigma$.

\begin{align}
    v^\top  J^{w_2,a}_{con} v &= 16 c^4 (1 + 6 c^2 - 119 c^4) \sigma^8
\end{align}
This is negative for $c=1$, therefore, this system is not quasimonotone.

The trace of $J^{w_2,a}_{con}$ is strictly negative for $w_2=0$ and $a<\sigma/\sqrt{5}$, which implies $J^{w_2,a}_{con}$ has an eigenvalue with strictly negative real part. Therefore, $J^{w_2,a}_{con}$ is not Hurwitz.
\end{proof}

\begin{proposition}
\label{consensusnotquasi}
$F_{con} = \beta J^\top F$ is not quasimontone for the ($w_2,a$)-subsystem. Also, its Jacobian is not Hurwitz.
\end{proposition}
\begin{proof}
This corresponds to $F_{lin}$ with $\alpha=0,\beta=\beta,\gamma=0$. We consider two cases.

\begin{align}
F^{w_2,a}_{con} &=
\begin{bmatrix}
4\beta w_2 a^2 \\
2a \beta (- \sigma^2 + a^2) + 4\beta w_2^2 a
\end{bmatrix}
\end{align}

\begin{align}
J^{w_2,a}_{con} &=
\begin{bmatrix}
4 \beta a^2 & 8 \beta w_2 a \\
8\beta w_2 a & 2 \beta (-\sigma^2+3a^2) + 4 \beta w_2^2
\end{bmatrix}
\end{align}

\begin{align}
    v &= \begin{bmatrix}
    0 & -1 \\
    1 & 0
    \end{bmatrix}
    \begin{bmatrix}
4\beta w_2 a^2 \\
2a \beta (- \sigma^2 + a^2) + 4\beta w_2^2 a
\end{bmatrix} \\
    &= \begin{bmatrix}
-2a \beta (- \sigma^2 + a^2) - 4\beta w_2^2 a \\
4\beta w_2 a^2
\end{bmatrix}
\end{align}

\textbf{Case 2}: Consider $x=[0,c\sigma]$. Then

\begin{align}
    v^\top  J^{w_2,a}_{con} v &= 16 \beta^3 c^4 \sigma^8 (c^2 - 1)^2
\end{align}
which, for $c\ne 1$, implies $\beta\ge0$ for the system to be quasimonotone.

\textbf{Case 2}: Consider $x=[2c\sigma,c\sigma]$. Then

\begin{align}
    v^\top  J^{w_2,a}_{con} v &= -16 \beta^3 c^4 \sigma^8 (-1 - 6c^2 + 119c^4)
\end{align}
which, for $c = 1$, implies $\beta\le0$ for the system to be quasimonotone. Combined with above, this implies $\beta=0$ for the system to be quasimonotone. In conclusion, $\beta J^\top  F$ is not quasimonotone.

The trace of $J^{w_2,a}_{con}$ is strictly negative for $w_2=0$ and $a<\sigma/\sqrt{5}$, which implies $J^{w_2,a}_{con}$ has an eigenvalue with strictly negative real part. Therefore, $J^{w_2,a}_{con}$ is not Hurwitz.
\end{proof}

\begin{proposition}
$F_{eg} = F - \gamma JF$ requires $\gamma \rightarrow \infty$ to be pseudomonotone for ($w_2,a$)-subsystem
\end{proposition}
\begin{proof}
This corresponds to $F_{lin}$ with $\alpha=1,\beta=0,\gamma=\gamma$. We consider two cases.

\begin{align}
F^{w_2,a}_{eg} &=
\begin{bmatrix}
(-\sigma^2 + a^2) + 4\gamma w_2 a^2 \\
2a \gamma (- \sigma^2 + a^2) - 4 \gamma w_2^2 a - 2 w_2 a
\end{bmatrix}
\end{align}

\begin{align}
J^{w_2,a}_{eg} &=
\begin{bmatrix}
4 \gamma a^2 & 2 a + 8 \gamma w_2 a \\
-8\gamma w_2 a - 2 a & 2 \gamma (-\sigma^2+3a^2) - 4 \gamma w_2^2 - 2 w_2
\end{bmatrix}
\end{align}

\begin{align}
    v &= \begin{bmatrix}
    0 & -1 \\
    1 & 0
    \end{bmatrix}
    \begin{bmatrix}
(-\sigma^2 + a^2) + 4\gamma w_2 a^2 \\
2a \gamma (- \sigma^2 + a^2) - 4 \gamma w_2^2 a - 2 w_2 a
\end{bmatrix} \\
    &= \begin{bmatrix}
-2a \gamma (- \sigma^2 + a^2) + 4 \gamma w_2^2 a + 2 w_2 a \\
(-\sigma^2 + a^2) + 4\gamma w_2 a^2
\end{bmatrix}
\end{align}

\textbf{Case 1}: Consider $y=[\sigma,3\sigma]$ and $x=[3\sigma,5\sigma]$. Then

\begin{align}
    \langle F(y), x-y \rangle &= 2\sigma F_{w_2}(y) + 2\sigma F_{a}(y) = 2\sigma^3 \Big[ 8 + 36 \gamma \sigma + 48 \gamma \sigma - 12 \gamma \sigma - 6 \Big] \\
    &= 4\sigma^3 (1 + 36 \sigma \gamma) \\
    \langle F(x), x-y \rangle &= 2\sigma F_{w_2}(x) + 2\sigma F_{a}(x) = 12\sigma^3 ( -1 + 60 \sigma \gamma )
\end{align}
Then $\gamma \le -\frac{1}{36\sigma} \approx -\frac{0.027}{\sigma}$ or $\gamma \ge \frac{1}{60\sigma} \approx \frac{0.017}{\sigma}$ for the system to be quasimonotone.

\textbf{Case 2}: Consider $y=[\sigma,20\sigma]$ and $x=[20\sigma,5\sigma]$. Then

\begin{align}
    \langle F(y), x-y \rangle &= 19\sigma F_{w_2}(y) - 15\sigma F_{a}(y) = \sigma^3 (8181 - 207800 \sigma \gamma) \\
    \langle F(x), x-y \rangle &= 19\sigma F_{w_2}(x) - 15\sigma F_{a}(x) = 32\sigma^3 (108 + 4825 \sigma \gamma)
\end{align}
Then $\gamma \ge \frac{8181}{207800\sigma} \approx \frac{0.039}{\sigma}$ or $\gamma \ge \frac{108}{4825\sigma} \approx -\frac{0.022}{\sigma}$ for the system to be quasimonotone. The latter condition is more lenient, so the former is unnecessary.

For the system to be quasimonotone in both scenarios, we require that $\gamma \ge \frac{1}{60\sigma}$. This implies $\gamma$ must be arbitrarily large for small $\sigma$. In the limit, the effect of $F$ on the system is negligible. We consider this limit next.
\end{proof}

\begin{proposition}
\label{Fegpseudo}
$F_{eg} = - \gamma JF$ is pseudomonotone for ($w_2,a$)-subsystem.
\end{proposition}

\begin{proof}
Consider $x = [w_2,c\sigma]$ w.l.o.g.

Note this system is 2-d, therefore, there is only 1 vector $v$ (aside from scaling) that is perpendicular to $F$.
\begin{align}
    v^\top  J v &= 16 c^4 \sigma^6 ((-1 + c^2)^2 \sigma^2 + 2 (1 + c^2) w2^2) \ge 0 \,\, \forall \,\, c>0, w_2 \label{vjvfeg} \\
    \langle F(x) , x- x^* \rangle &= 2 c \sigma^2 ((-1 + c)^2 (1 + c) \sigma^2 + 2 w_2^2) \ge 0 \,\, \forall \,\, c>0, w_2
\end{align}

This satisfies conditions~\ref{cond:A} and~\ref{cond:C'}, therefore, this system is pseudomonotone.
\end{proof}

\begin{proposition}
\label{Fegpseudocon}
$F_{eg} = F-\gamma JF$ is pseudomonotone for the constrained $(w_2,a)$-subsystem.
\end{proposition}

\begin{proof}
We consider $\alpha=1$ in this case and let the user define a feasible region for which they are confident the equilibrium exists: $w_2 \in [w_2^{\min},w_2^{\max}]$ and $a \in [a_{\min},a_{\max}]$\textemdash the most important bounds being those on $a$. We will attempt to find a value for $\gamma$ that ensures the system is pseudomonotone within this region.

A partially sufficient (and necessary) condition for pseudomonotonicity is the following (see condition~\ref{cond:C'}).

\begin{align}
    \langle F(x), x-x^* \rangle &= 2 \gamma \Big( a (a - \sigma)^2 (a + \sigma) + 2 a \sigma w_2^2 \Big) - (a - \sigma)^2 w_2 \ge 0 \\
    \Rightarrow \gamma &\ge \frac{\overbrace{(a - \sigma)^2}^{a_1} w_2}{2 \Big( \underbrace{a (a - \sigma)^2 (a + \sigma)}_{a_0} + \underbrace{2 a \sigma}_{a_2} w_2^2 \Big)}
\end{align}

We can find the $w_2$ that maximizes this equation for a given $a$ by setting the derivative equal to zero and taking the positive root of the resulting quadratic. The denominator of the derivative is non-negative and only zero at equilibrium\textemdash this is not a concern because $\langle F(x), x-x^* \rangle = 0$ at equilibrium. Continuing and looking at the numerator of the derivative, we find

\begin{align}
    0 &= a_1 (a_0+a_2 d^2) - 2 a_1 a_2 d^2 \\
    &= a_1 ( a_0 - a_2 d^2 ) \\
    d^* &= \sqrt{a_0/a_2} \\
    &= \sqrt{\frac{(a-\sigma)^2(a+\sigma)}{2\sigma}}.
\end{align}

If we plug that back into the lower bound for $\gamma$, we get

\begin{align}
    \gamma &\ge \frac{|a - \sigma|^3 \sqrt{a+\sigma}/\sqrt{2\sigma}}{4 a (a - \sigma)^2 (a + \sigma)} \\
    &= \frac{|a - \sigma|}{4\sqrt{2} a \sigma^{1/2} \sqrt{a + \sigma}} \le \frac{a_{\max}}{4\sqrt{2} a_{\min}^2} \\
    &\ge \frac{a_{\max}}{4\sqrt{2} a_{\min}^2}
\end{align}

The condition above along with the following (see condition~\ref{cond:A}) are sufficient to ensure pseudomonotonicity.

\begin{align}
    v^\top  J v &= 16 a^4 \gamma^3 ( (a^2-\sigma^2)^2 + 2 w_2^2 (a^2 + \sigma^2)) \\
    &+ 16 \gamma^2 w_2 a^2 (2\sigma^2 w_2^2 + (a^2-\sigma^2)^2) \\
    &+ 2 \gamma ((a^2 - \sigma^2)^2 (3a^2 - \sigma^2) + w_2^2 (8 a^2 \sigma^2 - 2(a^2-\sigma^2)^2)) \\
    &-2 w_2 (a^2-\sigma^2)^2
\end{align}

If $w_2 \le 0$, then this quantity is greater than or equal to zero due to the result in equation~(\ref{vjvfeg}), which we have already shown to be greater than zero. Therefore, we focus on $w_2 > 0$. We can divide the analysis into two cases.

Consider $3a^2 \ge \sigma^2$. In this case, all coefficients of $\gamma$ terms except a $\gamma^1$ term and the last term (the constant) are positive. For simplicity, we can find the value for $\gamma$ such that the first part of the $\beta^2$ coefficient is greater than the two negative terms.

\begin{align}
    &16 w_2 a^2 \gamma^2 (a^2-\sigma^2)^2 - 4 \gamma  w_2^2 (a^2-\sigma^2)^2 - 2w_2 (a^2-\sigma^2)^2 \\
    &= 2 w_2 (a^2-\sigma^2) ( 8a^2 \gamma^2 - 2w_2 \gamma - 1 ) \ge 0 \\
    \Rightarrow & \gamma \ge \frac{2w_2 + \sqrt{4w_2^2 + 4(8a^2)}}{16a^2} \le \frac{w_2}{8a^2} + \frac{w_2 + \sqrt{8}a}{8a^2} \\
    &\Rightarrow \gamma \ge 
    \frac{w_2^{\max}}{4a_{\min}^2} + \frac{1}{2\sqrt{2}a_{\min}}
\end{align}

Now consider $3a^2 < \sigma^2$. One of the terms in the $\gamma^1$ coefficient is now negative. We will find a value for $\gamma$ such that the $\gamma^3$ term can drown out that negative term.

\begin{align}
    &16 a^4 \gamma^3 (a^2-\sigma^2)^2 - 2 \gamma (a^2 - \sigma^2)^2 (\sigma^2 - 3a^2) \\
    &\ge 2\gamma (a^2-\sigma^2)^2 (8a^4 \gamma^2 - \sigma^2) \\
    \Rightarrow & \gamma \ge \frac{\sigma}{2\sqrt{2}a^2} \\
    \Rightarrow & \gamma \ge \frac{a_{\max}}{2\sqrt{2}a_{\min}^2}
\end{align}

Combining the results, we have that

\begin{align}
    \gamma &\ge \max \Big\{ \frac{a_{\max}}{2\sqrt{2}a_{\min}^2} , \frac{w_2^{\max}}{4a_{\min}^2} + \frac{1}{2\sqrt{2}a_{\min}} \Big\}
\end{align}

Note this bound is not tight; it is just meant to provide a satisfactory estimate.    
\end{proof}

\begin{proposition}
\label{Fccpseudolim}
$F_{cc} = F + \beta (J^\top -J) F$ requires $\beta \rightarrow \infty$ to be pseudomonotone for the ($w_2,a$)-subsystem.
\end{proposition}
\begin{proof}
This corresponds to $F_{lin}$ with $\alpha=1,\gamma=\beta/2,\beta=\beta/2$.

\begin{align}
F^{w_2,a}_{cc} &=
\begin{bmatrix}
(-\sigma^2 + a^2) + 4\beta w_2 a^2 \\
2a \beta (- \sigma^2 + a^2) - 2 w_2 a
\end{bmatrix}
\end{align}

\begin{align}
J^{w_2,a}_{cc} &=
\begin{bmatrix}
4 \beta a^2 & 2 a + 8 \beta w_2 a \\
- 2 a & 2 \beta (-\sigma^2+3a^2) - 2 w_2
\end{bmatrix}
\end{align}

\begin{align}
    v &= \begin{bmatrix}
    0 & -1 \\
    1 & 0
    \end{bmatrix}
    \begin{bmatrix}
(-\sigma^2 + a^2) + 4\beta w_2 a^2 \\
2a \beta (- \sigma^2 + a^2) - 2 w_2 a
\end{bmatrix} \\
    &= \begin{bmatrix}
-2a \beta (- \sigma^2 + a^2) + 2 w_2 a \\
(-\sigma^2 + a^2) + 4\beta w_2 a^2
\end{bmatrix}
\end{align}

\textbf{Case 1}: Consider $x=[0,2\sigma]$. Then

\begin{align}
    v^\top  J^{w_2,a}_{cc} v &= 18 \beta \sigma^6 (11 + 128 \beta^2 \sigma^2)
\end{align}
implies that $\beta \ge 0$.

\textbf{Case 2}: Consider $x=[0,1/2\sigma]$. Then

\begin{align}
    v^\top  J^{w_2,a}_{cc} v &= \frac{9}{32} \beta \sigma^6 (-1 + 2 \beta^2 \sigma^2)
\end{align}
this, combined with above, implies that $\beta \ge \frac{1}{\sqrt{2}\sigma}$.

This implies $\beta$ must be arbitrarily large for small $\sigma$. In the limit, the effect of $F$ on the system is negligible. We consider this limit in Subsubsection~\ref{Fccpseudow2a}.
\end{proof}

\begin{proposition}
\label{Fccpseudow2a}
$F_{cc} = (J^\top -J) F$ is pseudomonotone for the $(w_2,a)$-subsystem.
\end{proposition}
\begin{proof}
\begin{align}
F^{w_2,a}_{cc} &= [8w_2 a^2, 4a(a^2-\sigma^2)] \\
J^{w_2,a}_{cc} &=
\begin{bmatrix}
8a^2 & 16w_2a \\
0 & 4(3a^2-\sigma^2)
\end{bmatrix}
\end{align}

Note that the skew part of the Jacobian of $F$ is full rank except at the boundary ($a=0$), so $F_{cc}=(J^\top -J)F$ maintains the same fixed points. This can be seen by looking at $F_{cc}$ above. We will simply need to constrain $a$ to be greater than 0.

In order for a system to be quasimonotone, we require condition~\ref{cond:A} (among other properties). We will now show that this property is satisfied for the ($w_2,a$)-subsystem.

\textbf{Case 1}: Consider the point $x=[w_2,a]$ and let $v$ be defined as follows:
\begin{align}
    v = F &= [-\sigma^2 + a^2, -2w_2 a]^\top .
\end{align}

$v^\top  F^{w_2,a}_{cc}$ is 0 as expected.
\begin{align}
    v^\top  F^{w_2,a}_{cc} &= -8 w_2 a^2 \sigma^2 + 8 w_2 a^4 - 8w_2 a^4 + 8w_2a^2 \sigma^2 \\
    &= 0
\end{align}

Now, we will compute $v^\top  J^{w_2,a}_{cc} v$ to see if it is greater than zero.

\begin{align}
v^\top  J^{w_2,a}_{cc} &=
\begin{bmatrix}
-\sigma^2 + a^2 & -2w_2 a
\end{bmatrix}
\begin{bmatrix}
8a^2 & 16w_2a \\
0 & 4(3a^2-\sigma^2)
\end{bmatrix} \\
&=
\begin{bmatrix}
-8 \sigma^2 a^2 + 8 a^4 & 16w_2a(a^2 - \sigma^2) - 8 w_2 a (3a^2-\sigma^2)
\end{bmatrix} \\
&=
\begin{bmatrix}
8 a^2 (a^2 - \sigma^2) & -8w_2a (a^2 + \sigma^2)
\end{bmatrix} \\
v^\top  J^{w_2,a}_{cc} v &=
\begin{bmatrix}
8 a^2 (a^2 - \sigma^2) & -8w_2a (a^2 + \sigma^2)
\end{bmatrix}
\begin{bmatrix}
-\sigma^2 + a^2 \\
-2w_2 a
\end{bmatrix} \\
&= 8 a^2 (a^2 - \sigma^2)^2 + 16 w_2^2 a^2 (a^2 + \sigma^2) \ge 0 \label{vjvfcc}
\end{align}

In addition to this, proving that $\langle F(x), x-x^* \rangle \ge 0$ is sufficient for proving condition~\ref{cond:C'}.

\begin{align}
\langle F^{w_2,a}_{cc}(y), y-x^* \rangle &= 8w_2 a^2 w_2 + 4a(a^2-\sigma^2)(a-\sigma) \ge 0
\end{align}

The last two terms of the sum are always the same sign due to the square function being ``monotone'' and the fact that $a$ is constrained to be non-negative. Therefore, $F_{cc}$ is pseudomonotone.
\end{proof}

\begin{proposition}
\label{Fccpseudocon}
$F_{cc} = F+\beta(J^\top -J)F$ is pseudomonotone for the constrained $(w_2,a)$-subsystem.
\end{proposition}

\begin{proof}
We consider $\alpha=1$ in this case and let the user define a feasible region for which they are confident the equilibrium exists: $w_2 \in [w_2^{\min},w_2^{\max}]$ and $a \in [a_{\min},a_{\max}]$\textemdash the most important bounds being those on $a$. We will attempt to find a value for $\beta$ that ensures the system is pseudomonotone within this region.

A partially sufficient (and necessary) condition for pseudomonotonicity is the following (see condition~\ref{cond:C'}).

\begin{align}
    \langle F(x), x-x^* \rangle &= 2 \beta \Big( a (a - \sigma)^2 (a + \sigma) + 2 a^2 w_2^2 \Big) - (a - \sigma)^2 w_2 \ge 0 \\
    \Rightarrow \beta &\ge \frac{\overbrace{(a - \sigma)^2}^{a_1} w_2}{2 \Big( \underbrace{a (a - \sigma)^2 (a + \sigma)}_{a_0} + \underbrace{2 a^2}_{a_2} w_2^2 \Big)}
\end{align}

We can find the $w_2$ that maximizes this equation for a given $a$ by setting the derivative equal to zero and taking the positive root of the resulting quadratic. The denominator of the derivative is non-negative and only zero at equilibrium\textemdash this is not a concern because $\langle F(x), x-x^* \rangle = 0$ at equilibrium. Continuing and looking at the numerator of the derivative, we find

\begin{align}
    0 &= a_1 (a_0+a_2 d^2) - 2 a_1 a_2 d^2 \\
    &= a_1 ( a_0 - a_2 d^2 ) \\
    d^* &= \sqrt{a_0/a_2} \\
    &= \sqrt{\frac{(a-\sigma)^2(a+\sigma)}{2a}}.
\end{align}

If we plug that back into the lower bound for $\beta$, we get

\begin{align}
    \beta &\ge \frac{|a - \sigma|^3 \sqrt{a+\sigma}/\sqrt{2a}}{4 a (a - \sigma)^2 (a + \sigma)} \\
    &= \frac{|a - \sigma|}{4\sqrt{2} a^{3/2} \sqrt{a + \sigma}} \le \frac{a_{\max}}{4\sqrt{2} a_{\min}^2} \\
    &\ge \frac{a_{\max}}{4\sqrt{2} a_{\min}^2}
\end{align}

The condition above along with the following (see condition~\ref{cond:A}) are sufficient to ensure pseudomonotonicity.

\begin{align}
    v^\top  J v &= 16 a^4 \beta^3 ( (a^2-\sigma^2)^2 + 2 w_2^2 (a^2 + \sigma^2)) \\
    &+ 32 \beta^2 w_2^3 a^4 \\
    &+ 2 \beta ((a^2 - \sigma^2)^2 (3a^2 - \sigma^2) + 8 a^4 w_2^2) \\
    &-2 w_2 (a^2-\sigma^2)^2
\end{align}

If $w_2 \le 0$, then this quantity is greater than or equal to zero due to the result in equation~(\ref{vjvfcc}), which we have already shown to be greater than zero. Therefore, we focus on $w_2 > 0$. We can divide the analysis into two cases.

Consider $3a^2 \ge \sigma^2$. In this case, all coefficients of $\beta$ terms except the last term (the constant) are positive. For simplicity, we can find the value for $\beta$ such that the first part of the $\beta^3$ coefficient is greater than the last term (the constant).

\begin{align}
    16 a^4 \beta^3 (a^2-\sigma^2)^2 - 2 w_2 (a^2-\sigma^2)^2 &\ge 0 \\
    \Rightarrow & \beta \ge \frac{1}{2} \Big( \frac{w_2^{\max}}{a_{\min}^4} \Big)^{1/3}
\end{align}

Now consider $3a^2 < \sigma^2$. One of the terms in the $\beta^1$ coefficient is now negative. We will find a value for $\beta$ such that the $\beta^3$ term can drown out the two negative terms.

\begin{align}
    &16 a^4 \beta^3 (a^2-\sigma^2)^2 - 2 \beta (a^2 - \sigma^2)^2 (\sigma^2 - 3a^2) - 2 w_2 (a^2-\sigma^2)^2 \\
    &= \frac{(a^2-\sigma^2)^2}{16a^4} \Big[ \beta^3 - \frac{2(\sigma^2 - 3a^2)}{16a^4} \beta - \frac{2w_2}{16a^4} \Big] \\
    &\ge \frac{(a^2-\sigma^2)^2}{16a^4} \Big[ \beta^3 - \underbrace{\frac{\sigma^2}{8a^4}}_{a_0} \beta - \underbrace{\frac{w_2}{8a^4}}_{a_1} \Big] \\
    &= \frac{(a^2-\sigma^2)^2}{16a^4} \Big[ 3a_0^{1/2} a_1^{2/3} + 2a_1 a_2^{2/3} \Big] \text{ for } \beta = a_0^{1/2} + a_1^{1/3} \\
    &\ge 0 \\
    \Rightarrow& \beta \ge a_0^{1/2} + a_1^{1/3} = \frac{1}{2\sqrt{2}} \frac{a_{\max}}{a_{\min}^2} + \frac{1}{2} \Big (\frac{w_2^{\max}}{a_{\min}^4} \Big)^{1/3}
\end{align}

This last lower bound is the greatest of the three, so it suffices to set $\beta$ greater than this value to ensure the system is pseudomonotone within the given feasible region.   
\end{proof}





\begin{proposition}
\label{w2anotmonotone}
$F_{lin}$ is not monotone for the ($w_2,a)$-subsystem (before scaling).
\end{proposition}

\begin{proof}
Let $F^{w_2,a}_{lin}$ be defined as follows:
\begin{align}
(\alpha I + \beta J^\top  - \gamma J) F &=
\begin{bmatrix}
\alpha & -2(\beta + \gamma) a  \\
2 (\beta + \gamma) a & \alpha - 2 (\beta - \gamma) w_2
\end{bmatrix}
\begin{bmatrix}
-\sigma^2 + a^2 \\
-2 w_2 a
\end{bmatrix} \\
&=
\begin{bmatrix}
\alpha (-\sigma^2 + a^2) + 4(\beta + \gamma) w_2 a^2 \\
2a (\beta + \gamma) (- \sigma^2 + a^2) + 4 (\beta - \gamma) w_2^2 a - 2 \alpha w_2 a.
\end{bmatrix}
\end{align}
Its Jacobian is then
\begin{align}
J^{w_2,a}_{lin} &= \begin{bmatrix}
4 (\beta+\gamma) a^2 & 2\alpha a + 8(\beta + \gamma) w_2 a \\
8 (\beta-\gamma) w_2 a - 2 \alpha a & 2(\beta + \gamma)(-\sigma^2 + 3a^2) + 4(\beta-\gamma)w_2^2 - 2 \alpha w_2
\end{bmatrix} \\
J_{sym} &= \begin{bmatrix}
4 (\beta+\gamma) a^2 & 8\beta w_2 a \\
8 \beta w_2 a & 2(\beta + \gamma)(-\sigma^2 + 3a^2) + 4(\beta-\gamma)w_2^2 - 2 \alpha w_2
\end{bmatrix}
\end{align}
The trace of the symmetrized Jacobian must be non-negative to ensure monotonicity because a negative trace implies the existence of a negative eigenvalue:
\begin{align}
Tr &= 2(\beta + \gamma)(-\sigma^2 + 5a^2) + 4(\beta-\gamma)w_2^2 - 2 \alpha w_2 \le 0 \,\, \forall a < \frac{\sigma}{\sqrt{5}}, w_2=0.
\end{align}
Assume $\beta+\gamma>0$. If $a<\sigma/\sqrt{5}$ and $w_2=0$, then the trace is less than zero.

Assume $\beta+\gamma<0$. If $a>\sigma/\sqrt{5}$ and $w_2=0$, then the trace is less than zero.

Assume $\gamma=-\beta$. Then
\begin{align}
    Tr &= 8\beta w_2^2 - 2 \alpha w_2 = 2 w_2 (4 \beta w_2 - \alpha).
\end{align}

If $w_2 < 0$, then $\beta \le \frac{\alpha}{4w_2}$. If $w_2 > 0$, then $\beta \ge \frac{\alpha}{4w_2}$. Therefore, $\beta = \frac{\alpha}{4w_2}$, however, $\beta$ and $\alpha$ are constants while $w_2$ is a variable. Therefore, $\alpha$ and $\beta$ must equal zero to satisfy this for all $w_2$ proving that no monotone linear combination exists.
\end{proof}

\begin{proposition}
\label{lincombnothurwitz}
$F_{lin}$ is not Hurwitz for the ($w_2,a)$-subsystem.
\end{proposition}

\begin{proof}
Consider $J^{w_2,a}_{lin}$ at $w_2=0$.
\begin{align}
J^{w_2,a}_{lin} &= \begin{bmatrix}
4 (\beta+\gamma) a^2 & 2\alpha a \\
- 2 \alpha a & 2(\beta + \gamma)(-\sigma^2 + 3a^2)
\end{bmatrix} \\
Tr &= 2(\beta + \gamma)(5a^2 - \sigma^2) \\
Det &= 8(\beta + \gamma)^2 (-\sigma^2 + 3a^2) a^2 + 4\alpha^2 a^2
\end{align}

If $\beta + \gamma < 0$, then $a>\sigma/\sqrt{5}$ implies the existence of an eigenvalue with negative real part. If $\beta + \gamma > 0$, then $a<\sigma/\sqrt{5}$ implies the existence of an eigenvalue with negative real part. If $\beta + \gamma = 0$, then the real part is zero.    
\end{proof}

\begin{proposition}
\label{Flinw2amonotone}
There exists an $F_{lin'}$ family after scaling by $\sfrac{1}{4a^2}$ that exhibits strict-monotonicity.
\end{proposition}
\begin{proof}
If we consider the same linear combinations above, but divide $F$ by $4a^2$, we can obtain a family of monotone fields (see Mathematica notebook).

The trace of the corresponding symmetrized Jacobian is

\begin{align}
    Tr &= \frac{(\beta + \gamma) (3 a^2 + \sigma^2) + \alpha w_2 + 2 (\gamma - \beta) w_2^2}{2 a^2}.
\end{align}

For constant $\beta$ and $\gamma$ and nonzero $\alpha$, there exists a value for $w_2$ that will force the trace to be negative, therefore $\alpha$ must be zero. Note that $\gamma$ must be greater than or equal to $\beta$ to ensure that the trace cannot be made negative in the limit as $w_2^2$ grows to infinity.

\textbf{Case 1}: Consider the case where $\beta=\gamma$. Then for any fixed $\beta$, $\gamma$, and nonzero $\alpha$,

\begin{align}
    w_2 &= -(3a^2+\sigma^2)\frac{\beta + \gamma}{\alpha} - \alpha
\end{align}

will cause the trace to be negative.

\textbf{Case 2}: Otherwise, consider solving the quadratic form for $w_2$ when $\beta + \gamma > 0$:

\begin{align}
    w_2 &= \frac{-\alpha \pm \sqrt{\alpha^2 - 8 (3a^2+\sigma^2) (\gamma - \beta) (\beta + \gamma)}}{4(\gamma - \beta)}.
\end{align}

For the trace to be non-negative, we need the leading coefficient of the quadratic to be positive, i.e., $\gamma - \beta > 0$. We also need there to be at most 1 real root, meaning the square root must be non-positive. If $\beta+\gamma>0$, then setting $a$ and $\sigma$ using the following formula will force the root to be positive:

\begin{align}
    3a^2 + \sigma^2 &< \frac{\alpha^2}{8(\gamma - \beta)(\beta+\gamma)}
\end{align}

For example, set $a=\sigma$, and then set $\sigma$ and $w_2$ as follows to force the trace to be negative:

\begin{align}
    \sigma &= \frac{3}{4}\frac{\alpha}{\sqrt{32(\gamma - \beta)(\beta+\gamma)}}, \\
    w_2 &= -\frac{\alpha}{4(\gamma-\beta)}.
\end{align}

\textbf{Case 3}: If $\beta+\gamma \le 0$, then the root is necessarily positive. Therefore, $\alpha$ must be set to zero.

The field and Jacobian are now wieldy enough to state:

\begin{align}
    F^{w_2,a}_{lin'} &= (\beta + \gamma) \begin{bmatrix}
    w_2, \frac{(a - \sigma) (a + \sigma)}{2a} - 4\Big(\frac{\gamma-\beta}{\beta+\gamma}\Big) \Big(\frac{w_2^2}{a}\Big)
    \end{bmatrix},
\end{align}
and
\begin{align}
    J^{w_2,a}_{lin'} &= (\beta+\gamma) \begin{bmatrix}
    1 & 0 \\
    -2 \Big(\frac{\gamma-\beta}{\beta+\gamma}\Big) \Big(\frac{w_2}{a}\Big) &
  \frac{1}{2} + \frac{\sigma^2}{2a^2} + \Big(\frac{\gamma-\beta}{\beta+\gamma}\Big) \Big(\frac{w_2^2}{a^2}\Big)
    \end{bmatrix}.
\end{align}

The trace is now

\begin{align}
    Tr &= \frac{(\beta + \gamma) (3 a^2 + \sigma^2) + 2 (\gamma - \beta) w_2^2}{2 a^2},
\end{align}

and is non-negative as long as both $\beta+\gamma \ge 0$ and $\gamma - \beta \ge 0$.

The determinant is

\begin{align}
    Det &= \frac{(\beta + \gamma)^2 (a^2 + \sigma^2) + 
 4 (\gamma - \beta) \beta w_2^2}{2 a^2},
\end{align}

which is non-negative as long as, in addition to the previous conditions, we have $\beta \ge 0$. The trace and determinant are both strictly positive if $\beta+\gamma > 0$.

In summary, $F^{w_2,a}_{lin'}$ is strictly-monotone, i.e., $J^{w_2,a}_{lin'} \succ 0$, if $\gamma \ge \beta \ge 0$ and $\gamma>0$.
%
\end{proof}

\begin{corollary}
\label{Fccegw2amonotone}
The $F_{lin'}$ family includes $F_{eg'}$ $(\gamma=\gamma,\beta=0)$ and $F_{cc'}$ $(\gamma=\beta)$. By Proposition~\ref{Flinw2amonotone}, $F_{eg'}$ and $F_{cc'}$ are at least strictly-monotone.
\end{corollary}

\begin{proposition}
\label{Fw2astrongccstricteg}
$F^{w_2,a}_{cc'}$ is $\sfrac{1}{2}$-strongly monotone and $F^{w_2,a}_{eg'}$ is only strictly-monotone.
\end{proposition}
\begin{proof}
We will look at both maps individually.

\textbf{Case $F^{w_2,a}_{cc'}$}:
The eigenvalues of $J^{w_2,a}_{cc'}$ are $\lambda_{1} = 1$ and $\lambda_2 = \frac{1}{2}\Big(1+\frac{\sigma^2}{a^2}\Big)$. Therefore, $J^{w_2,a}_{cc'} \succeq \frac{1}{2}$ and $F^{w_2,a}_{cc'}$ is $\sfrac{1}{2}$-strongly monotone.

\textbf{Case $F^{w_2,a}_{eg'}$}:
The eigenvalues of a $2 \times 2$ matrix can be written in terms of the trace and determinant as
\begin{align}
    \lambda_{1,2} &= \frac{Tr \pm \sqrt{Tr^2 - 4 Det}}{2} \\
    &= \frac{Tr}{2} \Big( 1 \pm \sqrt{1-\frac{4Det}{Tr^2}} \Big).
\end{align}

Therefore, if the term $\frac{4Det}{Tr^2}$ can be made arbitrarily small, then one of the eigenvalues can made arbitrarily close to zero. On the other hand, if this quantity has a finite lower bound, then the eigenvalues are lower bounded as a constant multiple of the trace.

The trace and determinant of $J^{w_2,a}_{eg'}$ are

\begin{align}
    Tr &= \frac{1}{2} \Big( 3 + \frac{\sigma^2}{a^2} \Big) + \frac{w_2^2}{a^2} \\
    Det &= \frac{1}{2} \Big( 1 + \frac{\sigma^2}{a^2} \Big).
\end{align}

and the quantity, $Q$, described is
\begin{align}
    Q &= \frac{8 a^2 (a^2 + \sigma^2)}{(3 a^2 + \sigma^2 + 2 w_2^2)^2}.
\end{align}

This term can be made arbitrarily small as $w_2$ goes to infinity. To be more rigorous, let $a=\sigma=1$ so that $Tr=2+w_2^2$ and $Det=1$. Then 

\begin{align}
    \lambda_{1,2} &= \frac{1}{2}(w_2^2+2)\Big(1-\sqrt{1-\frac{4}{w_2^2+2}}\Big) \\
    &= \frac{1}{2}\frac{\overbrace{\Big(1-\sqrt{1-\frac{4}{w_2^2+2}}\Big)}^{top}}{\underbrace{(w_2^2+2)^{-1}}_{bot}}.
\end{align}

An application of L'Hopital's rule shows that
\begin{align}
    \lim_{w_2 \rightarrow \infty} \frac{\sfrac{\partial top}{\partial w_2}}{\sfrac{\partial bot}{\partial w_2}} &= \frac{4}{(w_2^2 + 2) \sqrt{1 - \frac{4}{(w_2^2 + 2)^2}}} = 0.
\end{align}

The minimum eigenvalue only approaches zero in the limit, so $F^{w_2,a}_{eg'}$ is strictly-monotone.
\end{proof}

\begin{claim}
$F^{w_2,a}_{cc'}$ is the gradient of the following convex function: $f^{w_2,a}_{cc'} = w_2^2 + 1/2 \Big((a^2-\sigma^2) - \sigma^2 \log(\frac{a^2}{\sigma^2}\Big)$.
\end{claim}
\begin{proof}
The Jacobian of $F^{w_2,a}_{cc'}$ is symmetric and PSD, therefore it is the Hessian of some convex function. We can integrate $F^{w_2,a}_{cc'}$ to arrive at a convex function (with arbitrary constant). Integrating $F^{w_2,a}_{cc'}$ results in the following:
\begin{align}
    f^{w_2,a}_{cc'} &= w2^2 + 1/2 \Big((a^2-\sigma^2) - \sigma^2 \log\big(\frac{a^2}{\sigma^2}\big)\Big)
\end{align}
Note that $f^{w_2,a}_{cc'}$ must be convex along the subspace with $w_2=0$ as well, which implies that
\begin{align}
    g(a||\sigma) &= 1/2 \Big((a^2-\sigma^2) - \sigma^2 \log\big(\frac{a^2}{\sigma^2}\big)\Big)
\end{align}
is convex as well. This function is of individual interest because it may serve as a preferred alternative to KL-divergence.
\end{proof}

\subsection{Progressive Learning of LQ-GAN}

Here, we consider the stochastic setting where the GAN is trained using samples from $p(y)$ and $p(z)$. There are two ways to learn both the mean and variance of a distribution using $F^{w_2,a}_{cc}$. One is to first learn the mean to a high degree of accuracy, then stop learning the mean and start learning the variance. The other is to keep learning the mean with an appropriate weighting of the two systems to maintain stability. We discuss the former option first.

\begin{proposition}
\label{progshutoff}
Assume all $y \sim p(y)$ lie in $[y_{low},y_{hi}]$. After $k > \big( \frac{y_{hi}-y_{low}}{-|\mu|+\sqrt{\mu^2+d\sigma^2}} \big)^2 \log[\frac{\sqrt{2}}{\delta^{1/2}}]$ iterations, with probability, $1-\delta$, the $(w_1,b)$-subsystem can be ``shut-off'' and the ($w_2,a$)-subsystem safely ``turned-on'' resulting in a $\sfrac{1}{2}$-strongly-monotone $F^{w_2,a}_{cc'}$.
\end{proposition}

\begin{proof}
We begin by observing the symmetrized Jacobian of $F^{w_2,a}_{cc'}$:
\begin{align}
    J^{w_2,a}_{cc'} &= \begin{bmatrix}
        1 & 0 \\
        0 & \frac{a^2-b^2+\mu^2+\sigma^2}{2a^2}
    \end{bmatrix}
    = \begin{bmatrix}
        1 & 0 \\
        0 & \frac{G}{2a^2} + \frac{1}{2}
    \end{bmatrix},
\end{align}
where $G=\mu^2+\sigma^2-b^2$. In order for $F^{w_2,a}_{cc'}$ to be strongly monotone, we require $G \ge 0$. In other words, the square of the generator's estimate of the mean, $b_k$, learned from training the ($w_1,b$)-subsystem needs to be less than or equal to $\mu^2 + \sigma^2$.

Assume we are using $F^{w_1,b}_{cc'}$ with step size $\rho_k = \frac{1}{k+1}$ to train the ($w_1,b$)-subsystem. Note that this was shown equivalent to the standard running mean in Proposition~\ref{w1bmonotone}. Therefore, $b_k = Z = \frac{1}{K} \sum_{i=1}^k y_i$. Also, $\mathbb{E}[Z] = \mu$. Then, using Hoeffding's inequality, we find
\begin{align}
    Pr(|Z-\mathbb{E}[Z]|\ge t) &\le 2e^{-\frac{2kt^2}{(y_{hi}-y_{low})^2}} \\
    \Rightarrow Pr(|b_k-\mu|< t) &\ge 1-2e^{-\frac{2kt^2}{(y_{hi}-y_{low})^2}} = 1 - \delta \label{hoeffstop}
\end{align}

Assume $|b_k-\mu|< t$ and introduce a scalar: $0 < d < 1$. Remember, we require $b_k^2 < \mu^2 + \sigma^2$. And we know $\mu - t < b_k < \mu + t$ which implies
\begin{align}
    b_k^2 &< \mu^2 + \underbrace{t^2 + 2 |\mu|t}_{=d\sigma^2} < \mu^2 + \sigma^2 \label{baccurate} \\
    \Rightarrow 0 &= t^2 + 2 |\mu|t - d\sigma^2, t>0
\end{align}
This expression has two roots for $t$, one positive and one negative. $|b_k-\mu|$ can only be upper bounded by a positive number, so we select the positive root.
\begin{align}
    t_{roots} &= \frac{-2|\mu| \pm \sqrt{4\mu^2 + d4\sigma^2}}{2} \\
    &= -|\mu| \pm \sqrt{\mu^2 + d\sigma^2} \\
    t_{+} &= -|\mu| + \sqrt{\mu^2 + d\sigma^2}
\end{align} 
Plugging $t_+$ back into equation~(\ref{baccurate}) for $t$, we find that
\begin{align}
    G &= \mu^2 + \sigma^2 - b_k^2 > (1-d) \sigma^2.
\end{align}



Rearranging~(\ref{hoeffstop}) and plugging in $t$, we can derive the number of iterations required:
\begin{align}
    k &> \Big( \frac{y_{hi}-y_{low}}{-|\mu|+\sqrt{\mu^2+d\sigma^2}} \Big)^2 \log\big[\frac{\sqrt{2}}{\delta^{1/2}}\big]. \label{hoeff}
\end{align}

If we assume $p(y) \sim \mathcal{N}(\mu, \sigma^2)$ and use a Chernoff bound, we find
\begin{align}
    Pr(|b_k-\mu|< t) &\ge 1-2e^{-\frac{kt^2}{\sigma^2}} = 1 - \delta \label{chernstop} \\
    k &> \Big( \frac{\sigma}{-|\mu|+\sqrt{\mu^2+d\sigma^2}} \Big)^2 \log\big[\frac{2}{\delta}\big]. \label{chern}
\end{align}

The number of samples needed to maintain stability of the system grows as the true mean $\mu$ deviates from zero. This is not an artifact of the concentration inequalities (it occurs with both), but of the parameterization of the LQ-GAN\textemdash the samples are not mean centered before being passed to the quadratic discriminator, i.e., $w_2 y^2$ rather than $w_2 (y-\mu)^2$. This may explain why batch norm is so helpful (almost required) in stabilizing training.

\end{proof}

\begin{proposition}
\label{progongoing}
Assume all $y \sim p(y)$ lie in $[y_{low},y_{hi}]$. After $k > \big( \frac{y_{hi}-y_{low}}{-|\mu|+\sqrt{\mu^2+d\sigma^2}} \big)^2 \log[\frac{\sqrt{2}}{\delta^{1/2}}]$ iterations, with probability, $1-\delta$, the ($w_1,b$)-subsystem can be up-weighted and the ($w_2,a$)-subsystem ``turned-on'', resulting in a strictly-monotone LQ-GAN.
\end{proposition}

\begin{proof}
As before, assume we are running $F^{w_1,b}_{cc}$ on the ($w_1,b$)-subsystem and $F^{w_2,a}_{cc'}$ on the ($w_2,a$)-subsystem. Also, multiply $F^{w_1,b}_{cc}$ by $e>0$, i.e., increase the learning rate by $e$ or divide the learning rate of $F^{w_2,a}_{cc'}$ by $e$. The full symmetrized Jacobian of this system is:
\begin{align}
    J_{cc'} &= \begin{bmatrix}
        1 & 0 & 0 & 0 \\
        0 & e & 0 & 0 \\
        0 & 0 & \frac{a^2-b^2+\mu^2+\sigma^2}{2a^2} & \frac{b}{2a} \\
        0 & 0 & \frac{b}{2a} & e
    \end{bmatrix} = 
    \begin{bmatrix}
        1 & 0 & 0 & 0 \\
        0 & e & 0 & 0 \\
        0 & 0 & \frac{G}{2a^2} + \frac{1}{2} & \frac{b}{2a} \\
        0 & 0 & \frac{b}{2a} & e
    \end{bmatrix}
\end{align}

The upper left $2 \times 2$ block of this matrix is positive definite. In order to show the whole matrix is positive definite, it suffices to prove the lower right block is positive definite. The trace and determinant of that block are
\begin{align}
    Tr_{ab} &= 1/2 + e + \frac{G}{2 a^2} \\
    Det_{ab} &= \frac{2 e (a^2 + G) - b^2}{4 a^2}.
\end{align}

where $G=\mu^2+\sigma^2-b^2$ as before. We need $G \ge 0$ for $Tr_{ab}>0$ (for $\lim_{a\rightarrow0+}$) and $2eG \ge b^2$ for $Det > 0$. As before, Hoeffding's inequality says $k$ iterations are required for an accurate estimate of the mean (see Equation~(\ref{hoeff})). And as before, we find that $G=(1-d)\sigma^2$. We will focus on the determinant condition here. Let
\begin{align}
    G &= (1-d)\sigma^2 \ge \frac{b^2}{2e} \label{Gdef} \\
    \Rightarrow e &\ge \frac{b^2}{2(1-d)\sigma^2} \\
    \Rightarrow e &\ge \frac{\mu^2 + d\sigma^2}{2(1-d)\sigma^2} \\
    \text{or } \Rightarrow d &\le 1 - \frac{b^2}{2e\sigma^2} \\
    \Rightarrow d &\le 1 - \frac{\mu^2 + d\sigma^2}{2e\sigma^2}.
\end{align}

More simply, let $d=1/2$. Then set $e > \frac{\mu_{\max}^2}{\sigma_{\min}^2} + \frac{1}{2}$. This ensures the trace and determinant are both strictly positive which implies that the resulting system is at least strictly monotone.

We can show that this system is not strongly-monotone by upper bounding the minimum eigenvalue. To ease the analysis, let $H=2eG-b^2$ and note that $H < 2e\sigma^2$ (see Equation~(\ref{Gdef})), i.e., $H$ is finite. This allows us to upper bound the determinant, in turn, upper bounding the minimum eigenvalue. The determinant simplifies to
\begin{align}
    Det_{ab} &= \frac{e}{2} + \frac{H}{4a^2}.
\end{align}
The minimum eigenvalue is upper bounded as follows:
\begin{align}
    \lambda_{min} &= \frac{1}{2} \Big( Tr - \sqrt{Tr^2-4Det} \Big) \\
    &= \frac{1}{2} \Big( 1/2 + e + \frac{G}{2a^2} - \sqrt{(1/2 + e + \frac{G}{2a^2})^2 - 2e - \frac{H}{a^2}} \Big) \\
    \lim_{a \rightarrow 0^+} \lambda_{min} &= \frac{1}{2} \Big( 1/2 + e + \frac{G}{2a^2} - \sqrt{(1/2 + e + \frac{G}{2a^2})^2} \Big) = 0
\end{align}



As the system continues learning a more accurate mean (iterations, $k$, is increasing), $d$ is effectively decreasing towards zero. In the limit $\lim_{d\rightarrow0+} e \ge \frac{\mu^2}{2\sigma^2}$.

Given, $[y_{low},y_{hi}]$, we can set $\mu_{max}=\max(|y_{low}|,|y_{hi}|)$. Also, note that if the distribution is known to support $\epsilon$ balls at the ends of the specified interval, $[y_{low},y_{hi}]$, with some nonzero probabilities, $P_{low}$ and $P_{hi}$, then we can lower bound the variance as well. Specifically, let $P_{low}=\frac{\epsilon}{2}\big( p(y_{low}) + p(y_{low}+\epsilon)\big)$ and $P_{hi}=\frac{\epsilon}{2}\big( p(y_{hi}) + p(y_{hi}-\epsilon)\big)$. Then
\begin{align}
    \sigma^2 = \mathbb{E}[ (y - \mu)^2] &= \int_{y_{low}}^{y_{hi}} p(y) (y - \mu)^2 dy \\
    &\ge \int_{y_{low}}^{y_{low}+\epsilon} p(y) (y - \mu)^2 dy + \int_{y_{hi}+\epsilon}^{y_{hi}} p(y) (y - \mu)^2 dy \\
    &= \frac{\epsilon}{2}\big( p(y_{low}) + p(y_{low}+\epsilon)\big) (y_{low}-\mu)^2 \\ &+ \frac{\epsilon}{2}\big( p(y_{hi}) + p(y_{hi}-\epsilon)\big) (y_{hi}-\mu)^2 + \mathcal{O}(\epsilon^2) \\
    &\approx P_{low} (y_{low} - \mu)^2 + P_{hi} (y_{hi} - \mu)^2 \\
    &\ge P_{low} P_{hi} (y_{hi}-y_{low})^2 = \sigma_{\min}^2.
\end{align}
\end{proof}

\subsection{Analysis of the ($W_2,A$)-Subsystem for the N-d LQ-GAN}

Let $A$ be a lower triangular matrix with positive diagonal\textemdash $A$ represents the generator's guess at the square root of $\Sigma$.

\begin{proposition}
\label{multivarnotmonotone}
The 2-d LQ-GAN is not quasimonotone for $F_{cc}$ or $F_{eg}$ with or without scaling.
\end{proposition}

\begin{proof}
We will show that this system fails condition~\ref{cond:A}. Please refer to the Mathematica notebook for our derivations of these results.

Define the following skew symmetric matrix.
\begin{align}
    K &= \begin{bmatrix}
        0 & 1 & 0 & 0 & 0 & 0 \\
        -1 & 0 & 0 & 0 & 0 & 0 \\
        0 & 0 & 0 & 1 & 0 & 0 \\
        0 & 0 & -1 & 0 & 0 & 0 \\
        0 & 0 & 0 & 0 & 0 & 1 \\
        0 & 0 & 0 & 0 & -1 & 0
    \end{bmatrix} 
\end{align}

Let $v_{cc}=KF_{cc}$ and $v_{eg}=KF_{eg}$. Similarly, with scaling, let $v_{cc'}=KF_{cc'}$ and $v_{eg'}=KF_{eg'}$. Let

\begin{align}
    \Sigma &= \begin{bmatrix}
        1 & 1 \\
        1 & 100
    \end{bmatrix}
\end{align}

\begin{align}
    x &= \begin{bmatrix}
        W11 \\
        W12 \\
        W22 \\
        A11 \\
        A22 \\
        A21
    \end{bmatrix} =
    \begin{bmatrix}
        0 \\
        0 \\
        0 \\
        1 \\
        0.1 \\
        0.1
    \end{bmatrix}
\end{align}

Then 

\begin{align}
    v_{cc}^\top  J_{cc}(x) v_{cc}^\top  \Big\vert_{x} &= -189684 < 0 \\
    v_{eg}^\top  J_{eg}(x) v_{eg}^\top  \Big\vert_{x} &= -189684 < 0 \\
    v_{cc'}^\top  J_{cc'}(x) v_{cc'}^\top  \Big\vert_{x} &= -2.95426 \cdot 10^9 < 0 \\
    v_{eg'}^\top  J_{eg'}(x) v_{eg'}^\top  \Big\vert_{x} &= -2.95426 \cdot 10^9 < 0
\end{align}

This implies that neither system is quasimonotone (with, $cc'/eg'$, or without, $cc/eg$, scaling).
\end{proof}

\begin{proposition}
The 2-d LQ-GAN with $W_{11}$ and $A_{11}$ already learned, i.e., $W_{11}=0$ and $A_{11}=A_{11}^*$, is not quasimonotone for $F_{cc}$ or $F_{eg}$.
\end{proposition}

\begin{proof}
We will show that this system fails condition~\ref{cond:A}. Please refer to the Mathematica notebook for our derivations of these results.

Define the following skew symmetric matrix.
\begin{align}
    K &= \begin{bmatrix}
        0 & 1 & 0 & 0 \\
        -1 & 0 & 0 & 0 \\
        0 & 0 & 0 & 1 \\
        0 & 0 & -1 & 0
    \end{bmatrix} 
\end{align}

Let $v_{cc}=KF_{cc}$ and $v_{eg}=KF_{eg}$. Let

\begin{align}
    \Sigma &= \begin{bmatrix}
        1 & 1 \\
        1 & 100
    \end{bmatrix}
\end{align}

\begin{align}
    x &= \begin{bmatrix}
        W12 \\
        W22 \\
        A22 \\
        A21
    \end{bmatrix} =
    \begin{bmatrix}
        0 \\
        0 \\
        0.1 \\
        0.1
    \end{bmatrix}
\end{align}

Then 

\begin{align}
    v_{cc}^\top  J_{cc}(x) v_{cc}^\top  \Big\vert_{x} &= -189684 < 0 \\
    v_{eg}^\top  J_{eg}(x) v_{eg}^\top  \Big\vert_{x} &= -189684 < 0
\end{align}

This implies that neither system is quasimonotone.
\end{proof}

\begin{proposition}
The 3-d LQ-GAN with the diagonal of $A$ already learned, i.e., $A_{ii}=A_{ii}^*$, is not quasimonotone for $F_{cc}$ or $F_{eg}$ with or without scaling.
\end{proposition}

\begin{proof}
We will show that this system fails condition~\ref{cond:A}. Please refer to the Mathematica notebook for our derivations of these results.

Define the following skew symmetric matrix.
\begin{align}
    K &= \begin{bmatrix}
        0 & 1 & 0 & 0 & 0 & 0 \\
        -1 & 0 & 0 & 0 & 0 & 0 \\
        0 & 0 & 0 & 1 & 0 & 0 \\
        0 & 0 & -1 & 0 & 0 & 0 \\
        0 & 0 & 0 & 0 & 0 & 1 \\
        0 & 0 & 0 & 0 & -1 & 0
    \end{bmatrix}
\end{align}

Let $v_{cc}=KF_{cc}$ and $v_{eg}=KF_{eg}$. Let

\begin{align}
    \Sigma &= \begin{bmatrix}
        0.2 & 0.15 & 0.5 \\
        0.15 & 0.9 & 0.8 \\
        0.5 & 0.8 & 2
    \end{bmatrix}
\end{align}

\begin{align}
    x &= \begin{bmatrix}
        W12 \\
        W13 \\
        W23 \\
        A21 \\
        A31 \\
        A32
    \end{bmatrix} =
    \begin{bmatrix}
        10 \\
        10 \\
        10 \\
        0.1 \\
        0.2 \\
        -0.5
    \end{bmatrix}
\end{align}

Then 

\begin{align}
    v_{cc}^\top  J_{cc}(x) v_{cc}^\top  \Big\vert_{x} &= -1024.26 < 0 \\
    v_{eg}^\top  J_{eg}(x) v_{eg}^\top  \Big\vert_{x} &= -242766 < 0
\end{align}

This implies that neither system is quasimonotone.
\end{proof}

\begin{proposition}
The N-d LQ-GAN with all but a single row of $A$ fixed is strictly-monotone for $F_{cc}$, $F_{eg}$, and $F_{con}$.
\end{proposition}
\label{prop:covar}

\begin{proof}
First, note that the Cholesky decomposition of $\Sigma$, denoted by $A^*$, obeys the follow equation:
\begin{align}
    0 &= \Sigma_{ij} - \sum_{d=1}^{i} A^*_{id} A^*_{jd} \label{chol_sat}
\end{align}
where $i < j$. $\Sigma$ is symmetric, so $\Sigma_{ji}$ can be recovered as $\Sigma_{ij}$. This allows us to remove 1 degree of freedom from the system by defining the diagonal term in a single row of $A$ in terms of the other entries in the row:
\begin{align}
    A_{ii} &= \sqrt{\Sigma_{ii} - \sum_{d=1}^{i-1} A^{2}_{id}}
\end{align}
where as before $A_{ii}$ must be greater than zero. We assume that $\Sigma_{ii}$ has already been learned by \emph{Crossing-the-Curl} as described in the main body. The condition $A_{ii} > 0$ can be ensured by constraining $\sum_{d=1}^{i-1} A^{2}_{id} \le \Sigma_{ii} - \epsilon$ with $\epsilon \ll 1$\textemdash this can be achieved with a simple ball projection.

Consider learning a single row of $A$, specifically $A_{Ni}$ with $i < N$; $A_{NN}$ is recovered as discussed above and $A_{N,i>N} = 0$ by definition of the Cholesky decomposition. We will also set all $W_{2ij}=W_{2ji}$ equal to zero except where $i$ xor $j$ equals $N$. This has the effect of fixing parts of the system irrelevant for solving the $N$th row of $A$. For ease of exposition, we will drop the ``2'' subscript of $W_2$ in what follows.

We will begin by writing down the map for the entire system and then simplifying using the constraints and assumptions discussed above:

\begin{align}
F_{W_2} &= AA^T - \Sigma \\
&= \begin{bmatrix}
A_{11}^2 & A_{11}A_{21} & A_{11}A_{31} & \cdots  \\
A_{11}A_{21} & A_{21}^2 + A_{22}^2 & A_{21}A_{31} + A_{22}A_{32} & \cdots  \\
A_{11}A_{31} & A_{21}A_{31} + A_{22}A_{32} & A_{31}^2 + A_{32}^2 + A_{33}^2 & \cdots \\
\vdots & \vdots & \vdots & \ddots
\end{bmatrix} - \begin{bmatrix}
S_{11} & S_{12} & S_{13} & \cdots \\
S_{12} & S_{22} & S_{23} & \cdots \\
S_{13} & S_{23} & S_{33} & \cdots \\
\vdots & \vdots & \vdots & \ddots
\end{bmatrix} \\
F_{A} &= -2W_2A \\
&= -2\begin{bmatrix}
A_{11} W_{11} + A_{21} W_{12} + A_{31} W_{13} + \cdots & A_{22} W_{12} + A_{32} W_{13} + \cdots & A_{33} W_{13} + \cdots & \cdots  \\
A_{11} W_{12} + A_{21} W_{22} + A_{31} W_{23} + \cdots & A_{22} W_{22} + A_{32} W_{23} + \cdots & A_{33} W_{23} + \cdots & \cdots  \\
A_{11} W_{13} + A_{21} W_{23} + A_{31} W_{33} + \cdots & A_{22} W_{23} + A_{32} W_{33} + \cdots & A_{33} W_{33} + \cdots & \cdots  \\
\vdots & \vdots & \vdots & \ddots
\end{bmatrix}.
\end{align}

We are only interested in learning the $N$th row of $A$. Take $N=3$ for example. Notice that the $3$rd row of $A$, $A_{3:},$ only contains the following $W_2$ terms: $W_{13}, W_{23}$. The rest are set to zero as mentioned earlier. The reason for this will become apparent soon. We fix all other entries to zero to highlight the relevant subsystem below:

\begin{align}
F_{W_2} &= AA^T - \Sigma \\
&= \begin{bmatrix}
0 & 0 & A_{11}A_{31} - S_{13} & \cdots \\
0 & 0 & A_{21}A_{31} + A_{22}A_{32} - S_{23} & \cdots \\
A_{11}A_{31} - S_{13} & A_{21}A_{31} + A_{22}A_{32} - S_{23} & 0 & \cdots \\
\vdots & \vdots & \vdots & \ddots
\end{bmatrix} \\
F_{W_{i<N}} &= 2 \big( \sum_{d \le i} A_{id} A_{Nd} - S_{iN} \big) \\
F_{A} &= -2W_2A \\
&= -2\begin{bmatrix}
0 & 0 & 0 & \cdots  \\
0 & 0 & 0 & \cdots  \\
A_{11} W_{13} + A_{21} W_{23} & A_{22} W_{23} & 0 & \cdots  \\
\vdots & \vdots & \vdots & \ddots
\end{bmatrix} \\
F_{A_{N>i}} &= -2\big( \sum_{d < N} A_{di} W_{dN} \big).
\end{align}

Notice that the map $F_{W_2}$ is zero only if Equation~(\ref{chol_sat}) is satisfied for $\Sigma_{iN}$ and $W_{dN} = 0$ for all $d<N$. Therefore, setting all other entries of $W_2$ as prescribed simplified the system, while maintaining the correct fixed point.

In order to determine the monotonicity of this system, we need to compute the Jacobian of $F = [F_{W_2};F_{A}]$:

\begin{align}
    J &= \begin{bmatrix}
    \frac{\partial F_{W_{i<d}}}{\partial W_{k<d}} & \frac{\partial F_{W_{i<d}}}{\partial A_{d>k}} \\
    \frac{\partial F_{A_{d>i}}}{\partial W_{k<d}} & \frac{\partial F_{A_{d>i}}}{\partial A_{d>k}}
    \end{bmatrix} \\
    &= -2\begin{bmatrix}
    (d-1) \times 0 & -A_{i \ge k} \\
    A_{k \ge i} & (d-1) \times 0
    \end{bmatrix} \\
    &= -2\begin{bmatrix}
    0 & 0 & -A_{11} & 0 \\
    0 & 0 & -A_{21} & -A_{22} \\
    A_{11} & A_{21} & 0 & 0 \\
    0 & A_{22} & 0 & 0
    \end{bmatrix} \text{ for $N=3$} \\
    &= -2\begin{bmatrix}
    0 & -A_{:d-1} \\
    A_{:d-1}^\top & 0
    \end{bmatrix}
\end{align}
which is skew-symmetric and constant with respect to the variables being learned: $W_{2,i<N}$ and $A_{N>i}$. Therefore, $J+J^\top=0$ is PSD, which implies $F$ is monotone. The fact that $J$ is constant along with Proposition~\ref{affinemon} imply that $F_{cc}=F_{eg}=F_{con}=-JF$ are also monotone:

\begin{align}
    F_{cc} = F_{eg} = F_{con} &= 2\begin{bmatrix}
    -A_{:d-1} F_{A_{d>i}} \\
    A_{:d-1}^\top F_{W_{i<d}}
    \end{bmatrix}.
\end{align}

Note that the component of $F_{cc}$ corresponding to the dynamics of $A$,  is independent of $W_2$. This means the dynamics are now decoupled from $W_2$ and can be run separately. By inspecting the symmetrized Jacobian of $F_{cc}$ we can show that it is a block matrix composed of positive definite matrices:
\begin{align}
    J_{sym} &= \frac{1}{4}(J-J^\top)^\top (J-J^\top) \\
    &= J^\top J = -JJ \\
    &= \begin{bmatrix}
    A_{:d-1}A_{:d-1}^\top & 0 \\
    0 & A_{:d-1}^\top A_{:d-1}
    \end{bmatrix}.
\end{align}
$A_{:d-1}A_{:d-1}^\top$ is positive definite because $A$ is constrained to be of Cholesky form. Moreover, the eigenvalues of $A_{:d-1}^\top A_{:d-1}$ are the same as $A_{:d-1}A_{:d-1}^\top$, therefore both blocks are positive definite. This implies the entire matrix $J_{sym}$ is positive definite which means $F_{cc}=F_{eg}=F_{con}$ are strictly monotone. Note that we do not require $A_{:d-1}=A^*_{:d-1}$ for strict monotonicity. In practice, the system will actually be both strongly-monotone and smooth. This is because $A$ is constrained with a projection onto a ball and the diagonal of $A$ is restricted to be larger than $\epsilon$. These two conditions guarantee a nonzero, finite minimum and maximum value for the eigenvalues of $A_{:d-1}A_{:d-1}^\top$\textemdash the minimum corresponds to strong-monotonicity and the maximum corresponds to smoothness.
\end{proof}

Unlike the ($w_2,a$)-subsystem where monotonicity depends on the accuracy of the learned mean, this system is monotone as long as $A_{:d-1}$ is PSD which is guaranteed from the form we have prescribed to $A$. This result suggests learning the rows of $A$ in succession, and each subsystem is guaranteed to be strictly monotone. Note that the variance, i.e., diagonal of $\Sigma$, will be slightly off the true value if the mean, $\mu$, is not first learned perfectly. The learned $A$ will then be slightly off the true $A^*$ and errors will compound, but still not affect monotonicity. The subsystems corresponding to each row of $A$ can be revisited to learn the entries of $A$ more accurately. Permuting the dimensions of $x$ such that the dimensions corresponding to highest variance are learned first may ensure subsystems with maximal \emph{strong}-monotonicity. We leave a detailed examination to future research.


\subsection{An $\mathcal{O}(N/k)$ Algorithm for LQ-GAN}
\label{sub:cc_lqgan}
Here we present pseudocode for solving the stochastic LQ-GAN. The maps corresponding to learning the mean and variance by \emph{Crossing-the-Curl} are both strongly convex and can therefore be solved with a simple projected gradient method. We argued in the previous subsection that the map associated with learning the covariance terms is strongly-monotone and smooth, not only strictly monotone. In practice, we found that a projected Extragradient algorithm~\cite{kannan2017pseudomonotone} gave better results. The full procedure is outlined in Algorithm~\ref{alg:cc_lqgan}. Replace sample estimates with the true $\mu$ and $\Sigma$ for the deterministic LQ-GAN.
\begin{algorithm}[H]
    \caption{\emph{Crossing-the-Curl} for LQ-GAN}
    \label{alg:cc_lqgan}
    \begin{algorithmic}
        \STATE Input: Sampling distribution $p(y)$, max iterations $K$, batch size $B$, lower bound on variance $\sigma_{\min}$
        \STATE \textbf{(1)} Learn Mean
        \STATE $\mu_0 = [0,\ldots,0]^\top$
        \FORALL{$k=1,2,\ldots,K$}
            \STATE $\hat{\mu} = \frac{1}{B} \sum_{s=1}^B (y_s \sim p(y))$
            \STATE $\mu_k = \frac{k}{k+1} \mu_{k-1} + \frac{1}{k+1} \hat{\mu}$, \hspace{0.3cm} i.e., $\mu_k = \mu_{k-1} - \rho_k F^b_{cc}$ with step size $\rho_k = \frac{1}{k+1}$
        \ENDFOR
        \STATE \textbf{(2)} Learn Variance
        \STATE $\sigma_0 = [1,\ldots,1]^\top$
        \FORALL{$k=1,2,\ldots,K$}
            \STATE $\hat{\sigma}^2 = \frac{1}{B-1} \sum_{s=1}^B [(y_s \sim p(y)) - \mu_K]^2$
            \STATE $F^a_{cc'} = (\sigma_k^2 - \hat{\sigma}^2)/(2 \sigma_k)$
            \STATE $\sigma_k = \texttt{clip}(\sigma_{k-1} - \frac{1}{k+1} F^a_{cc'}, \sigma_{\min}, \infty)$
        \ENDFOR
        \STATE \textbf{(3)} Learn Covariance
        \STATE $A_0 = LT(I_N)$, i.e., lower triangular part of Identity matrix
        \STATE $A_{0,11} = \sigma_{K,1}$
        \FORALL{$d=2,\ldots,N$}
            \FORALL{$k=1,2,\ldots,K$}
                \STATE $y_s \sim p(y), \,\, s=1,\ldots,B$
                \STATE $\hat{\Sigma} = \frac{1}{B-1} \sum_{s=1}^B (y_s - \mu_K)^\top (y_s - \mu_K)$
                \STATE $F_{W_{i<d}} = 2 \big( \sum_{j \le i} A_{k-1,ij} A_{k-1,dj} - \hat{\Sigma}_{id} \big)$
                \STATE $F^A_{cc} = A^\top_{k-1,:d-1} F_{W_{i<d}}$ where $A_{k-1,:d-1}$ refers to the top left $d-1 \times d-1$ block of $A_{k-1}$
                \STATE $\hat{A}_{k,d:} = A_{k-1,d:} - \frac{1}{k+1} F^A_{cc}$ where $A_{k-1,d:}$ refers to the $d$th row of $A_k$ excluding the diagonal
                \IF{$\sum_j \hat{A}_{k,dj}^2 > \sigma_{K,d}^2 - \sigma_{\min}^2$}
                \STATE $\hat{A}_{k,dj} = \hat{A}_{k,dj} \cdot \sigma_{K,d} / \sqrt{\sum_j \hat{A}_{k,dj}^2 + \sigma_{\min}^2}$
                \ENDIF
                \STATE $F_{W_{i<d}} = 2 \big( \sum_{j \le i} A_{k-1,ij} \hat{A}_{k,dj} - \hat{\Sigma}_{id} \big)$
                \STATE $F^A_{cc} = A^\top_{k-1,:d-1} F_{W_{i<d}}$ where $A_{k-1,:d-1}$ refers to the top left $d-1 \times d-1$ block of $A_{k-1}$
                \STATE $A_{k,d:} = A_{k-1,d:} - \frac{1}{k+1} F^A_{cc}$ where $A_{k-1,d:}$ refers to the $d$th row of $A_k$ excluding the diagonal
                \IF{$\sum_j A_{k,dj}^2 > \sigma_{K,d}^2 - \sigma_{\min}^2$}
                \STATE $A_{k,dj} = A_{k,dj} \cdot \sigma_{K,d} / \sqrt{\sum_j A_{k,dj}^2 + \sigma_{\min}^2}$
                \ENDIF
            \ENDFOR
            \STATE $A_{K,dd} = \sqrt{\sigma_{K,d}^2 - \sum_j A_{K,dj}^2}$
        \ENDFOR
    \end{algorithmic}
\end{algorithm}

\subsubsection{Convergence Rate}
As mentioned above, the maps for learning the mean and variance are both strongly convex which implies a $\mathcal{O}(1/k)$ stochastic convergence rate for each, the sum of which is still $\mathcal{O}(1/k)$.

In practice, the maps for learning each row of $A$ are strongly-monotone and smooth (see last paragraph of proof of Proposition~\ref{prop:covar}) which implies a $\mathcal{O}(1/k)$ stochastic convergence rate for each as well. Because this technique consists of $N+1$ steps for learning the full $N$-d LQ-GAN, it requires $\hat{k}=Nk$ iterations which, in total, implies a $\mathcal{O}(N/k)$ stochastic convergence rate.

Hidden within this analysis is the fact that each iteration of learning the mean and variance is $\mathcal{O}(N)$ in terms of time-complexity and each iteration for learning each row of $A$ is $\mathcal{O}(N^2)$, therefore this entire procedure is $\mathcal{O}(N^3/k)$ in terms of FLOPS. This is expected as the complexity of a Cholesky decomposition to compute $A=\Sigma^{1/2}$ is also $\mathcal{O}(N^3)$. Note that unlike the complexity of computing $F$ each iteration which can be mitigated with parallel computation, the sequential nature of the stagewise procedure cannot be amortized which is why we report a $\mathcal{O}(N/k)$ convergence rate and not $\mathcal{O}(1/k)$.

Another subtle point is that the LQ-GAN is locally monotone about the equilibrium. Recall from Theorem D.1 on p.26 in~\cite{nagarajan2017gradient} that the Jacobian at the equilibrium is of the following form (remember our definition for the Jacobian is the negative of theirs):
\begin{align}
    J &= \begin{bmatrix}
    J_{DD} & J_{DG} \\
    -J_{DG}^\top & 0
    \end{bmatrix}
\end{align}
where $J_{DD}$ is positive definite. The symmetrized Jacobian is then
\begin{align}
    J_{sym} = \frac{1}{2}(J+J^\top) &= \begin{bmatrix}
    J_{DD} & 0 \\
    0 & 0
    \end{bmatrix} \succeq 0.
\end{align}
This implies $F$ is monotone where $F = [\nabla V_{A,b}; -\nabla V_{W_2,w_1}]$. Therefore, we can use stagewise procedure in Algorithm~\ref{alg:cc_lqgan} to converge to a local neighborhood about the equilibrium, constrain the system to this neighborhood with a projection (which will guarantee smoothness of the map), and then continue with an extragradient method applied to the full system. The local convergence rate will still be $\mathcal{O}(1/k)$ with $\mathcal{O}(N^3)$ iteration complexity due to the matrix multiplications required in computing $F$ (see Proposition~\ref{multivarsoln}).

\subsection{Deep Learning Specifications and Results}
\label{deepnets}
We also experimented on common neural-net driven tasks. We tested $F_{lin}$ with $(\alpha,\beta,\gamma)=(1,10,10^{-4})$ on a mixture of Gaussians and $(\alpha,\beta,\gamma)=(1,10,0.1)$ on CIFAR10 against $F_{con}$, i.e., $(\alpha,\beta,\gamma)=(1,10,0)$. Introducing a small $-JF$ term can help accelerate training (see Figure~\ref{fig:MO8G+CIFAR10}).
\begin{figure}[htbp]
    \begin{minipage}{0.42\textwidth}
        \centering
        \includegraphics[scale=0.25]{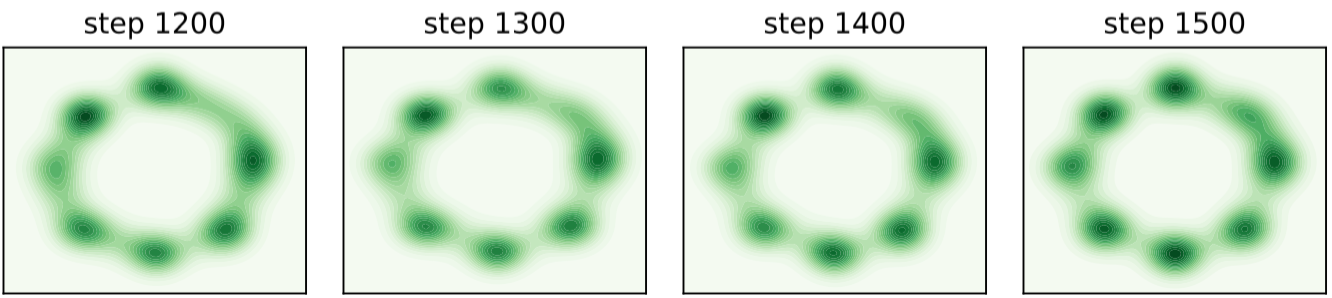}\\
        \includegraphics[scale=0.25]{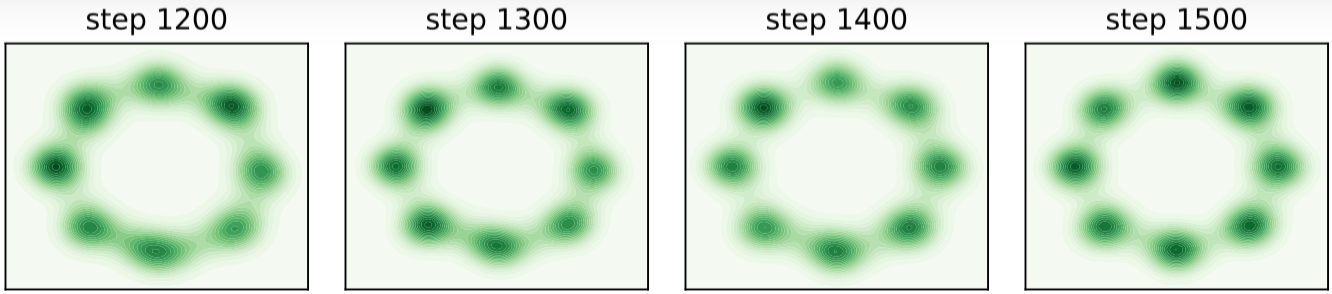}
    \end{minipage}
    \begin{minipage}{0.45\textwidth}
        \centering
        \includegraphics[scale=0.11]{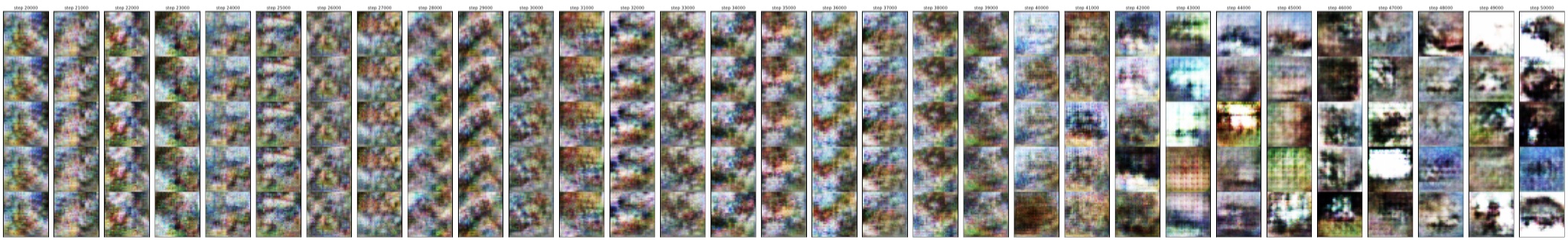}\\
        \includegraphics[scale=0.11]{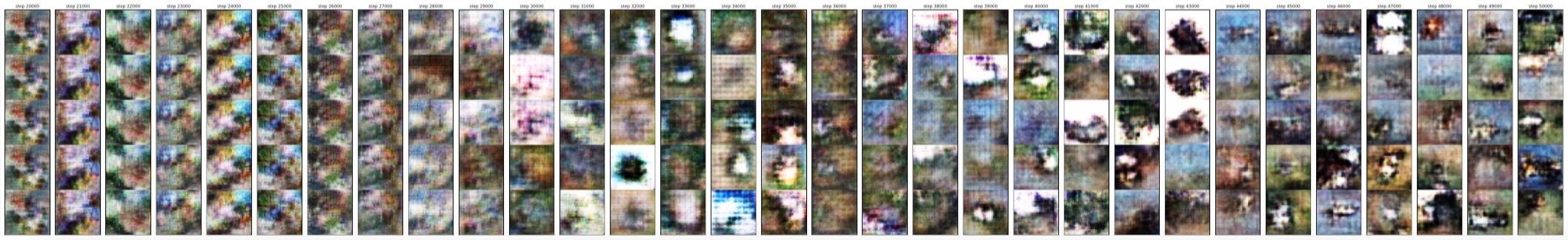}
    \end{minipage}
    \caption{$F_{con}$ (top) vs $F_{lin}$ (bottom) on a mixture of Gaussians (left) and CIFAR10 (right). Each column of images corresponds to an epoch with epochs increasing left to right.}
    \label{fig:MO8G+CIFAR10}
\end{figure}

\subsubsection{Images at End of Training for Mixture of Gaussians}
See Figure~\ref{fig:MO8Gend}.
\begin{figure}[htbp]
    \centering
    \begin{minipage}{0.42\textwidth}
        \centering
        \includegraphics[scale=0.125]{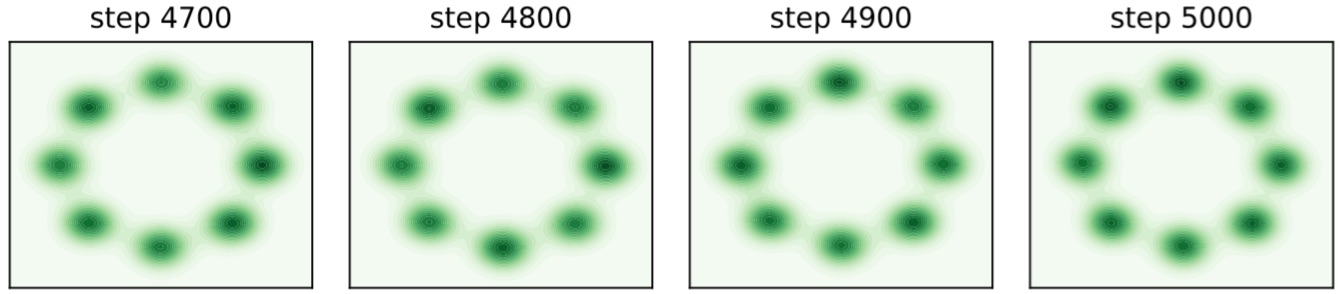}\\
        \includegraphics[scale=0.125]{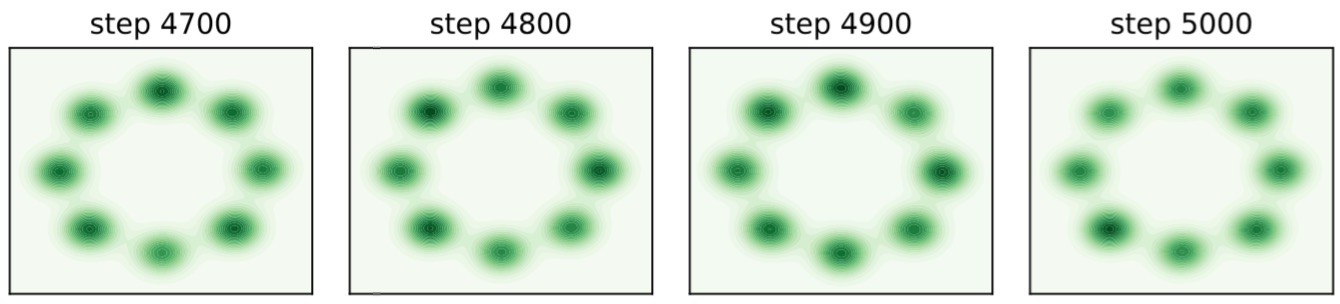}
    \end{minipage}
    \begin{minipage}{0.225\textwidth}
        \centering
        \includegraphics[scale=0.25]{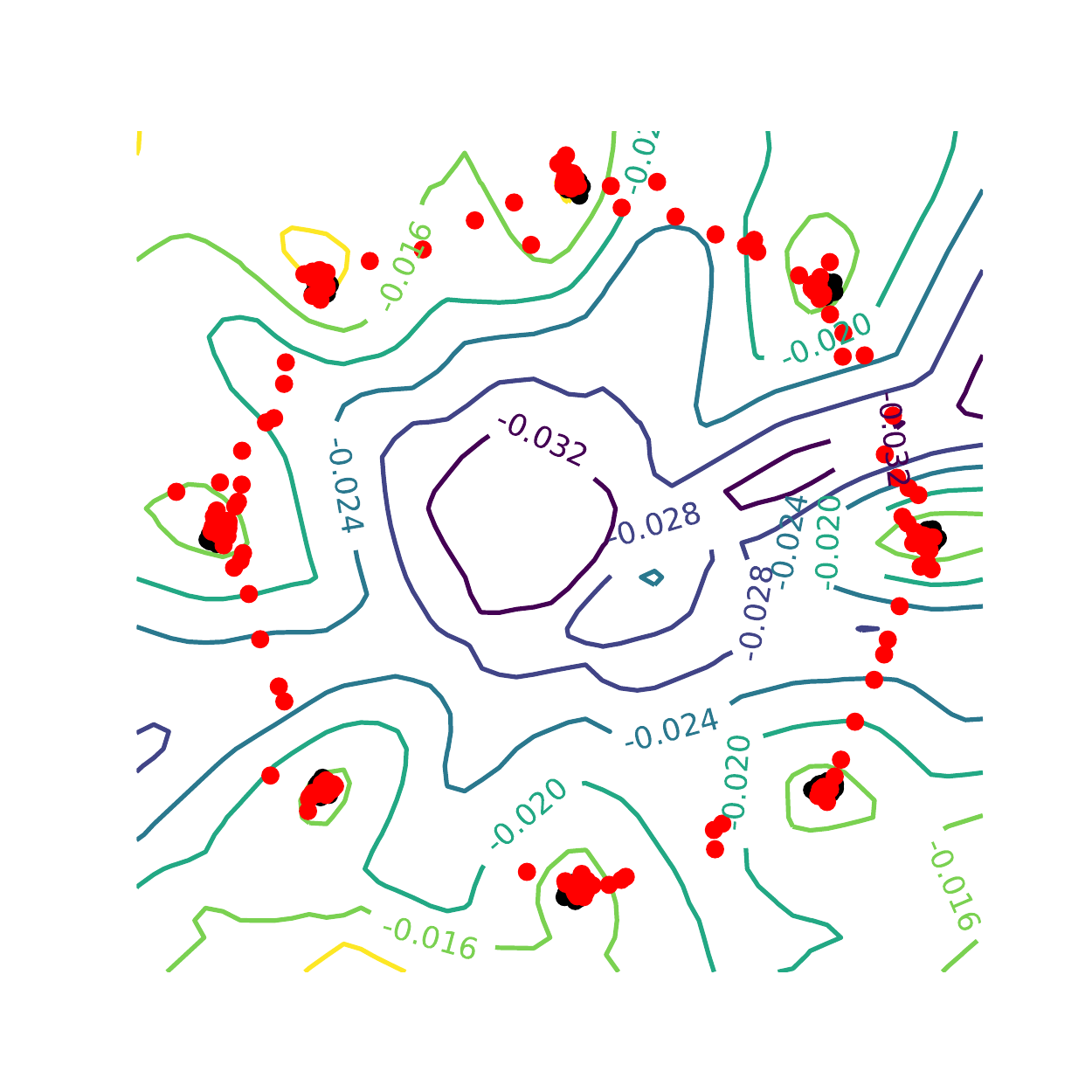}
    \end{minipage}
    \begin{minipage}{0.225\textwidth}
        \centering
        \includegraphics[scale=0.25]{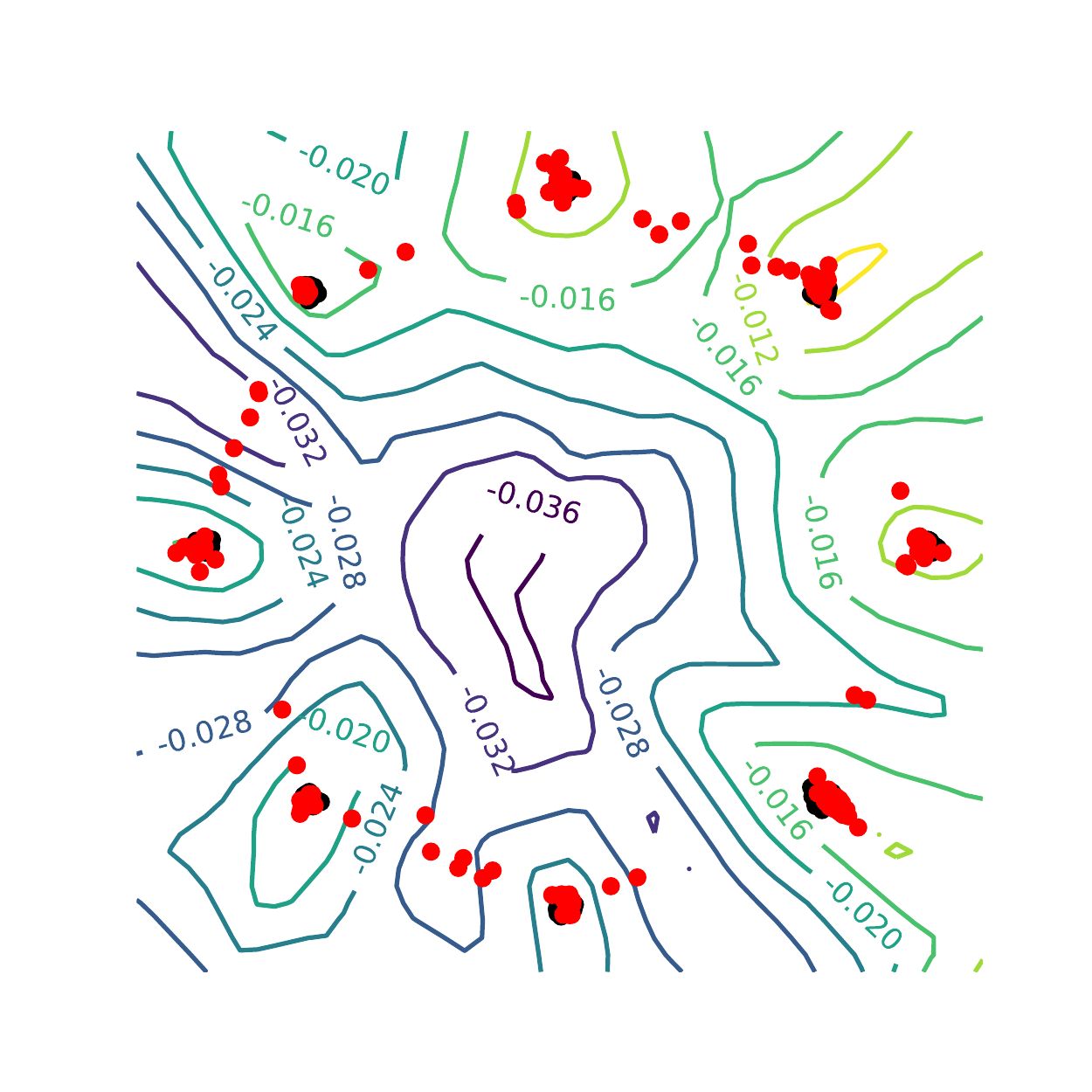}
    \end{minipage}
    \caption{$F_{con}$ (top row) vs $F_{lin}$ (bottom row) on a mixture of Gaussians. Contour plots of discriminator along with samples in red shown for $F_{con}$ (left) and $F_{lin}$ (right).}
    \label{fig:MO8Gend}
\end{figure}

\subsubsection{Mixture of Gaussians Network Architectures}

Both the generator and discriminator are fully connected neural networks. The relevant hyperparameters for setting up the GAN are itemized below.
\begin{itemize}
    \item batch size 512
    \item divergence Wasserstein
    \item disc optim Adam
    \item disc learning rate 0.001
    \item disc n hidden 16
    \item disc n layer 4
    \item disc nonlinearity ReLU
    \item gen optim Adam
    \item gen learning rate 0.001
    \item gen n hidden 16
    \item gen n layer 4
    \item gen nonlinearity ReLU
    \item betas [0.5, 0.999]
    \item epsilon 1e-08
    \item max iter 5001
    \item z dim 16
    \item x dim 2
\end{itemize}

$F_{con}$ was used with $\beta=1.0$ and $F_{lin}$ was used with $(\alpha,\beta,\gamma)=(1.0,1.0,0.001)$.

\subsubsection{Images at End of Training for CIFAR10}
See Figure~\ref{fig:CIFAR10end}.
\begin{figure}[htbp]
    \centering
    \begin{minipage}{0.45\textwidth}
        \centering
        \includegraphics[scale=0.12]{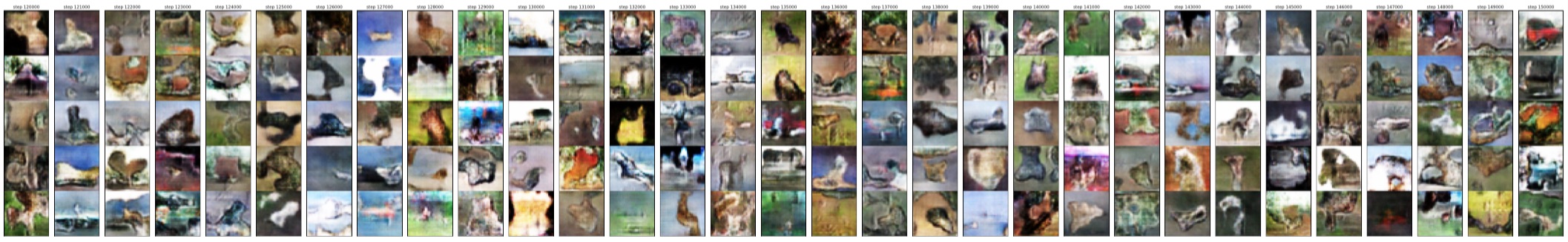}\\
        \includegraphics[scale=0.12]{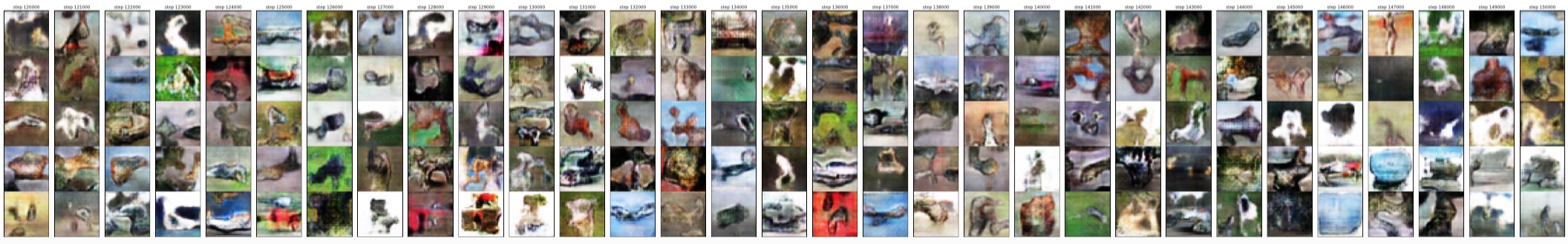}
    \end{minipage}
    \begin{minipage}{0.25\textwidth}
        \centering
        \includegraphics[scale=0.27]{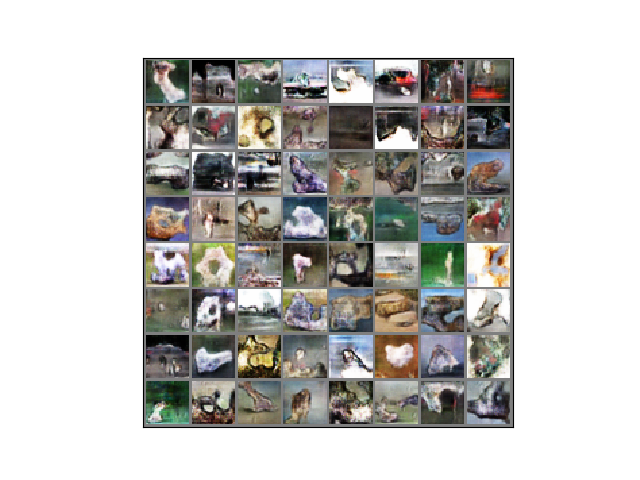}
    \end{minipage}
    \begin{minipage}{0.25\textwidth}
        \centering
        \includegraphics[scale=0.27]{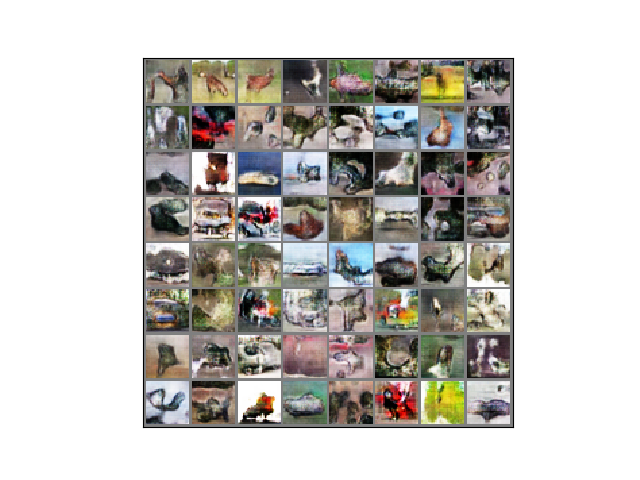}
    \end{minipage}
    \caption{$F_{con}$ (top row) vs $F_{lin}$ (bottom row) on CIFAR10. Images generated at final iteration shown for $F_{con}$ (left) and $F_{lin}$ (right).}
    \label{fig:CIFAR10end}
\end{figure}

\subsubsection{CIFAR10 Network Architectures}

Both the generator and discriminator are convolutional neural networks; we copied the architectures used in~\cite{mescheder2017numerics}. The generator consists of a linear layer, followed by 4 deconvolution layers ($5 \times 5$ kernel, $2 \times 2$ stride, leaky ReLU, 64 hidden channels), followed by a final linear layer with a tanh nonlinearity. The discriminator consists of 4 convolution layers ($5 \times 5$ kernel, $2 \times 2$ stride, leaky ReLU, 64 hidden channels) followed by a linear layer. The relevant hyperparameters for setting up the GAN are itemized below.
\begin{itemize}
    \item batch size 64
    \item divergence JS
    \item disc optim RMSprop
    \item disc learning rate 0.0001
    \item gen optim RMSprop
    \item gen learning rate 0.0001
    \item betas [0.5, 0.999]
    \item epsilon 1e-08
    \item max iter 150001
    \item z dim 256
    \item x dim 1024
\end{itemize}

$F_{con}$ was used with $\beta=10.0$ and $F_{lin}$ was used with $(\alpha,\beta,\gamma)=(1.0,10.0,0.0001)$.

\end{document}